\documentclass[11pt,final]{article}

\usepackage{fullpage}
\usepackage{amsmath,amssymb,amsthm}
\usepackage{lmodern}
\usepackage[numbers,sort,compress]{natbib}
\usepackage[colorlinks=true]{hyperref}
\hypersetup{
     urlcolor    = blue,
	 citecolor   = teal,
	 linkcolor   = red
}

\usepackage{amsmath}
\usepackage{amsthm}
\usepackage{nicefrac}
\usepackage{mathtools}
\usepackage{graphics}
\usepackage{amsbsy}
\usepackage{algorithm}
\usepackage{algpseudocode}
\makeatletter
\renewcommand\theHALG@line{\thealgorithm.\arabic{ALG@line}}
\makeatother

\algtext*{EndFor}
\algtext*{EndIf}
\algtext*{EndProcedure}
\algtext*{EndFunction}

\newcount\Includeappendix 
\newcommand{\MAB}{\textnormal{MAB}}

\newcommand{\E}{\mathbb{E}} 
\newcommand{\R}{R} 
\newcommand{\His}{H}
\newcommand{\ro}{\pmb{P}}
\newcommand{\delt}{\pmb{\delta}}
\newcommand{\Mmu}{\pmb{\mu}}
\newcommand{\roSub}{P}
\newcommand{\MmuSub}{\mu}
\newcommand{\p}{\pmb{p}}
\newcommand{\q}{\pmb{q}}
\newcommand{\cv}{\pmb{c}}
\newcommand{\uv}{\pmb{u}}

\newcommand{\call}{\mathcal{L}}
\newcommand{\calr}{\mathcal{R}}
\newcommand{\calu}{\mathcal{U}}
\newcommand{\err}{err}

\newcommand{\OPT}{\textnormal{OPT}}
\newcommand{\MER}{\textnormal{MER}}
\newcommand{\DP}{\textnormal{DP}} 
\newcommand{\DPS}{\textnormal{DP}^\star} 
\newcommand{\ORS}{\textnormal{LCB}^\star}
\newcommand{\AORS}{\textnormal{A-LCB}^\star}
\newcommand{\OR}{\textnormal{LCB}}
\newcommand{\LPOR}{\textnormal{L-LCB}}
\newcommand{\Balg}{\textnormal{MOP}}

\usepackage{dsfont}
\newcommand{\ind}{\mathds{1}}

\newcommand{\DOAlg}{match} 
\newcommand{\mainAlg}{EES} 
\newcommand{\LongDOAlg}{Lmatch}
\newcommand{\SSO}{SOB} 

\DeclarePairedDelimiter\ceil{\lceil}{\rceil}
\DeclarePairedDelimiter\floor{\lfloor}{\rfloor}

\newcommand{\PIOPT}{\textnormal{PICO}}

\newtheorem{theorem}{Theorem}
\newtheorem{lemma}{Lemma}
\newtheorem{property}{Property}
\newtheorem{definition}{Definition}
\newtheorem{proposition}{Proposition}

\newtheorem{corollary}{Corollary}
\newtheorem{example}{Example}
\newtheorem{remark}{Remark}
\newtheorem{observation}{Observation}
\newtheorem{assumption}{Assumption}

\usepackage{xcolor}

\DeclarePairedDelimiter\abs{\lvert}{\rvert}%
\newcommand\norm[1]{\left\lVert#1\right\rVert}

\DeclareMathOperator*{\argmax}{arg\,max}
\DeclareMathOperator*{\argmin}{arg\,min}

\algnewcommand{\IIf}[1]{\State\algorithmicif\ #1\ \algorithmicthen}
\algnewcommand{\EndIIf}{\unskip\ \algorithmicend\ \algorithmicif}

\makeatletter
\def\munderbar#1{\underline{\sbox\tw@{$#1$}\dp\tw@\z@\box\tw@}}
\makeatother

\usepackage{cleveref}
\newenvironment{proofof}[1]{\begin{proof}[\textnormal{\textbf{Proof of \Cref{#1}}}]}{\end{proof}} 

\begin{document}

\title{Learning with Exposure Constraints in Recommendation Systems}


\author{Omer Ben{-}Porat%
\thanks{%
    {The Faculty of Data and Decision Sciences, Technion - Israel Institute of Technology   (\url{omerbp@technion.ac.il})}}
\and Rotem Torkan%
\thanks{%
    {The Faculty of Data and Decision Sciences, Technion - Israel Institute of Technology   (\url{torkan.rotem@campus.technion.ac.il})}}
}



\maketitle    
\begin{abstract}
Recommendation systems are dynamic economic systems that balance the needs of multiple stakeholders. A recent line of work studies incentives from the content providers' point of view. Content providers, e.g., vloggers and bloggers, contribute fresh content and rely on user engagement to create revenue and finance their operations. In this work, we propose a contextual multi-armed bandit setting to model the dependency of content providers on exposure. In our model, the system receives a user context in every round and has to select one of the arms. Every arm is a content provider who must receive a minimum number of pulls every fixed time period (e.g., a month) to remain viable in later rounds; otherwise, the arm departs and is no longer available. The system aims to maximize the users' (content consumers) welfare. To that end, it should learn which arms are vital and ensure they remain viable by subsidizing arm pulls if needed. We develop algorithms with sub-linear regret, as well as a lower bound that demonstrates that our algorithms are optimal up to logarithmic factors.
\end{abstract}

\section{Introduction}
Recommendation systems (RSs) are the principal ingredient of  many online services and platforms like Youtube, Quora, Substack, and Medium. 
Algorithmically, those platforms treat the task of recommendation as a matching problem. RSs match a user's context, i.e., their past interactions, demographics, etc., to an item from a predetermined list of items, e.g., news articles, which will hopefully satisfy that user. The quality of a user-content match is initially unclear, so many data-driven approaches have been proposed to determine a matching's quality; for instance, collaborate filtering \cite{Konstan97}, matrix completion \cite{Ramlatchan18}, and online learning \cite{bouneffouf2020survey}. 
However, due to their rapid adoption in commercial applications, many RSs are now dynamic economic systems with multiple stakeholders, facing challenges beyond dissolving uncertainty in matching. Fairness \cite{Biswas21,Gomez22,Patil21,hossain2021fair},  misinformation \cite{hassan2019trust}, user incentives \cite{kremer2014implementing,ben2022modeling}, and privacy \cite{jeckmans2013privacy} are only some of the challenges RSs face.

A recent body of research addresses tradeoffs among stakeholders \cite{burke2017multisided,mansoury2021fairness,burke2018balanced}.
Online platforms have three main stakeholders: The commercial company that runs the platform, content consumers, and content providers. Content consumers, which we refer to as users  for simplicity, enjoy the RSs' content. Content providers, e.g., bloggers and vloggers, create value for RSs by providing relevant and attractive content constantly. Content providers benefit from RSs too, both explicitly by, e.g., obtaining revenue flow from advertisements when users watch their content and interact with it; and implicitly when they gain popularity on the RS that can later be directed to other channels (such as affiliate programs, marketing campaigns, etc.). Observing that an RS's recommendation policy affects content providers, several works address content providers' incentives and develop RSs that incorporate their strategic behavior \cite{craigContentProviderAware, craigprincipled}. For instance, \citet{OmerMosheItay20} consider content provider deviations to produce more revenue-effective items to increase their revenue. Furthermore, \citet{googleOmer20} model the dependency of content providers on user engagement to create revenue and finance their operation. However, these works and others operate under complete information assumptions, entirely neglecting the user-content uncertainty. 

In this work, we propose a model that accounts for both content providers' strategic behavior \textit{and} user-content uncertainty. We consider a contextual multi-armed bandit problem, where the arms are content providers, and in each round, a new user context arrives at the system. The rounds are divided into phases, for instance, a week, month, quarter, or year. The reward of every arm in each phase is the number of pulls it gets, where arm pulls model exposure or user engagement. Arms require a minimum  reward in a phase to produce new content for the following phase; otherwise, the arm ceases to create new content and departs. The minimum number of pulls represents the cost of revenue, i.e., the cost of producing and delivering the content to consumers. We assume that contents become obsolete from phase to phase, so, in every round, the system can only recommend the content of arms that are still viable. The system is fully aware of the exposure thresholds but has no information about each arm's desirability: Some of the arms are of low quality, and their departure does not affect the user welfare. 

We model uncertainty in the stochastic sense; namely, we assume that the context arrival distribution and user-arm match qualities are initially unknown but can be learned through experimentation. Since failing to satisfy an arm's exposure constraint in a phase leads to its irrevocable departure, the system should not only balanced exploration and exploitation but also protect the reward of arms it perceives as vital for long-term user welfare. We use the standard notion of regret to quantify the system's performance, where the baseline is an algorithm that is fully aware of the context arrival distribution and the user-arm match qualities. 

\subsection{Our Contribution}
Our contribution is three-fold. First, we simultaneously address user-content uncertainty and content providers' strategic behavior, an area that has been under-explored in current literature to the best of our knowledge. 
Coupling these two elements brings us closer to real-world RSs' challenges. The second contribution is technical: We devise low-regret algorithms for our model. We begin by assuming that both context arrival distribution and user-arm reward distribution are known. As it turns out, even this task, which we term the stochastic optimization task and study in Section~\ref{sec_stochastic}, is non-trivial. We propose two algorithms that differ in their runtime and performance guarantees. Those algorithms are later leveraged to craft a $\tilde O(T^{\nicefrac{2}{3}})$-regret for several parameter regimes when arrival and reward distributions are unknown. We also provide a lower bound that demonstrates our algorithms are optimal up to logarithmic factors. Our third contribution is a deeper understanding of the notion of \textit{arm subsidy}. Arm subsidy is when we deliberately pull a sub-optimal arm to ensure it remains viable, thereby trading short-term and long-term welfare.  Although our model is stylized, this view allows us to analyze the different motivations for arm subsidy (see Subsection~\ref{subsec:armsubsidy}) and conclude it is vital for long-term social welfare, strengthening the 
ideas argued by previous work \cite{googleOmer20}. 

\subsection{Related Work}
Our motivation stems from previous work on  strategic behavior in Machine Learning \cite{Feng2020,sellke2021price,leonardos2021exploration,hardt2016strategic,gast2020linear,jagadeesan2022competition,Zhang2022EC}, and more specifically on strategic publishing \cite{mosheorencacm,OmerMosheItay20}. This line of work adopts a game-theoretic perspective and analyzes equilibrium behavior \cite{MosheOmer18,HronJiri22}. In particular, \citet{MosheOmer18} suggest that content providers can create many types of content, for instance, articles on sports, economics, etc., and that they are willing to alter the content they create to increase their exposure. This strategic modeling initiates a game, and their paper seeks recommendation policies that induce pure equilibrium profiles. 
\citet{googleOmer20} study a special case where every content provider has two options: Either to provide their default content (e.g., economic posts for economic-focused bloggers) or to leave the platform (stop creating new content). Indeed, their modeling is similar in spirit to ours; however, \citet{googleOmer20} assume complete information, whereas we do not. Consequently, the questions we ask  and the algorithms we develop differ significantly.

Our work is also related to two recent lines of work on multi-armed bandits. First, to bandits with strategic arms \cite{Braverman19, Feng2020}. \citet{Braverman19} assume that pulling an arm generates revenue for that arm, which she can split and pass to the principal pulling the arms and keep some to itself. \citet{Feng2020} introduce a similar model where arms can manipulate the rewards they generate to appear more attractive. In contrast, in our model, the rewards generated from arm pulls cannot be manipulated. 

Second, our model is related to fair content exposure in bandits \cite{ron2021corporate,Patil21}. \citet{Patil21} consider a bandit setting with exposure constraints the system must satisfy. Their motivation is fairness, where the regulator demands that each arm be pulled a minimal number of times, potentially harming the user welfare. However, in this paper the system need not satisfy the exposure constraints of all arms. In our work, the exposure constraints result from the providers' (namely, arms) economic considerations. The system will only satisfy the constraints of providers who are essential for maximizing the social welfare of the users. 

From a broader perspective, our work relates to bandits with complex rewards schemes, i.e.,  abandonment elements \cite{ben2022modeling,cao2019dynamic,YangZix22}, and  non-stationary rewards \cite{levine2017rotting,besbes2014stochastic,seznec2019rotting,pike2019recovering,Leqi21,Kleinberg18}. Our work is also related to multi-stakeholder recommendation systems \cite{burke2017multisided,ZhanRouhan2021, Biswas21,beutel2019fairness,mehrotra2018towards,Singh2018} and fairness in machine learning \cite{Mehrabi21,WangYifan22,Chenn20}. 

\section{Model}\label{sec:model}
    
We adopt the blogging motivation as our ongoing example. 
There is a set of arms (content providers) $K=[k]$, and a sequence of user contexts of length $T$. We assume that contexts compactly represent users by their \textit{type}, e.g., age, location, etc. Namely, the system cannot distinguish two users of the same context. We denote by $U$ the set of user contexts, $U=[n]$, and use the terms context and user type interchangeably.

In each round $t$ ,$1\leq t \leq T$, a user of type $u_t$ arrives at the system. We assume a stationary arrival distribution $\ro$ over $U$; namely, we assume that $u_1, \dots ,u_T \sim \ro^T$. After observing the user type $u_t \in U$, the system pulls an arm, denoted by $a_t\in K$. The arm pull models displaying content $a_t$ to the user of that round. Afterward, the system gets feedback, for instance, a like or a rating, modeling the user's \textit{utility} from that content, which we denote by $v_{t}$. We further assume that for every user type $u\in U$ and every arm $a\in K$, the associated utility is a random variable drawn from a distribution $D_{u,a}$. We let $\Mmu$ denote the matrix of expected utilities, where $\MmuSub_{u,a}$ is the expected utility of a user type $u\in U$ from an arm $a\in K$. In other words, $\MmuSub_{u,a}=\E_{v\sim D_{u,a}}[v]$. For simplicity, we assume that the distributions $D_{u,a}$ for every $a\in K$ and $u\in U$ are supported in the $[0,1]$ interval; hence, $v_t \in [0,1]$ for every $1\leq t\leq T$ almost surely.

An online decision making algorithm is a function from the history $\left(u_{t'},a_{t'},v_{t'}\right)_{t'=1}^{t-1}$ and a user type to an arm. That is, the algorithm selects the arm to pull given the history and the context $u_t$ in round $t$. The \textit{reward} of an algorithm $A$ is the expected sum of user utilities. Formally, 
\begin{equation}\label{eq:utilitysumbefore}
\E[r^A]=\E\left[\sum_{t=1}^T v_{t}^A\right]=\E\left[ \sum_{t=1}^T \mu_{u_t,a_t^A} \right].
\end{equation}
When the algorithm $A$ is clear from the context, we remove the superscript $A$.

Thus far, we have defined a standard contextual $\MAB$ setting with stochastic utilities, a discrete set of contexts (or types), and stochastic context arrival. In the next subsection, we present a novel component that captures the strategic behavior of content providers.

\subsection{Exposure Constraints}

Our model departs from previous works by considering the following \textit{exposure constraint}:
Each arm requires a minimal amount of exposure, i.e., the number of times the system pulls it in a given time frame, to keep creating fresh content. 
Arms are content providers who depend on interactions with users for income (or on exposure as a proxy); thus, an arm $a\in K$ that receives less exposure than $\delta_a$ might face difficulties in paying their periodic bills. This constraint represents the providers' need for exposure and the tendency to churn the system in lack thereof.

Specifically, every arm requires a predetermined number of exposures in every \textit{phase}, where phase lengths are fixed and comprised of $\tau$ rounds. Namely, at the end of every phase, each arm $a\in K$ examines the number of exposures it received, and if it is less than some given and publicly known threshold $\delta_a \in \{0,1,\dots \tau\}$, then arm $a$ departs and cannot be pulled anymore.\footnote{If the exposure thresholds are private, linear regret is inevitable; see Subsection~\ref{sec:thmdelta}.} To exemplify, $\tau$ could be the number of days in a month, where each day is a round. An arm $a\in K$ could require at least $\delta_a$ impressions every month to keep creating fresh content in the following months, otherwise, it churns the system. Observe that we have $\nicefrac{T}{\tau}$ phases. 
 
We further denote by $K_t\subseteq K$ the subset of arms that are still available in round $t$.\footnote{Alternatively, one could say that  all arms are always available, but the utility from arms $K\setminus K_t$ is 0 almost surely at round $t$. Indeed, in such a case, pulling an arm from $K\setminus K_t$ is weakly dominated by any other arm selection.} Mathematically, the exposure constraint is modeled as follows. An arm $a \in K_t$ is viable in round $t$ if in every phase prior to the current phase $\floor{\nicefrac{t}{\tau}}+1$,
arm $a$ received at least $\delta_a$ pulls. Namely, if 
\[
\forall j, 1\leq j \leq \floor{\nicefrac{t}{\tau}}: \sum_{t=(j-1)\cdot\tau+1}^{j\cdot\tau}\ind_{[a=a_t]}\geq \delta_a,
\]
where $\ind_{[a=a_t]}$ indicates the event that $a=a_t$. Rewriting the expected reward defined in Equation~\eqref{eq:utilitysumbefore} with the above exposure constraints, we have
\begin{equation}\label{eq:utilitysum}
\E[r^A]=\E\left[\sum_{t=1}^T v^A_{t}\right]=\E\left[\sum_{t=1}^T \mu_{u_t,a^A_t}\cdot \ind_{[a^A_t \in K_t^A]}\right] .
\end{equation}

\subsection{Arm Subsidy}\label{subsec:armsubsidy}
A surprising phenomenon that our model exhibits is \textit{arm subsidy}: Deliberately pulling a sub-optimal arm to maintain its exposure constraint. Indeed, such short-term harm is sometimes required to achieve long-term social welfare (i.e., the expected reward in Equation~\eqref{eq:utilitysum}). To illustrate this concept and the reasons it is vital, we present the following examples.

\begin{example}\label{example1221}
\textup{Consider the instance with $k=2$ arms, $n=2$ user types, arrival distribution $\ro=(0.5,0.5)$, phase length $\tau =100$, exposure thresholds $\delt=(40,40)$, and expected utility matrix $\Mmu = \bigl( \begin{smallmatrix}1 & 0\\ 0 & 1\end{smallmatrix}\bigr).$
The exposure thresholds suggest that each arm requires $40$ pulls in each phase of length $100$ rounds. 
To showcase why arm subsidy is needed, let us assume that the algorithm knows the utility matrix $\Mmu$. Consider a myopic algorithm that pulls the best arm for every user type and overlooks the exposure constraints. The myopic algorithm pulls arm $i$ for user type $i$, for $i\in \{1,2\}$, since it yields the best utility, resulting in an expected utility of 1 in every round (where the expectation is taken over the user type). 
Apparently, the myopic algorithm is optimal; however, the exposure constraints complicate the setting. At some phase, fewer than $\delta_i=40$ type $i$-users will arrive for some $i\in \{1,2\}$.\footnote{Formally, let $E_j$ denote the event that less than 40 type $i$-users (for some $i\in\{1,2\}$) arrive at phase $j$. Then, $\Pr[E_j]\approx 0.035$; hence, the probability one arm will depart  after at most 200 phases is $\Pr\left[\cup_{j=1}^{200} E_j\right]\approx 0.999$.} W.l.o.g., assume that $i=1$; thus, the myopic algorithm would not satisfy arm $1$'s exposure constraint. For all rounds in the remaining phases, the only available arm is arm 2, and all type $1$-users receive a zero utility; therefore, the reward is halved and user welfare has severely deteriorated.}
\end{example}

In the instance given in Example~\ref{example1221}, an algorithm might decide to subsidize arm $1$ by pulling it for some type $2$ users, deliberately harming the short-term reward to ensure that arm $1$ will not depart (and similarly for arm $2$). In such a case, the subsidy stems from the variance of the user arrival---protecting arms against rare events. However, subsidies might be required more often, as we demonstrate in Example~\ref{example1331} below.

\begin{example}\label{example1331}
\textup{Consider the instance with $k=2$ arms, $n=2$ user types, arrival distribution $\ro=(0.5,0.5)$, phase length $\tau =100$, exposure thresholds $\delt=(10,60)$, and expected utility matrix $\Mmu = \bigl( \begin{smallmatrix}1 & 0\\ 0 & 1\end{smallmatrix}\bigr).$
Notice that arm $1$ ($2$) requires $10$ ($60$) pulls in each phase of length $100$ rounds to remain viable. Furthermore, as in the previous example, assume that the algorithm knows the utility matrix $\Mmu$.
The myopic algorithm selects arm $i$ for type $i$-users. However, in a typical phase, there will be less than $\delta_2=60$ type $2$-users, causing arm $2$ to depart.\footnote{Let $E_j$ be the event that more than 60 type $2$-users arrive at phase $j$. Consequently, $\Pr[E_j]< 0.03$, and the probability that arm $2$ will depart after at most two phases is $1-\Pr\left[E_1 \cap E_2\right]\approx 0.999$.} 
Therefore, without subsidy, arm $2$ departs early and the myopic algorithm receives a reward of $\approx 0.5 T$. 
In contrast, consider an algorithm that subsidizes arm $2$; namely, it pulls arm $2$ for roughly $20\%$ of type $1$-users, ensuring both exposure constraints are satisfied and acts myopically otherwise. Such an algorithm receives an expected reward of $\approx 0.9 T$ (since roughly $10\%$ of the rounds are subsidy rounds yielding a zeroed utility, and other rounds yield utility 1).} 
\end{example}
For the instance in Example~\ref{example1331}, optimal algorithms subsidize arm $2$ continuously due to structural reasons---pulling it only when it is optimal is not enough for keeping it viable. 

Examples~\ref{example1221} and \ref{example1331} might mislead into thinking that subsidy is always beneficial; however, this is not the case. To see this, consider the following Example~\ref{example1441}. 
\begin{example}\label{example1441}
\textup{
Consider a modification of the instance of Example~\ref{example1331},
where the arrival distribution is $\ro=(0.9,0.1)$ and all other parameters are the same, namely $k=2, n=2, \tau =100, \delt=(10,60), \Mmu = \bigl( \begin{smallmatrix}1 & 0\\ 0 & 1\end{smallmatrix}\bigr)$. In this case, type $2$-users are scarce, and an algorithm that subsidizes arm $2$ with at least $\delta_2=60$ users per phase receives a reward of $\approx 0.5T$. However, an algorithm that forgoes arm $2$ receives a reward of $\approx 0.9T$; thus, subsidizing arm $2$ significantly harms the reward.}
\end{example}

\subsection{Algorithmic Tasks} \label{secAlgTask}

First and foremost, we study the learning task, in which the goal is to find an online decision-making algorithm that minimizes regret (formally explained later in Subsection~\ref{secRegret4325}). The algorithm knows the number of arms $k$, number of user types $n$, phase length $\tau$, number of rounds $T$, and exposure thresholds $\delt$; however, it is unaware of the expected utilities $\Mmu$ and the arrival distribution $\ro$.

A second, auxiliary task that is crucial for both quantifying regret and the learning algorithm we propose in Section~\ref{sec:learningalg} is the stochastic optimization task. In this task, the algorithm is aware of $\ro,\Mmu$, and we want to maximize the expected reward in Equation~\eqref{eq:utilitysum}. Each user type $u_t$ is known in the Bayesian sense, i.e., the algorithm knows $\ro$ and that $u_t\sim \ro$, but observes the realized value only in round $t$ and not before. Once the algorithm observes $u_t$ in round $t$, it can compute the utility $\mu_{u_t,a}$  for every $a\in K$ since it knows $\Mmu$. 

We denote $\OPT$ as the algorithm that maximizes the expected reward in the stochastic optimization task. Naturally, $\OPT$ serves as a benchmark for the learning task as well. 
In contrast to standard stochastic $\MAB$ settings, where the optimal algorithm is myopic---pulling the best arm in every round (or the best arm per the given context in that round)---$\OPT$ is more complex. To see this, recall that in Examples~\ref{example1221} and \ref{example1331}, the optimal policy is not myopic as it has to decide whether and how to subsidize. To that end, we study the stochastic optimization task in Section~\ref{sec_stochastic}.

\paragraph{Reward Notation}
To clarify when we refer to algorithms for the stochastic optimization task, we adopt a different notation for the reward, albeit calculated similarly (see Equation~\eqref{eq:utilitysum}). We use the tilde notation, i.e., $\E[\tilde r^A]$, to denote the reward for a stochastic optimization algorithm $A$.


\subsection{Parametric Assumptions}\label{subsec:parametric}
Similarly to previous works, we consider the number of rounds $T$ to be the scaling factor. Namely, we assume the other parameters are either constants or functions of $T$. More specifically, we assume that the number of arms $k$ and the number of user types $n$ are constants, and that $\roSub_u$ is a constant that is bounded away from zero for every user type $u\in U$.

Furthermore, we make two simplifying assumptions, allowing us to focus on regimes we believe are more appropriate to the problem at hand. Importantly, we relax these two assumptions in Section~\ref{sec:extension}. First, we ensure that exploration can be fast enough.
\begin{assumption}\label{assumption:gamma}
There exists a constant $\gamma \in (0,\nicefrac{1}{k}]$, such that
\begin{equation}\label{eq:assumption1}
   \sum_{a\in K} \max(\delta_a, \gamma \cdot \tau) \leq \tau. 
\end{equation}
\end{assumption}
For instance, if $\delta_a \leq \frac{\tau}{k}$ for every $a\in K$, then Inequality~\eqref{eq:assumption1}  is satisfied with $\gamma = \frac{1}{k}$. We remark that $\gamma$ is treated as a constant that does not depend on $\tau$. With this assumption, one can satisfy all arms' exposure constraints in a phase while exploring each arm at least $\floor{\gamma \tau}$ rounds. Second, we assume that the phase length $\tau$ is not too large.
\begin{assumption}\label{assumption:tau}
The phase length $\tau$ is upper bounded by $O(T^{\nicefrac{2}{3}})$. 
\end{assumption}
Assumption~\ref{assumption:tau} is necessary to focus on algorithms that are stationary, i.e., operate the same in every phase (as we formally describe in Subsection~\ref{subsec:PIC algclass}). Before we end this subsection, we stress again that Section~\ref{sec:extension} relaxes Assumption~\ref{assumption:gamma} and Assumption~\ref{assumption:tau}.

\subsection{Regret}\label{secRegret4325}
Recall that Subsection~\ref{secAlgTask} defines $\OPT$ as the optimal algorithm for the stochastic optimization task. To quantify the success of an online decision-making algorithm in the learning task, we adopt the standard notion of \textit{regret}, where the baseline is $\OPT$. Formally, the regret $\R^A$ of an algorithm $A$ is the difference between its reward and the reward of $\OPT$, i.e., $\E[\R^A]=\E[\tilde{r}^{\OPT}]-\E[r^A]$. Observe that $\E[\tilde{r}^{\OPT}]$ is constant (given the instance parameters); thus, minimizing the regret is equivalent to maximizing the reward.

Having defined the notion of regret, we now show that planning for subsidy is essential. Optimal algorithms for the standard multi-armed bandit algorithms that are not modified for exposure constraints, e.g., $UCB$ \cite{slivkins2019introduction}, do not subsidize arms at all. 
The following Theorem~\ref{the_12321} suggests that avoiding subsidy as well as blind subsidy, ensuring all arms receive their minimal exposure regardless of the rewards, results in high regret.
\begin{theorem} \label{the_12321}
Any algorithm that either
\begin{enumerate}
    \item maintains all arms always, i.e., has $K_1=K_2=\dots=K_T$ with probability 1; or,
    \item avoids subsidy
\end{enumerate}
suffers a regret of $\Omega(T)$ in some instances.
\end{theorem}
The proof of Theorem~\ref{the_12321} appears in {\ifnum\Includeappendix=1{Appendix~\ref{subsec: proof thm 1}}\else{the appendix}\fi}. This theorem calls for devising algorithms that incorporate careful subsidies while exploring and exploiting. In the next section, we take the first step toward that goal and study the stochastic optimization task.

\section{Stochastic Optimization} \label{sec_stochastic}



In the stochastic optimization task, the goal is to maximize the expected reward $\tilde{r}$ (recall Subsection~\ref{secAlgTask}). In this section, we propose several algorithms and techniques for this task. First, in Subsection~\ref{subsec:PIC algclass}, we highlight a class of approximately optimal algorithms that are rather simple to analyze. Second, in Subsection~\ref{subsec:dynProg}, we develop an approximately optimal dynamic programming-based algorithm, termed $\DP^\star$, with a runtime of $O(T+\tau^{k})$.\footnote{Recall that $k,n$ are considered constants; thus, we hide them inside the big O notation.} Third, in Subsection~\ref{subsec:lcb_approx}, we devise $\OR^\star$, which has weaker performance guarantees but a much faster runtime of $O(T+ \tau ^ 3)$. We stress that the proposed algorithms in this section are stochastic optimization algorithms, explicitly assuming that $\ro,\Mmu$ are known. 


\subsection{Phase-Independent and Committed Algorithms}\label{subsec:PIC algclass}
In this subsection, we highlight a narrow class of simple algorithms yet potentially approximately optimal, as we prove in Theorem~\ref{theoremOAlg1}. 
We say that an algorithm is \textit{phase-independent} if it pulls an arm in any round based on the sequence of users arriving on that phase only, ignoring previous phases. More formally,\footnote{In fact, phase-independent algorithms are also independent of the realized rewards. However, since we care about maximizing the expectation $\E[\tilde{r}]$, we can treat the expected rewards and not the realized ones.}
\begin{property}[Phase independence] \label{prop_PI12343}
An algorithm $A$ is phase-independent if for every two phases $i,j$ such that $0\leq i < j < \frac{T}{\tau}$ and every $t$ such that $1\leq t\leq \tau$ it holds that,
\[
\text{if } \left(u_{i\cdot \tau + l} \right)_{l=1}^{t} = \left(u_{j\cdot \tau + l} \right)_{l=1}^{t} \text{, then } a^A_{i\cdot \tau + t} = a^A_{j\cdot \tau + t}.
\]
\end{property}
Phase-independent algorithms are easier to analyze, since we optimize over $\tau$ and not $T>>\tau$ rounds. Another useful property is the following. We say that an algorithm is \textit{committed} to a subset of arms if it pulls arms from that subset only \textit{and} satisfies the exposure constraints of every arm in that subset. Put differently,
\begin{property}[Committed]\label{prop_commited}
An algorithm $A$ is committed to an arm subset $Z\subseteq K$ if for all $t$, $1\leq t \leq T$ it always holds that 
\begin{itemize}
    \item $a_t^A \in Z$; and,
    \item $Z\subseteq K^A_t$.
\end{itemize}
\end{property}
For an algorithm $A$ to be committed to $Z$, it must satisfy the two conditions of Property~\ref{prop_commited} with a probability of $1$ (over the randomization of user type arrivals and rewards). Further, the second condition, i.e., $Z\subseteq K^A_t$, means that $A$ must ensure that the arms $Z$ receive their exposure minimums and subsidize if needed. 

Naturally, an optimal algorithm need not be committed. Two extreme cases are the first phase, where we can temporarily exploit arms with high rewards that are too costly to subsidize; and the last phase, where we have no incentive to subsidize at all. Committed algorithms cannot pick arms outside their committed subset (the first condition of Property~\ref{prop_commited}). Due to similar reasons, optimal algorithms need not be phase-independent. 
However, as we prove in Theorem~\ref{theoremOAlg1}, focusing on phase-independent and committed algorithms suffices for finding good approximations. 

For convenience, we denote by $\tilde r^A_{i:j}$ the reward of an algorithm $A$ in rounds $i,\dots,j$; namely, $\tilde r^A_{i:j} = \sum_{t=i}^j v_t^A$. Clearly, $\tilde r^A_{i:j}$ depends on the arms pulled in rounds $1,\dots ,i-1$. In particular, $\tilde r^A_{1:\tau}$ is the reward of the first phase. Since we focus on phase-independent and committed algorithms (recall Properties~\ref{prop_PI12343} and~\ref{prop_commited}), it holds that 
\begin{equation}\label{eq:phaseToT}
   \E[\tilde{r}^A]=\frac{T}{\tau}\E_{\q\sim \ro^{\tau}}[\tilde{r}_{1:\tau}^A],  
\end{equation}
where $\ro^\tau$ is the distribution over i.i.d. user types. 

Let $\PIOPT$ be any algorithm that
maximizes the expected reward among all the algorithms satisfying Properties~\ref{prop_PI12343} and~\ref{prop_commited}. Namely, 
\[
\PIOPT \in \argmax_{\substack{A\text{ satisfies } \text{Prop. \ref{prop_PI12343} and~\ref{prop_commited}}}} \E[\tilde{r}^{A}].
\]
In addition, we denote by $\PIOPT(Z)$ the optimal algorithm that maintains Properties~\ref{prop_PI12343} and~\ref{prop_commited} and commits to the arm subset $Z\subseteq K$. The following Theorem~\ref{theoremOAlg1} asserts that $\PIOPT$ is almost optimal.
\begin{theorem}\label{theoremOAlg1}
It holds that
$\E[\tilde{r}^{\PIOPT}]\geq \E[\tilde{r}^{\OPT}]-k\cdot \tau.$
\end{theorem}
The proof of Theorem~\ref{theoremOAlg1} appears in {\ifnum\Includeappendix=1{Appendix~\ref{Asubsec:pico appr opt}}\else{the appendix}\fi}. Note that $\E[\tilde{r}^{\PIOPT}]$ is typically in the order of $T$. Since $\tau k$ is sub-linear in $T$ (recall Assumption~\ref{assumption:tau}) , this approximation suffices.

While Theorem~\ref{theoremOAlg1} claims approximate optimality, it does not explain how to implement $\PIOPT$. In the next Subsections~\ref{subsec:dynProg} and \ref{subsec:lcb_approx} we propose an exact and an approximate implementation of $\PIOPT$.



\subsection{Dynamic Programming}\label{subsec:dynProg}
\begin{algorithm}[t]
\caption{Dynamic Programming $\PIOPT$ ($\DP^\star$)}
\label{alg:meranddp}
\begin{algorithmic}[1]
\Statex \textbf{Input}: $Z\subseteq K,U,\delt,\tau,T,\Mmu,\ro$
\State pick $Z^\star \in \argmax_{Z\subseteq K}\MER(\emptyset, Z)$ \label{algline:DPbestZ} \Comment{offline computation}
\State follow $\DP(Z^\star)$  \label{algline:followDPZ}
\Procedure{$\DP$}{$Z$} \label{algline:procDPZ}
\State compute $\MER(\emptyset, Z)$ \label{algline:offlinemer}
\State $H^0 \gets \emptyset$ \label{algline:initH}
\For{$t=1,\dots \tau$}\label{algline:forloop}
    \State observe $u_t$ \label{algline:observeu}
    \State  $a_t\leftarrow \argmax_{a\in Z} \left\{\MmuSub_{u_t,a}+\MER(\His^{t-1} \cup \{a\},Z)
    \right\}$ \label{algline:argmax}
    \State pull arm $a_t$
    \State $\His^t \gets \His^{t-1} \cup \{a_t\}$ \label{algline:updateH}
\EndFor
\State go to line~\ref{algline:initH} \label{algline:repeat}\Comment{repeat the process for phases $2,\dots ,\nicefrac{T}{\tau}$}
\EndProcedure
\Function{$\MER$}{$\His,Z$} \label{algline:procmfr}
  \If {$|\His|=\tau$} \label{algline:stoppingif}
\If {$\forall a\in Z: \His(a)\geq \delta_a$} return 0\label{alline:ifsatisfied}
\Else {} return $-\infty$\label{mfr_line5NEW}
\EndIf
\Else {} return $\E_{u\sim \ro}\left[\max_{a\in Z} \left\{ \MmuSub_{u,a}+ \MER(\His \cup \{a\}) \right\}\right] $ \label{algline:MERreturn}
\EndIf
\EndFunction
\end{algorithmic}
\end{algorithm}

In this subsection, we develop $\DP^\star$, which implements $\PIOPT$ with runtime of $O(T+\tau^{k})$. 
$\DP^\star$, which is implemented in Algorithm~\ref{alg:meranddp} comprises calls to the $DP(Z)$ procedure and the $MER(H,Z)$ function (Lines~\ref{algline:procDPZ} and~\ref{algline:procmfr}). We begin by introducing a few useful notations. 

Throughout the execution of an algorithm, let the \textit{phase history} $\His$ represent the prefix of pulls in the current phase. More specifically, we denote by $\His^t$ the multi-set of arms pulled in rounds $1,2,\dots t$ in that phase, for $1\leq t\leq \tau$. Consequently, $\abs{\His^t} = t$. Furthermore, we slightly abuse the notation and denote by $\His^t(a)$ the number of times arm $a$ was pulled so far in that phase. For example, after the fifth round of some phase (i.e., $t=5$), $\His^5$ could be $\{a_1,a_1,a_1,a_2,a_4\}$, and so $\His^5(a_1)=3,\His^5(a_2)=\His^5(a_4)=1$, etc. 

\subsubsection{Maximum Expected Reward Function ($\MER$)}
We now explain the $\MER$ function, see Line~\ref{algline:procmfr} in Algorithm~\ref{alg:meranddp}. $\MER$, which is recursive, receives a subset $Z\subseteq K$ and a phase history $H$ as input. $\MER$ computes the highest possible \textit{remaining} expected reward $\E_{\q}\left[\tilde{r}_{\abs{H}+1:\tau}\right]$ of an algorithm committed to $Z$ provided that the phase history is $H$. 

$\MER$ begins with the stopping condition in Line~\ref{algline:stoppingif}. The if clause asks whether the phase history is complete (i.e., if there were $\tau$ arm pulls in that phase). If so, it examines the phase history and checks whether all arms in $Z$ satisfy their exposure constraints (the if clause in Line~\ref{alline:ifsatisfied}). In such a case, it returns zero as the sum of the remaining rewards as there are no further rounds in the phase. Otherwise, if there exists an arm whose exposure constraint is unsatisfied by the end of the phase, it returns $-\infty$ (Line~\ref{mfr_line5NEW}). Clearly, if the exposure constraints are unsatisfied, the complete phase history $H$ could not have been obtained by an algorithm committed to $Z$. 

The recursive step appears in Line~\ref{algline:MERreturn}. $\MER(H,Z)$ relies on the extended phase histories of $H$ after one round, i.e., $\His \cup \{a\}$ for every $a\in  Z$. For instance, if the user type arriving at round $\abs{H}+1$ is $u$ and arm $a$ is selected, then 
$ 
\MER(H,Z)=\MmuSub_{u,a}+\MER(\His \cup \{a\}) ,
$ 
where the first term in the right-hand side represents the reward in round $\abs{H}+1$ and the second term, per the inductive step, represents the maximal expected remaining reward $\tilde r_{\abs{H}+2:\tau}$ of any algorithm committed to $Z$ provided that the phase history is $\His \cup \{a\}$. To compute $\MER(H,Z)$, we take expectation over all possible user types. For each user type, we pick the arm that maximizes the expected remaining reward conditioned on that user type.

\subsubsection{Dynamic Programming Procedure ($\DP(Z)$)}
We now turn to explain $\DP(Z)$, see Line~\ref{algline:procDPZ} in Algorithm~\ref{alg:meranddp}. It receives an arm subset $Z\subseteq K$, and executes $\MER(\emptyset, Z)$ (Line~\ref{algline:offlinemer}). By doing so, as we prove later in Theorem~\ref{thm:dp}, we obtain a recipe of the optimal algorithm. In Line~\ref{algline:initH}, it initializes the phase history. Then, the for loop in Line~\ref{algline:forloop} handles the arm pulls of the phase. The procedure observes $u_t$, picks an arm $a_t$ and updates the phase history in Lines~\ref{algline:observeu}-\ref{algline:updateH}. When the phase ends, we move to Line~\ref{algline:repeat}, and repeat Lines~\ref{algline:initH}-\ref{algline:updateH} for the remaining phases. 

Notice that $\DP(Z)$ satisfies Propoerties~\ref{prop_PI12343} and \ref{prop_commited}. It is phase-independent since it operates in phases, restarting after every phase (Line~\ref{algline:repeat}). To claim that $\DP$ is committed to $Z$, we need to show the two conditions in Property~\ref{prop_commited}. The first condition holds since $\DP$ pulls arms from $Z$ in every round (Line~\ref{algline:argmax}). The second condition holds since it follows $\MER$, and $\MER$ sets $-\infty$ for completing a phase without satisfying the exposure constraints of the arms $Z$ (the else clause in Line~\ref{mfr_line5NEW}).

Having explained $\DP(Z)$, we now present its formal guarantees.
\begin{theorem}\label{thm:dp}
For every $Z\subseteq K$, it holds that
\[ \MER(\emptyset,Z)=\E_{\q}[\tilde{r}_{1:\tau}^{\DP(Z)}]=\E_{\q}[\tilde{r}_{1:\tau}^{\PIOPT(Z)}].
\]
\end{theorem}
The proof of Theorem~\ref{thm:dp} appears in {\ifnum\Includeappendix=1{Appendix~\ref{Asubsec:thmdp}}\else{the appendix}\fi}.
Since $\DP(Z)$ satisfies Properties~\ref{prop_PI12343} and~\ref{prop_commited}, Equation~\eqref{eq:phaseToT} and Theorem~\ref{thm:dp} also hint that $E[\tilde{r}^{\DP(Z)}]=\E[\tilde{r}^{\PIOPT(Z)}].$
The correctness of Theorem~\ref{thm:dp} follows from the optimality of $\MER$. In every round, after observing the user type in Line~\ref{algline:observeu}, we pull the arm that balances the instantaneous utility $\mu_{u_t,a}$ and the expected reward of the future rounds. 

\subsubsection{Finding the Optimal Arm Subset}
We are ready to explain $\DP^\star$. In Line~\ref{algline:DPbestZ}, it picks the best subset $Z^\star$ to commit to, i.e., the one maximizing the expected reward. Afterward, in Line~\ref{algline:followDPZ}, it executes the $\DP$ procedure when initiated with $Z^\star$. 
Due to Theorem~\ref{thm:dp} and Equation~\eqref{eq:phaseToT},
\begin{equation}\label{eq:optimalDP}
\E[\tilde{r}^{\DP^\star}]:=\max_{Z\subseteq K}\E[\tilde{r}^{\DP(Z)}]=\E[\tilde{r}^{\PIOPT}].
\end{equation}

\subsubsection{Runtime}
The following Proposition~\ref{MERruntime} asserts that Line~\ref{algline:DPbestZ}, namely computing $\MER(H,Z)$ for every $H$, can be done in $poly(\tau)$ time.
\begin{proposition}\label{MERruntime}
Finding $Z^\star$ in Line~\ref{algline:DPbestZ} takes  $O(2^k\tau (\tau+k)^{k-1} \cdot kn)$ steps.
\end{proposition}
The last ingredient of the runtime is the execution of $\DP(Z^\star)$ in Line~\ref{algline:followDPZ}. Every round of $\DP(Z^\star)$ takes roughly $k$ steps; therefore, executing $\DP(Z^\star)$ requires runtime of $O(Tk)$. Overall, we conclude that executing $\DP^\star$ takes $O(Tk+2^k\tau (\tau+k)^{k-1} \cdot kn)$. While exponential in $k$, recall that we focus on the regime where $n,k$ are constants. In such a case, the runtime is actually $O(T+\tau^k)$.\footnote{Importantly, as we discuss in Subsection~\ref{subsec:pick z}, the $2^k$ factor required for finding $Z^\star$ is inevitable as the problem is NP-complete.} A runtime of $O(T+\tau^k)$ is feasible when the phase length is short, e.g., if $\tau = \log T$; however, for some applications, e.g., if $\tau=T^{\nicefrac{2}{3}}$ and $k=100$, $O(\tau^k)$ is still impractical. In the next subsection, we propose a faster algorithm that trades performance with runtime.

\subsection{Lower Confidence Bound Approximation} \label{subsec:lcb_approx}
In this subsection, we develop the Lower Confidence Bound algorithm (denoted $\OR^\star$). Its performance only approximates that of $\PIOPT$, but it is much more efficient than $\DP^\star$ (runtime of $O(\tau^3)$). To motivate $\OR^\star$, recall that stochastic optimization is an online task. The problem would have been easier if we knew the user type arrival vector $\q$ (see Equation~\eqref{eq:phaseToT}): We could optimize the expected reward in an offline, deterministic fashion (that we later name $\DOAlg$). Despite that we do not know $\q$, we can use concentration inequalities to create a lower confidence bound for the user type arrival. $\OR$ implements these two techniques: Concentration and offline matching.

A useful notation for this section is that of $c(\q)$. We let $c(\q)\in \mathbb N^n$ be the user type aggregates such that $c(\q)_u = \sum_{i=1}^\tau \ind_{q_t=u}$. 

\subsubsection{Offline Matching}
We begin by explaining the offline matching problem. The input is a subset $Z\subseteq K$, all the instance parameters $\textit{and}$ a user type arrival sequence $\uv=(u_1,\dots,u_\tau)$ of a phase. The goal is to find $(a_1,\dots,a_\tau)\in Z^\tau$ that maximizes the phase reward $\sum_{t=1}^\tau \mu_{u_t,a_t}$ while satisfying the exposure constraints of every $a\in Z$. We emphasize that the algorithm can only pull arms from $Z$. A more careful look tells us that we can drop the order of $\uv=(u_1,\dots,u_\tau)$ since the objective function is invariant under permutations: Swap two entries in $\uv$, and the optimal solution will be swapping these two entries as well. We can therefore restate the offline matching problem with the aggregate $c(\uv)$, where the  solution is a matrix $M\in \mathbb N^{n \times k}$ such that $M(u,a)$ is the number of times a user of type $u$ was matched to an arm $a$. For convenience, let $M(u,\cdot)=\sum_{a\in K} M(u,a)$ and similarly $M(\cdot, a)= \sum_{u \in U} M(u,a)$ be the summation of the $u$'th row and the $a$'th column in $M$, respectively.

The function $\DOAlg(c(\uv),Z)$, which appears in Line~\ref{alglineor:doalgproc} of Algorithm~\ref{alg:lcbapprox234f}, returns the optimal  matching for the user type arrival aggregate $c(\uv) \in \mathbb N^n$ and subset of arms $Z$. Since the implementation details of $\DOAlg$ are secondary for the rest of this subsection, we defer its formal discussion to Subsection~\ref{subsec: match algorithm}. 
Importantly, its runtime is $O(\tau^3 k^3)$. 

\subsubsection{Concentrated Aggregates and Slack Users} 

Next, we rely on the concentration of type arrival distribution $\q$ (see Equation~\eqref{eq:phaseToT}) to create a (high probability) lower confidence bound for the user type arrival. 
Notice that since $\q\sim \ro^{\tau}$, $c(\q)\sim multinomial(\ro, \tau)$. A straightforward application of Hoeffding's inequality suggests that
\begin{equation}\label{eq:concentration}
\Pr_{\q\sim \ro^\tau}\left(\exists u: c(\q)_u < P_u \cdot \tau-\sqrt{\tau\log{\tau}} \right) < \frac{2n}{\tau^2}.
\end{equation}
Let ${\cv}=(c_1,\dots, c_n)$ be a vector such that ${c}_u =  P_u \cdot \tau-\sqrt{\tau\log{\tau}}$. Inequality~\eqref{eq:concentration} suggests that with probability of at least $1-\frac{2n}{\tau^2}$ over user type arrival vectors, there will be at least $c_{u}$ users of type $u$ for every $u\in U$. We term this event the \textit{clean event}. The complementary case, which we term the \textit{bad event}, is when there is a user type for which $c(\q)_u < c_u$; this happens with probability at most $\frac{2n}{\tau^2}$.

We want to construct a stochastic optimization algorithm that relies on the lower bound confidence aggregates. One straightforward approach would be to execute $\DOAlg(\cv,Z)$ offline and then match the first $c_u$ users of type $u$ according to the optimal matching obtained for $\DOAlg(\cv,Z)$. However, this approach suffers from two drawbacks. First, bad events happen in a given phase with a probability of roughly $\nicefrac{2n}{\tau^2}$. Since we have $\nicefrac{T}{\tau}$ phases, we are likely to  experience many such bad events and need to account for them. And second,  $\DOAlg(\cv,Z)$ might be infeasible since it assumes $\sum_u c_u$ users in a phase and not $\tau$; hence, perhaps not all the exposure constraints could be satisfied simultaneously (recall  that Inequality~\eqref{eq:assumption1} ensures that $\tau$ users are enough). To that end, we augment the set of user types $U$ with an additional type we call \textit{slack users} and denote $u^\star$. We extend the utility matrix $\Mmu$ such that for all $a\in K$, $\mu_{u^\star, a}=0$. Slack users are used to ensuring the exposure constraints hold and have no impact on the reward. 

Next, we slightly modify the lower confidence bound vector ${\cv}=(c_1,\dots, c_n)$. Denote by $\cv^+=(c_1,\dots,c_n,c_{n+1})$, where $c_{n+1}=\tau-\sum_{i=1}^n c_i  = n\sqrt{\tau\log\tau}$. I.e., the first $n$ entries of $\cv^+$ are the lower confidence bound of each user type, and the last entry absorbs the slack between the sum of the lower bounds and the phase length. Indeed, this is precisely the number of slack users. Executing $\DOAlg(\cv^+,Z)$, we get a matching satisfying the exposure constraints of the arms in $Z$. The optimal matching $M$, $M\in \mathbb{N}^{(n+1)\times k}$, has now an additional row for the type $u^\star$. Moreover, for every $u\in U \cup \{u^\star\}$, $M(u,\cdot)=c^+_u$. Finally, executing $\DOAlg$ on $\cv^+$ or on a random aggregate vector $c(\q)$ yields a similar reward.
\begin{proposition}\label{prop_short1234322}
For every $ Z\subseteq K$, it holds that  \[\DOAlg(\cv^+,Z)\geq \E_{\q}[\DOAlg(c(\q),Z)]-n\sqrt{\tau\log {\tau}}.\]
\end{proposition}

\subsubsection{The $\OR(Z)$ Procedure}
\begin{algorithm}[t]
\caption{$\OR^\star$ algorithm}\label{alg:lcbapprox234f}
\begin{algorithmic}[1]
 \Statex\textbf{Input}: $U, K, \delt,\tau,T,\Mmu,\ro$
 \State pick $Z^\star \in \argmax_{Z\subseteq K}\DOAlg(\cv^+, Z)$ \label{alglineor:bestZ}
 \State follow $\OR(Z^\star)$ \label{alglineor:followORZ}
\Procedure{$\OR$}{$Z$} \label{alglineor:procORZ}
\State let $M$ be the optimal matching of $\DOAlg({\cv^+},Z)$ \label{alglineor:matchinglcb}
\For {$t=1,\dots,\tau$}\label{alglineor:for}
    \State observe $u_t$ \label{alglineor:observeu}
    \If {$M(u^\star,\cdot) = 0$ and $M(u_t,\cdot)=0$ } \label{alglineor:badevent} 
    \State select $u_t \in \{u\in U: M(u,\cdot) >0  \}$\label{alglineor:changeubad} \Comment{bad event} 
    \ElsIf {$M(u_t,\cdot)=0$}  $u_t \gets u^\star$\label{alglineor:setstar}
    \EndIf
    \State select $a_t \in \left\{a\in Z :M(u_t,a) >0   \right\}$, pull $a_t$ \label{alglineor:at}
    \State $M(u_t,a_t) \gets M(u_t,a_t)-1$\label{alglineor:updateM}
\EndFor
\State go to line~\ref{alglineor:matchinglcb} \label{alglineor:init}\Comment{repeat the process for phases $2,\dots ,\nicefrac{T}{\tau}$}
\EndProcedure
\Function{$\DOAlg$}{$\cv,Z$} \label{alglineor:doalgproc}
\State return the reward-maximizing matching $M$ of the users in $\cv$ to arms in $Z$, satisfying the exposure constraints of arms $Z$ 
\EndFunction
\end{algorithmic}
\end{algorithm}
We now explain the $\OR(Z)$ procedure, which is implemented in Line~\ref{alglineor:procORZ} of Algorithm~\ref{alg:lcbapprox234f}. It gets a subset of arms $Z\subseteq K$ as an input. In Line~\ref{alglineor:matchinglcb} we execute the function $\DOAlg(\cv^+,Z)$ and store the matrix $M$ indicating the optimal offline matching. In Line~\ref{alglineor:for} we enter the for loop of the phase. Throughout the execution of the phase, we use 
$M(u,a)$ to count the number of pulls of arm $a\in Z$ required for the user type $u\in U\cup \{u^\star \}$ under the clean event in the remaining rounds. Similarly, $M(u,\cdot)$ counts the minimal number of type $u$-users we still expect to see under the clean event. Lines~\ref{alglineor:badevent}-\ref{alglineor:changeubad} are devoted to the bad event; we address it later on. In the clean event, we proceed to Line~\ref{alglineor:setstar}. We then ask whether we have already observed at least $c_{u_t}=P_{u_t}\tau-\sqrt{\tau\log{\tau}}$ type $u_t$-users (recall that $M(u_t,\cdot)=c^+_{u_t}$ at the beginning of the phase). In such a case, we treat $u_t$ as a slack user and set $u_t \gets u^\star$. In Line~\ref{alglineor:at} we select an arm $a_t$ from the set of arms for which pulls are still required according to  $\DOAlg(\cv^+,Z)$ and pull it. In the following Line~\ref{alglineor:updateM}, we update the matrix $M$ that $(u_t,a_t)$ was selected, meaning that this pair will be pulled less in the remaining rounds. After the phase terminates, we move to Line~\ref{alglineor:init} and start a new phase. 

The case of the bad event, addressed in Line~\ref{alglineor:badevent}, requires a more careful attention. 
\begin{proposition}\label{prop:badevent_ifonlyif}
Line~\ref{alglineor:changeubad}
is executed if and only if the bad event has occurred in that phase.
\end{proposition}

The proof of Proposition~\ref{prop:badevent_ifonlyif} appears in {\ifnum\Includeappendix=1{Appendix~\ref{subsec: badevent ifonlyif proof}}\else{the appendix}\fi}. Further,
\begin{proposition}\label{prop:DOlalgiscommitted}
$\OR(Z)$ satisfies Properties~\ref{prop_PI12343} and \ref{prop_commited}. 
\end{proposition}
Proposition~\ref{prop:DOlalgiscommitted} guarantees that all the arms in $Z$ remain viable during $\OR(Z)$'s execution almost surely. The proof of Proposition~\ref{prop:DOlalgiscommitted} appears in {\ifnum\Includeappendix=1{Appendix~\ref{subsec: DOlalg is committed proof}}\else{the appendix}\fi}. 

After explaining how  $\OR(Z)$ operates, we are ready to prove its formal guarantees. The next Lemma~\ref{lemma:lcbz and pico z} shows that  $\OR(Z)$ approximates $\PIOPT(Z)$.
\begin{lemma}\label{lemma:lcbz and pico z} For every $Z\subseteq K$ it holds that 
\[\E_{\q}[\tilde{r}_{1:\tau}^{\OR(Z)}] \geq \DOAlg(\cv^+,Z)-\frac{2n}{\tau} \geq
\E_{\q}[\tilde{r}_{1:\tau}^{\PIOPT(Z)}] -\tilde{O}\left(n \sqrt \tau\right).\]
\end{lemma}
\subsubsection{Finding the Optimal Arm Subset}
We are ready to explain $\OR^\star$. In Line~\ref{alglineor:bestZ}, it picks the best subset $Z^\star$ given the arrival sequence $\cv^+$ using the $\DOAlg$ function. Then, in Line~\ref{alglineor:followORZ}, it executes the $\OR$ procedure when initiated with $Z^\star$ it picked in the previous line. 
The next theorem proves the formal guarantees of $\ORS$.
\begin{theorem} \label{thm:ors}
It holds that
\[
\E[\tilde r^{\OR^\star}]  \geq  E[\tilde r^{\PIOPT}]- \tilde{O}\left(\frac{nT}{\sqrt{\tau}}\right).
\]
\end{theorem}
\begin{proofof}{thm:ors}
Notice that
\[
\E[\tilde r^{\OR^\star}] = \E[\tilde r^{\OR(Z^\star)}] = \frac{T}{\tau}\E_{\q}[\tilde{r}_{1:\tau}^{\OR(Z^\star)}],
\]
where the first equality follows from the definition of $\ORS$ in Line~\ref{alglineor:followORZ}, and the second from Equation~\eqref{eq:phaseToT} since $\OR$ is committed and phase-independent due to Proposition~\ref{prop:DOlalgiscommitted}.
Next, Lemma~\ref{lemma:lcbz and pico z} suggests that  
\begin{align*}
\frac{T}{\tau}\cdot\E_{\q}[\tilde{r}_{1:\tau}^{\OR(Z^\star)}] &\geq \frac{T}{\tau} \cdot \left(  \DOAlg(\cv^+,Z^\star)-\frac{2n}{\tau} \right) \geq  \frac{T}{\tau} \cdot \max_{Z\subseteq K} \left(  \DOAlg(\cv^+,Z)-\frac{2n}{\tau} \right) \\
&\geq  \frac{T}{\tau} \cdot \max_{Z\subseteq K}  \left(  \E_{\q}[\tilde{r}_{1:\tau}^{\PIOPT(Z)}] -\tilde{O}\left(n \sqrt \tau\right) \right),
\end{align*}
where the second inequality is due to the definition of $Z^\star$ in Line~\ref{alglineor:bestZ}, and the last inequality is again due to Lemma~\ref{lemma:lcbz and pico z}. By definition of $\PIOPT$ we have $\E_{\q}[\tilde{r}_{1:\tau}^{\PIOPT}] = \max_{Z\subseteq K} \E_{\q}[\tilde{r}_{1:\tau}^{\PIOPT(Z)}] $; thus, 
\begin{align*}
\frac{T}{\tau} \cdot \max_{Z\subseteq K}  \left(  \E_{\q}[\tilde{r}_{1:\tau}^{\PIOPT(Z)}] -\tilde{O}\left(n \sqrt \tau\right) \right) = \frac{T}{\tau} \cdot  \left(  \E_{\q}[\tilde{r}_{1:\tau}^{\PIOPT}] -\tilde{O}\left(n \sqrt \tau\right) \right)= E[\tilde r^{\PIOPT}]- \tilde{O}\left(\frac{nT}{\sqrt{\tau}}\right),
\end{align*}
where the last equality follows from Equation~\eqref{eq:phaseToT}. Altogether, we have shown that
\[
\E[\tilde r^{\OR^\star}]  \geq  E[\tilde r^{\PIOPT}]- \tilde{O}\left(\frac{nT}{\sqrt{\tau}}\right).
\]
\end{proofof}
\subsubsection{Runtime}
To conclude this subsection, note that finding $Z^\star$ in Line~\ref{alglineor:bestZ} of Algorithm~\ref{alg:lcbapprox234f} consists of running the $\DOAlg(\cv^+,Z)$ function for every $Z\subseteq K$; thus, it takes $O(2^k \tau^3 k^3)$. Every round of $\OR(Z^\star)$ takes $O(1)$, so overall $\OR^\star$ takes $O(T+2^k \tau^3 k^3)$. Here too, treating $k$ as a constant we get $O(T+\tau^3)$; furthermore, even if $k=O(\log T)$ this runtime is still useful in practice. 

\section{Learning}\label{sec:learningalg}
\begin{algorithm}[t]
\caption{$\mainAlg$ algorithm}\label{alg:learning}
\begin{algorithmic}[1]
\Statex \textbf{Input}: $U,K, \delt, \tau, T, \SSO$ - a stochastic optimization black box
\State $N(a)\gets 0$ for every $a\in K$\label{alglinees:init}
\For{$t=1,\dots,\tau$}\label{alglinees:forloop}
\If {$\exists a\in K$ s.t. $N(a)\leq max(\delta_a,\gamma\tau)$}\label{alginees:if not enough}
\State pull $a$ , $N(a)\gets N(a)+1$ \label{alglinees:pullexplore1}
\Else \label{alginees:else all were pulled}
\State pull an arm $a$ u.a.r.  ,$N(a)\gets N(a)+1$  \label{alglinees:pullexplore2}
\EndIf
\EndFor
\State go to Line~\ref{alglinees:init} \label{algees:go to first line} \Comment repeat the exploration  in phases $2,\dots,\ceil{\frac{\nicefrac{1}{\gamma}T^{\nicefrac{2}{3}}}{\tau}}$ \label{alglinees:repeat}
\State calculate $\hat{\Mmu},\hat{\ro}$ \label{alglinees:claculs}
\State follow $\SSO(U,K, \delt, \tau,T-\frac{1}{\gamma}\cdot T^{\nicefrac 2 3}, \hat{\ro},\hat{\Mmu})$ in the remaining rounds \label{alglinees:follow}
\end{algorithmic}
\end{algorithm}

In this section, we leverage our previous results for the stochastic optimization task to develop low-regret algorithms. Recall that in the learning task, and unlike in stochastic optimization, the user arrival distribution and the utility matrix, $\ro$ and $\Mmu$, respectively, are unknown. Namely, algorithms have to learn them to maximize the reward $r$.  
In Subsection~\ref{subsec:EESalg}, we present the Explore Exploit and Subsidize meta-algorithm (denoted by $\mainAlg$) and analyze its regret. 
Later, in  Subsection~\ref{subsec:lowerBound} we show that its regret is optimal up to logarithmic factors.

\subsection{Meta Algorithm}\label{subsec:EESalg}
The $\mainAlg$ algorithm is implemented in Algorithm~\ref{alg:learning}. We first explain how it operates and then turn to analyze its performance guarantees. $\mainAlg$ receives the observable instance parameter as an input, along with a stochastic optimization black-box that we call $\SSO$ (later, we replace $\SSO$ with $\DP^\star$ and $\OR^\star$).  $\mainAlg$ has two stages. The first stage is exploration, executed in Lines~\ref{alglinees:init}-\ref{algees:go to first line}. In Line~\ref{alglinees:init} we initiate the arm pull counters for every arm. Then, we enter the for loop in Line~\ref{alglinees:forloop} for $\ceil{\nicefrac{T^{\nicefrac{2}{3}}}{\gamma \tau}}$ phases. The if clause in Line~\ref{alginees:if not enough} asks whether there is an arm that was not pulled at least $\max(\delta_a,\gamma\tau)$ times in this phase. Recall that  $\delta_a$ is the exposure threshold, and $\gamma\tau$ is the amount of arm pulls we require to achieve good parameter estimation (as we explain shortly). In the complementary case, where all arms $a\in K$ were pulled at least $\max(\delta_a,\gamma\tau)$ times each, we pull an arm chosen uniformly at random. Either way, we update the arm pulls counter $N(a)$. When the phase terminates, Line~\ref{algees:go to first line} sends us back to Line~\ref{alglinees:init} to start a new phase. Notice that by the end of the exploration stage, all arms are still viable. To see this, recall that our assumption in Inequality~\eqref{eq:assumption1} suggests that at some point in every phase, all arms $a$ will be pulled at least $\max(\delta_a,\gamma\tau)$ times, which  is by definition greater than the exposure constraint $\delta_a$.

After the exploration stage ends, the exploitation stage begins. Line~\ref{alglinees:claculs} calculates the 
empirical distribution $\hat \ro$ and the utility matrix estimate $\hat \Mmu$. Namely, for all $ u\in U$ and $a\in K$,
\begin{align*}
\hat{\roSub}_u=\frac{1}{T_0}\sum_{t=1}^{T_0}\ind_{[u_t=u]}\quad , \quad
\hat{\MmuSub}_{u,a}=\frac{\sum_{t=1}^{T_0} v_{t} \cdot \ind_{[u_t=u,a_t=a]}}{\sum_{t=1}^{T_0} \ind_{[u_t=u,a_t=a]}},
\end{align*}
where $T_0 = \frac{1}{\gamma}T^{\nicefrac{2}{3}}$. 
We formally prove in {\ifnum\Includeappendix=1{Appendix~\ref{Asubsec:EES regret}}\else{the appendix}\fi} that  
\[
\norm{\ro - \hat \ro}_{\infty} \leq O\left( \sqrt{\nicefrac{\log T}{T^{\nicefrac{2}{3}}}}\right)
\]
holds w.h.p., where $\norm{\cdot}_{\infty}$ is the $L_\infty$ norm. Similarly, we show that $\norm{\Mmu - \hat \Mmu}_{\infty} \leq O(\sqrt{\nicefrac{\log T}{T^{\nicefrac{2}{3}}}})$ holds w.h.p. as well. Ultimately, in Line~\ref{alglinees:follow} we call the stochastic optimization black-box $\SSO$ with the known instance parameters and the estimates $\hat \ro, \hat \Mmu$, and follow it for the next $T-\nicefrac{1}{\gamma}\cdot T^{\nicefrac 2 3}$ rounds. Working towards bounding the regret, we introduce the following stability property. 
\begin{property}[stability]\label{prop:stability}
A stochastic optimization algorithm $A$ is \textit{stable} if for every two user arrivals sequences $\q,\q'\in U^{\tau}$ disagreeing on exactly one entry, it holds that 
$
|\tilde{r}_{1:\tau}^{A}(\q)-\tilde{r}_{1:\tau}^{A}(\q')|= O(1).
$
\end{property}
In words, a stochastic optimization algorithm is stable if 
changing one entry in the type arrival sequence  does not affect the reward much. Clearly, $\DP^\star$ and $\OR^\star$ are stable. We are now ready to analyze the reward of $\mainAlg$ in terms of the stochastic optimization black box $\SSO$.
\begin{theorem}\label{theoremMainRegret1}
Let $\err(\SSO)=\sup\left\{ \E[\tilde{r}^{\PIOPT}]-\E[\tilde{r}^{\SSO}]\right\}$ be the approximation error of $\SSO$ w.r.t. $\PIOPT$. Then, 
\[
\E[r^{\mainAlg(\SSO)}]\geq \E[\tilde{r}^{\SSO}]- \tilde{O}\left((\sqrt{\gamma}n^2+\nicefrac{1}{\gamma})\cdot T^{\nicefrac{2}{3}}\right)-\err(\SSO).
\]
\end{theorem}
Theorem~\ref{theoremMainRegret1} introduces a new quantity for lower-bounding a learning algorithm's reward, denoted by $\err$. For the stochastic optimization algorithms $\ORS$ and $\DPS$ from Section~\ref{sec_stochastic}, Equation~\eqref{eq:optimalDP} and Theorem~\ref{thm:ors} hint that $\err(\DPS)=0$ and $\err(\ORS)=O(\nicefrac{T}{\sqrt{\tau}})$, respectively. On the right-hand side of Theorem~\ref{theoremMainRegret1}, $\E[\tilde{r}^{\SSO}]$ is the reward of $\SSO$ with the actual instance parameters $\Mmu,\ro$; namely, this inequality quantifies the difference between no information about $\Mmu,\ro$ (the left-hand side) and complete information of $\Mmu,\ro$ (the right-hand side). we now prove Theorem~\ref{theoremMainRegret1}.

\begin{proofof}{theoremMainRegret1}
Fix a reward matrix $\Mmu$ and arrival distribution $\ro$. Throughout the proof, we use $\hat \Mmu, \hat \ro$ as the estimates $\mainAlg$ calculates in Line~\ref{alglinees:claculs}. The following proposition suggests that the estimates must be close to the true parameters w.h.p. 
\begin{proposition}\label{shortappendix_propReg1}
Let $ \epsilon_1^u= \sqrt{\frac{\log T}{\roSub_u\cdot T^{\nicefrac{2}{3}} -\sqrt{\roSub_u \cdot T^{\nicefrac{2}{3}}\log T} }}$ for every $u\in U$ and let $ \epsilon_2=\sqrt{\frac{\log T}{\frac{1}{\gamma}\cdot T^{\nicefrac{2}{3}}}}$. It holds that
\begin{align*}
    &\Pr\Big[\exists u\in U, a\in K:|\MmuSub_{u,a}-\hat{\MmuSub}_{u,a}|>\epsilon_1^u \textnormal{ or } |\roSub_u-\hat{\roSub}_u| >\epsilon_2 \Big] \leq \frac{4\cdot k \cdot n+2\cdot n}{T^2}.
\end{align*}
\end{proposition}
The proof of Proposition~\ref{shortappendix_propReg1} and all other formal statements inside the proof of Theorem~\ref{theoremMainRegret1} appear in {\ifnum\Includeappendix=1{Appendix~\ref{Asubsec:EES regret}}\else{the appendix}\fi}. Notice that $\epsilon_1^u=O\left(\sqrt{\nicefrac{\log T}{T^{\nicefrac{2}{3}}}}\right)$ for every $u\in U$ since $\roSub_u$ is a constant for every $u\in U$ (recall our assumption in Subsection~\ref{subsec:parametric}). 
Proposition~\ref{shortappendix_propReg1} helps us bounding the probability of the \textit{clean event}, i.e., sampling rewards and type arrivals such that $\forall u\in U, a\in K:|\MmuSub_{u,a}-\hat{\MmuSub}_{u,a}|\leq\epsilon_1^u$ and $|\roSub_u-\hat{\roSub}_u| \leq\epsilon_2 $. The bad event causes an expected regret of at most $T\cdot \frac{4\cdot k \cdot n+2\cdot n}{T^2} =O(\frac{kn}{T})= O(1)$, so we can ignore it in what follows. 
 
Let $A$ be any stochastic optimization algorithm. For ease of notation, we denote by $V^{\Mmu,\ro}(A)$ the expected reward of $A$ when the instance parameters are $\Mmu,\ro$. I.e., 
\[V^{\Mmu,\ro}(A)=\E_{u_1,\dots,u_T \sim \ro^T}[\sum_{t=1}^T \MmuSub_{u_t,a^A_t}].
\]
Next, Proposition~\ref{shortappendix_propReg2} shows that the reward discrepancy in case we slightly perturb $\Mmu,\ro$ is limited.
\begin{proposition}\label{shortappendix_propReg2}
Let $A$ be any stochastic optimization algorithm satisfying Properties~\ref{prop_PI12343}-\ref{prop:stability}. Then,
\[ \forall \Mmu',\ro': |V^{\Mmu,\ro}(A(\Mmu',\ro'))-V^{\hat{\Mmu},\hat{\ro}}(A(\Mmu',\ro'))|\leq \tilde{O}(\sqrt{\gamma}n^2\cdot T^{\nicefrac{2}{3}}).\]
\end{proposition}
We leverage Proposition~\ref{shortappendix_propReg2} to show the following claim.
\begin{proposition}\label{shortappendix_propReg3}
Let $A$ be any stochastic optimization algorithm satisfying Properties~\ref{prop_PI12343}-\ref{prop:stability}. Then,
\[V^{\Mmu,\ro}(A(\Mmu,\ro))-V^{\Mmu,\ro}(A(\hat{\Mmu},\hat{\ro})) \leq \tilde{O}(\sqrt{\gamma}n^2\cdot T^{\nicefrac{2}{3}})+\err(A).\]
\end{proposition}
Finally, we lower bound $\mainAlg(\SSO)$'s reward using Proposition~\ref{shortappendix_propReg3}.
\begin{align*}
    \E[r^{\mainAlg(\SSO)}]\geq \E\left[\tilde{r}_{\frac{T^{\nicefrac 2 3}}{\gamma}:T}^{\SSO(\hat{\Mmu},\hat{\ro})}\right] &\geq  \E[\tilde{r}_{1:T}^{\SSO(\hat{\Mmu},\hat{\ro})}]-O\left(\nicefrac{1}{\gamma}\cdot T^{\nicefrac{2}{3}}\right) \\& \geq \E[\tilde{r}^{\SSO(\Mmu,\ro)}]- O\left(\left(\nicefrac{1}{\gamma}+\sqrt{\gamma}n^2\right) T^{\nicefrac{2}{3}}\right)-\err(\SSO).
\end{align*}
\end{proofof}

To conclude the analysis, we present the consequences of executing $\mainAlg$ with $\DP^\star$ and $\OR^\star$ from Section~\ref{sec_stochastic}. Inequality~\eqref{eq:optimalDP} hints that $\E[\tilde{r}^{\DP^\star}]=\E[\tilde{r}^{\PIOPT}]$, and Theorem~\ref{theoremOAlg1} bounds the difference between $\E[\tilde{r}^{\PIOPT}]$ and $\E[\tilde{r}^{\OPT}]$. Applying these to Theorem~\ref{theoremMainRegret1},
\[
\E[r^{\mainAlg(\DP^\star)}]\geq \E[\tilde{r}^{\OPT}] -k\cdot \tau - \tilde{O}\left((\sqrt{\gamma}n^2+\nicefrac{1}{\gamma})\cdot T^{\nicefrac{2}{3}}\right) - \err(\DPS).
\]
Since $\err(\DPS)=0$, when rearranging, we get a regret of
\[\E[\R^{\mainAlg(\DP^\star)}] = \E[\tilde{r}^{\OPT}]-\E[r^{\mainAlg(\DP^\star)}] \leq   \tilde{O}\left((\sqrt{\gamma}n^2+\nicefrac{1}{\gamma})\cdot T^{\nicefrac{2}{3}}+k\cdot \tau\right). \]
Using $\OR^\star$ as the stochastic optimization black-box, Theorem~\ref{theoremOAlg1} and Theorem~\ref{thm:ors} suggest that
\[\E[\R^{\mainAlg(\OR^\star)}]\leq  \tilde{O}\left((\sqrt{\gamma}n^2+\nicefrac{1}{\gamma})\cdot T^{\nicefrac{2}{3}}+ k\cdot \tau + \nicefrac{nT}{\sqrt{\tau}}\right).\]

Table \ref{tab:table1} summarizes  the performance of $\mainAlg(\DP^\star)$ and $\mainAlg(\OR^\star)$ under the regime we focus on, i.e., when $n,k$ are constants and $\tau ={O}(T^{\nicefrac{2}{3}})$. In particular, for $\tau = \Theta(T^{\nicefrac{2}{3}})$, the regret of both algorithms is $\tilde O(T^{\nicefrac{2}{3}})$. $\mainAlg(\OR^\star)$ is significantly faster, and is appealing even if $k=\log T$. The runtime improvement of $\OR^\star$ results in inferior regret bound compared to $\DP^\star$, as it has an additional additive factor of $\nicefrac{T}{\sqrt \tau}$. Consequently, when $\tau$ is small its regret is closer to $T$; for instance, if $\tau = T^{\nicefrac{1}{3}}$ then $\E[\R^{\mainAlg(\OR^\star)}]=\tilde O(T^{\nicefrac{5}{6}})$. On the other hand, if $\tau$ is small then a runtime of $\tau^k$ might be feasible and therefore $\mainAlg(\DP^\star)$ should be used instead. 


\begin{table}[t]
  \begin{center}
    \caption{The $\mainAlg$ algorithm: performance and runtime.}
    \label{tab:table1}
    \begin{tabular}{c|c|c} 
       & $\mainAlg(\DP^\star)$ & $\mainAlg(\OR^\star)$ \\
      \hline
      runtime & $O(T+\tau^{k})$ & $O(T+ \tau^3)$ \\
      regret &  $\tilde O( T^{\nicefrac{2}{3}}+\tau)$ & $\tilde O( T^{\nicefrac{2}{3}}+\nicefrac{T}{\sqrt{\tau}}+\tau)$
    \end{tabular}
  \end{center}
\end{table}


\subsection{Lower Bound}\label{subsec:lowerBound}
Having shown that $\mainAlg$ incurs a regret of $\tilde O(T^{\nicefrac{2}{3}})$ for every $\tau \leq O(T^{\nicefrac{2}{3}})$, we now turn to show that such regret is inevitable. Namely, we show that for any algorithm, there exists an instance in which it incurs a regret of $\Omega(T^{\nicefrac{2}{3}})$. 
\begin{theorem}\label{theMainLB}
Fix any $\tau$ such that $\tau = O(T^{\nicefrac{2}{3}})$, and any arbitrary learning algorithm $A$. Then, there exists an instance for which $\E[\R^A]= \Omega(T^{\nicefrac{2}{3}})$. 
\end{theorem}
\begin{proofof}{theMainLB}
We show two indistinguishable instances, and prove that any algorithm incurs a regret of $\Omega(T^{\nicefrac{2}{3}})$ on at least one of them. For both instances $n=2,k=3,\ro=(1-p,p),\delt=(0,p\tau,p\tau)$ for any constant $0<p <\nicefrac{1}{2}$. The instances differ in  their expected rewards. We assume Bernoulli rewards, where Instance 1 has $\Mmu=\left(\begin{smallmatrix}
\nicefrac 1 2 & 0 & 0 \\0 &  \nicefrac {1+\epsilon_1 } {2} & \nicefrac 1 2\end{smallmatrix}\right)$ and Instance 2 has $\Mmu'=\left(\begin{smallmatrix} \nicefrac 1 2 & 0 & 0 \\ 0 &  \nicefrac {1} {2} & \nicefrac {1+\epsilon_1 } {2} \end{smallmatrix}\right)$, where $\epsilon_1=\frac{1}{2\sqrt{2c}}T^{-\nicefrac{1}{3}}$ for a constant $c$ such that $c\cdot T^{\nicefrac{2}{3}}>\tau$. We pick each instance w.p. $\nicefrac{1}{2}$.

Fix a learning algorithm $A$, and let $t^{\star}$ denote the last round of the last phase that all arms are viable (possibly a random variable). For example, if the first arm to depart departs  by the end of phase $i$, then $t^{\star}=(i-1) \cdot \tau$ (the last round in phase $i-1$). We prove the following lemma in {\ifnum\Includeappendix=1{Appendix~\ref{Asubsec:lb regret}}\else{the appendix}\fi}.
\begin{lemma}\label{lemmaLBbeginning}
For any algorithm $A$, $\E[R^A]\geq \Omega(p\cdot t^{\star})$ for both instances.
\end{lemma}
Next, we conduct a case analysis on $t^{\star}$. If $t^{\star}\geq T^{\nicefrac{2}{3}}$, then the regret is $\E[R^A]=\Omega(T^{\nicefrac{2}{3}})$ by Lemma~\ref{lemmaLBbeginning}. Otherwise, an arm departs after $t^{\star}+\tau \leq T^{\nicefrac{2}{3}}+\tau \leq 2\cdot \max( T^{\nicefrac{2}{3}},\tau)$ rounds. If arm $a_1$ is the departing arm, every type $1$-user generates no reward, leading to a regret of $\E[R^A]=\Theta(T)$.

The remaining case is when $t^{\star}< T^{\nicefrac{2}{3}}$ and arm $a_2$ or $a_3$ is the departing arm. Depending on the instance, one of these arms is instrumental in achieving low regret. In Instance 1 with rewards $\Mmu$, arm $a_2$'s departure leads to losing $\nicefrac{\epsilon_1}{2}$ every time a type $2$-user enters the system. The same applies for Instance 2 with rewards $\Mmu'$ and arm $a_3$. The following lemma uses standard KL divergence machinery for indistinguishability that will turn useful for our analysis.
\begin{lemma}\label{lemmaLBtail}
No algorithm can distinguish the two instances with a probability of at least $\nicefrac{3}{4}$ after $2\max( T^{\nicefrac{2}{3}},\tau)$ rounds.
\end{lemma}
As a corollary of Lemma~\ref{lemmaLBtail}, the instrumental arm departs with a probability of at least $\nicefrac{1}{4}$. Otherwise, we could use $A$ for distinguishing the two instances, in contradiction to Lemma~\ref{lemmaLBtail}. Consequently, with a probability of at least $\nicefrac{1}{4}$, $A$ loses $\nicefrac{\epsilon_1}{2}=\Theta(T^{-\nicefrac{1}{3}})$ for every type 2-user in the remaining rounds. As a result, $\E[R^A]\geq \frac{1}{4}\cdot p (T-t^{\star}) \cdot \frac{\epsilon_1}{2}= \Omega(T^{\frac {2} {3}})$.
\end{proofof}
Interestingly, the construction of Theorem~\ref{theMainLB} does not assume uncertainty about $\ro$. Namely, it holds even if the learning algorithm has complete information on the user type distribution $\ro$.

\section{Extensions and Further Analysis}\label{sec:extension}
This section discusses some of our results and relaxes several modeling assumptions.

\subsection{Exhaustive Arm Subset Selection is Inevitable}\label{subsec:pick z}
Both $\DPS$ and $\ORS$ exhaustively iterate over all subsets $Z\subseteq K$ to find $Z^\star$ (see Line~\ref{algline:DPbestZ} of Algorithm~\ref{alg:meranddp}  and  Line~\ref{alglineor:bestZ} of Algorithm~\ref{alg:lcbapprox234f}, respectively). Unfortunately, as we now show, this exhaustive search over $2^k$ elements is inevitable. 

To show this, we present the following task. Assume that we are given the user type arrival sequence $\uv \in U^\tau$ of a phase. Our goal is to maximize the sum of rewards under the commitment constraint: Using only arms for which the exposure constraint is satisfied by the end of the phase. This task is equivalent to a two-stage task, where we first decide on the subset of arms $Z\subseteq K$ to commit to, and then match the users in $\uv$ optimally to arms in $Z$. Notice that this is a deterministic, complete information task. We term this task the \textit{planning task} and remark that \citet{googleOmer20} study this task as a special case. As it turns out, even this deterministic task can require a runtime of $O(2^k)$.
\begin{proposition}\label{prop:NPcompleteness}
    Solving the planning task requires $\Omega(2^k)$ run time for some instances.
\end{proposition}

The proof, which appears in {\ifnum\Includeappendix=1{Appendix~\ref{AppenPlanning}}\else{the appendix}\fi}, uses a reduction from the exact cover problem (an NP-complete problem). 
Since the planning case is computationally hard in terms of $k$, the stochastic variant of finding $\OPT$ (or non-constant approximations to $\OPT$) is also hard.
\citet{googleOmer20} suggested in their work several approximations, e.g., using sub-modular optimization or LP rounding. Those approximations lead to linear regret w.r.t. $\OPT$, since they achieve constant approximations. However, if we require the regret benchmark to run in polynomial time,  those approximations and others could be useful. We consider such approximations in Subsection~\ref{subsec: polynomial runtime alg}.


\subsection{Unknown Exposure Thresholds}\label{sec:thmdelta}
Recall that our model assumes complete information about the exposure thresholds  $\delt$. An obvious question is whether one could relax this assumption. Unfortunately, Proposition~\ref{prop:missingDelta} suggests that uncertainty about $\delt$ implies high regret.
\begin{proposition}\label{prop:missingDelta}
Let $A$ be any algorithm that does not get $\delt$ as input. It holds that $\E[\R^A]= \Omega(T)$.
\end{proposition}
The proof of Proposition~\ref{prop:missingDelta} appears in {\ifnum\Includeappendix=1{Appendix~\ref{Asubsec:missing delta}}\else{the appendix}\fi}. We keep the challenge of different modeling of exposure constraints for future work.
\subsection{Relaxing Assumption~\ref{assumption:gamma}}\label{subsec:relax assumption gamma}
This subsection discusses the consequences of relaxing the Assumption~\ref{assumption:gamma}. Inequality~\eqref{eq:assumption1} serves our previous analysis in two crucial ways. First, it allows us to explore all arms for several phases. And second, it ensures that we can explore fast enough. We now analyze the ramifications of violating it. All the proofs of the claims we mention in this subsection appear in {\ifnum\Includeappendix=1{Appendix~\ref{sec: assumption relaxation}}\else{the appendix}\fi}.

Assume that $\sum_{a \in K}\delta_a > \tau$. In such a case, the lower bound we presented in Subsection~\ref{subsec:lowerBound} increases dramatically. In particular,  
\begin{itemize}
    \item If $\sum_{a \in K}\delta_a \geq \tau +\min_a{\delta_a}$, then the lower bound becomes linear in $T$.
    \item Else, $\sum_{a\in K}\delta_a < \tau +\min_a{\delta_a}$. We show a lower bound of  $\Omega\left(T\cdot \sqrt{\nicefrac{k}{\tau}}\right)$.
\end{itemize}
Interestingly, these two cases have different regret. To see this, let us focus on the arm $a_{\min}$ with the lowest exposure threshold, i.e., $a_{\min} = \argmin_{a} \delta_a$. In the first case, if we pull $a_{\min}$ even once, one of the remaining arms $K\setminus \{a_{\min} \}$ must depart by the end of the first phase since $\sum_{a \in K\setminus \{a_{\min} \}}\delta_a > \tau -1 $. So we either avoid pulling $a_{\min}$ or letting one arm in $K\setminus \{a_{\min} \}$ depart. In the second case, however, we can explore all arms at least once before deciding which exposure constraints to satisfy.


We also investigate cases where fast exploration is not guaranteed. Namely, we have that $\sum_{a=1}^{k}\delta_a \leq \tau$ but a constant $\gamma$ such that $\sum_{a\in K} \max(\delta_a, \gamma \cdot \tau) \leq \tau$ does not exist. Consider the following weaker version of 
Assumption~\ref{assumption:gamma}: We assume there exists a function $f=f(\tau)$ such that $\sum_{a\in K} \max(\delta_a, f(\tau)) \leq \tau.$
Notice that $f(\tau)$ can be any positive function, e.g., $O(1)$ or $\sqrt{\tau}$. Moreover, we obtain the special case of Assumption~\ref{assumption:gamma} if   $f(\tau)=\gamma\tau$ for a constant $\gamma$. For this relaxation, we show a modified version of Theorem~\ref{theoremMainRegret1}, suggesting a regret of $O(\nicefrac{\tau}{f(\tau)}T^{\nicefrac{2}{3}})$. We also describe how to modify $\mainAlg$ slightly to achieve a better regret of $O((\nicefrac{\tau}{f(\tau)})^{\nicefrac{1}{3}}T^{\nicefrac{2}{3}})$.

\subsection{Longer Phase Length}\label{subsec: relax assumption tau is short}
In Assumption~\ref{assumption:tau}, we limited the phase length to be $\tau=O(T^{\nicefrac{2}{3}})$. We exploited this fact in Theorem~\ref{theoremOAlg1}, which promises that algorithms satisfying Properties~\ref{prop_PI12343} and \ref{prop_commited} incur a factor of $k\tau$ in their reward. The following Proposition~\ref{prop: tight LB pico opt} asserts that focusing on phase-independent and committed algorithms can lead to a reward decrease of at least $\tau$ in some instances.
\begin{proposition}\label{prop: tight LB pico opt}
There exists an instance such that $\E[\tilde{r}^{\OPT}]-\E[\tilde{r}^{\PIOPT}]=\Omega(\tau)$.    
\end{proposition}
For long phase length, namely $\tau=\omega(T^{\nicefrac{2}{3}})$, Proposition~\ref{prop: tight LB pico opt} conveys a negative result.  It suggests that executing $\mainAlg$ with any stochastic optimization algorithm $Alg$ that satisfies Properties~\ref{prop_PI12343} and \ref{prop_commited} as a black box yields a regret of $\E[\R^{\mainAlg(Alg)}]=\Omega(\tau)=\omega(T^{\nicefrac{2}{3}})$.

Instead, we develop a different approach for cases where $\tau=\omega(T^{\nicefrac{2}{3}})$. We devise $\LPOR$, which extends $\ORS$. $\LPOR$ is not phase-independent and not committed, and can thus avoid the $\tau$ approximation factor that Proposition~\ref{prop: tight LB pico opt} hints. Furthermore, we prove that the regret of $\mainAlg$ with $\LPOR$  as a black-box is still $\tilde{O}(T^{\nicefrac{2}{3}})$. From the lower bound perspective, we show a bound of $\Omega(\nicefrac{T}{\sqrt{\tau}})$ for the $\tau=\omega(T^{\nicefrac{2}{3}})$ regime. Tightening the lower bound and developing optimal algorithms that settle the gap for this regime are left for future work. All the proofs of the insights mentioned in this subsection appear in {\ifnum\Includeappendix=1{Appendix~\ref{sec: long phase reg}}\else{the appendix}\fi}.

\subsection{Polynomial Runtime Algorithm}\label{subsec: polynomial runtime alg}
In Section~\ref{sec:learningalg}, we present a low-regret algorithm for the learning task that exhibits a runtime factor of $\Omega(2^k)$, which is unrealistic for large values of 
$k$. However, Proposition~\ref{prop:NPcompleteness} implies that exponential runtime is not a virtue of our algorithms but rather an inherent property of the problem: $\OPT$ and even $\PIOPT$ incur a $\Omega(2^k)$ runtime factor. The idea that efficient and low-regret algorithms are compared against inefficient benchmarks is cumbersome and unfair. Consequently, in this subsection, we adopt a more lenient regret benchmark, namely a $(1-\nicefrac{1}{e})$ approximation to $\OPT$. Such a relaxed benchmark allows us to develop algorithms with both low-regret and efficient runtime.

We devise an algorithm tailored for the learning task with a runtime of $O(T+\tau^3k^5)$. This algorithm yields a $\tilde{O}(T^{\nicefrac{2}{3}}+\tau+\nicefrac{T}{\sqrt{\tau}})$ regret with respect to the aforementioned weaker benchmark. Elaborated details of the algorithm and its comprehensive analysis are available in {\ifnum\Includeappendix=1{Appendix~\ref{subsec: robust_approx_appen}}\else{the appendix}\fi}. The foundation of this algorithm and its analysis are mainly based on two prior works:
\begin{itemize}
    \item \citet{googleOmer20} that investigate the planning task as a special case and establish its sub-modularity; and
    \item \citet{nie2023framework}, which introduce methods for constructing low regret algorithms for various computationally hard problems. Notably, this includes the task of maximizing sub-modular functions while adhering to cardinality constraints.
\end{itemize}

\subsection{The $\DOAlg$ Algorithm}\label{subsec: match algorithm}
In this subsection, we explain the computation of $\DOAlg(c(\uv),Z)$ function (see Subsection~\ref{subsec:lcb_approx}). We first construct a bipartite graph $G(V,E)$ where $V=\call \cup \calr$.
The left side contains user nodes: Each user type $u\in U$ is represented by $c(\uv)_u$ identical nodes in $\call$, and $\abs{\call}=\tau$. The right side $\calr$ comprises of arm duplicates: Each arm $a\in Z$ is represented by $\tau$ identical nodes in $\calr$ ($\abs{\calr}=\abs{Z}\cdot\tau\leq k\tau$). We can name the nodes in $\calr$ such that $\calr=\{x^a_i: 1\leq i\leq \tau, a\in Z\}$. There is an edge between every node in $\call$ and every node in $\calr$, i.e., the bipartite graph is complete.

The weight of an edge between a type $u$-user node $y\in \call$ and a  node of arm $a$ $x^a_i$ is $2+\MmuSub_{u,a}$ if $i$ is among the first $\delta_a$ copies of $a$ in $\calr$, and $\MmuSub_{u,a}$ otherwise. Formally,
\[
w(y,x^a_i) = 
\begin{cases}
2+\MmuSub_{u,a} & i\leq \delta_a \\
\MmuSub_{u,a}    & o.w.
\end{cases}.
\]
The final step is to perform maximum weighted bipartite matching on $G$ via, e.g.,  the Hungarian algorithm. $G$ has $\tau+k\tau$ nodes, so the run time of $\DOAlg(c(\uv),Z)$ is $O(k^3\tau^3)$. It only remains to prove optimality and validity. Namely, showing that the resulting match maximizes the phase reward, and is valid, i.e., each user is matched to exactly one arm, and each arm is matched at least $\delta_a$ times.

Let us start with validity. Each user is represented by one node in $\call$; thus, they are matched exactly once (by the correctness of the maximum weighted bipartite matching). In addition, each arm in $a\in Z$ is matched at least $\delta_a$ times. Otherwise, there is a better match since $\forall u\in U, a\in K: 0\leq \MmuSub_{u,a}\leq 1$ and the first $\delta_a$ copy nodes of every arm have a higher weight. We do not match arms $a\in K\setminus Z$ at all. Therefore, the resulting match of $\DOAlg(c(\q),Z)$ is valid. The optimality of the matching also implies the resulting match maximizes the reward among all the valid matches given $Z$. Since otherwise, there is a match in the graph with a higher weight, which is a contradiction.

\section{Discussion}\label{secDiscussion}
We present a MAB setting where arms have exposure constraints, i.e., arms can abandon the system unless they receive enough pulls per phase. In Section~\ref{sec_stochastic} we study the stochastic optimization task, where user arrival sequences are sampled from a known distribution and  devised  $\DP^\star$ and $\OR^\star$. Then, Section~\ref{sec:learningalg} proposes $\mainAlg$, which is a meta-algorithm for the learning task. Combined with  $\DP^\star$ or $\OR^\star$, $\mainAlg$ achieves low regret.
Our work is among the first to introduce strategic behavior via exposure constraints in the face of uncertainty. As such, there are several future directions. 

First, from a technical point of view, our algorithms have optimal regret up to logarithmic factors under the assumptions of Section~\ref{subsec:parametric}. Obtaining optimal algorithm under different assumptions and other regimes, for example, if $\tau=\Omega(T^{\nicefrac{2}{3}})$ (as seen in Subsection~\ref{subsec: relax assumption tau is short}) is left for future work.

Second, our model relies on the assumption that the exposure thresholds of the providers are publicly known. Note this is a rather strong assumption that does not apply in most real-world cases. Regrettably, as demonstrated in Subsection~\ref{sec:thmdelta}, any algorithm lacking knowledge of these thresholds incurs a linear regret. Future research could involve exploring other models that account for uncertainty surrounding the thresholds. For instance, one could consider a mechanism design model in which arms report their thresholds to the system, potentially manipulating and misreporting. In such a model, the system might pursue both truthfulness and the optimization of user welfare.
 
Finally and more conceptually, studying other exposure constraints is a promising direction. Our constraints are hard and require minimal periodic exposure, but constraints  that are soft, e.g., an arm departs with a probability that decreases with the amount of exposure it receives are also realistic and deserve future attention.

\bibliography{sample-base}
{
\appendix
{\ifnum\Includeappendix=1{
}

\section{Proofs}

\subsection{Proof of Theorem~\ref{the_12321}}\label{subsec: proof thm 1}
\begin{proof}
We have two cases, one for algorithms that always subsidize all arms and one for algorithms that never subsidize. We start with the former and continue with the latter.

Consider the instance with one agent type $n=1$, two arms $k=2$, phase length of $\tau=100$, exposure thresholds of $\delt=(10,20)$, and expected reward of $\Mmu=(1,0).$ Trivially, $\OPT$ pulls arm $a_1$ only and $\tilde{r}^{\OPT}=T$. However, any algorithm $A$ that subsidizes arm $a_2$, pulls it at least $20$ rounds in a phase. So, $r^A \leq 0.8 T$ and therefore, $\R^A\geq 0.2T=\Omega(T)$.

If an algorithm $A$ never subsidize, consider the instance in Example~\ref{example1331} (see Subsection~\ref{subsec:armsubsidy}). $A$ knows that $a_1$ is better than $a_2$; thus, $\mu_{1,1}-\mu_{1,2}=1$, it only pulls $a_1$ for type 1-users (recall that $A$ does not subsidize). Consequently, arm $a_2$ gets $50<\delta_2=60$ pulls every round in expectation; thus, from concentration inequalities, arm $a_2$ departs after round $\nicefrac{T}{2}$ w.h.p. To conclude, $A$ incures a regret of at least $\nicefrac{T}{2}\cdot \nicefrac{1}{2}= \Theta(T)$.
%
\end{proof}

\subsection{Proof of Theorem~\ref{theoremOAlg1}}\label{Asubsec:pico appr opt}
\begin{proof}
We assume w.l.o.g. that all the algorithms are deterministic.
We begin by setting the following notations. For every phase $i, 1\leq i \leq \nicefrac{T}{\tau}$ and user arrival history $\p\in U^{(i-1)\tau}$ we define 
\[\R_{i,\p}=\left\{\q\in U^{\tau}: K^{\OPT(\p)}_{(i-1)\tau}=K^{\OPT(\p\oplus\q)}_{i\tau}\right\},\quad S_{i,\p}=U^{\tau}\setminus R_{i,\p},
\] where $\oplus$ is the concatenation operator. In other words, $R_{i,\p}$ is the set of all user arrival sequences $\q\in U^{\tau}$ under which no arm departs at the end of phase $i$ when $\OPT$ is executed on the type arrival $\p\oplus\q$. We further simplify the notation and denote by $\tilde{r}^A_{i,\p}(\q)=\tilde{r}^{A(\p\oplus\q)}_{(i-1)\tau:i\tau}$ the reward of an algorithm $A$ in phase $i$, where the history is $\p$ and the user arrival sequence in phase $i$ is $\q$.

Next, observe that the expected reward of an algorithm $A$ after $T$ rounds is equivalent to $\nicefrac{T}{\tau}$ times its expected reward in a random phase; namely, 
\begin{align*}
    \E[\tilde{r}^{A}]=\frac{T}{\tau}\cdot \E_{i\sim UNI(1,\nicefrac{T}{\tau}),\p\sim \ro^{(i-1)\cdot \tau}} \left[\E_{\q}[\tilde{r}_{i,\p}^{A}(\q)]\right].
\end{align*}
Since $U^{\tau}=R_{i,\p}\cup S_{i,\p}$, we can expand the expectation $E_{\q}[\tilde{r}_{i,\p}^{A}(\q)]$ and get
\begin{align}\label{eq:sdjsth5}
    \E[\tilde{r}^{\OPT}]= \frac{T}{\tau}
    \E_{i\sim UNI(1,\nicefrac{T}{\tau}),\p\sim \ro^{(i-1)\cdot \tau}} \left[\sum_{\q\in R_{i,\p}} \tilde{r}_{i,\p}^{\OPT}(\q)\cdot \Pr[\q] + \sum_{\q\in S_{i,\p}} \tilde{r}_{i,\p}^{\OPT}(\q)\cdot \Pr[\q]\right].
\end{align}
Next, we prove Propositions~\ref{appendix_shortsition1} and \ref{appendix_shortsition2}. We use these propositions to bound the components in the expectation in Equation~\eqref{eq:sdjsth5}.
\begin{proposition}\label{appendix_shortsition1}
For all $i,\p$ such that $1\leq i \leq \nicefrac{T}{\tau} ,\p\in U^{(i-1)\tau} $, it holds that
\[
\sum_{\q\in R_{i,\p}} \tilde{r}^{\OPT}_{i,\p}(\q)\cdot \Pr[\q] \leq \E_{\q}[\tilde{r}_{1:\tau}^{\PIOPT}].
\]
\end{proposition}

\begin{proposition}\label{appendix_shortsition2}
$\E_{i\sim UNI(1,\nicefrac{T}{\tau}),\p\sim \ro^{(i-1)\tau}}\left[ \sum_{\q\in S_{i,\p}} \Pr[\q]\right] \leq \frac{k}{\nicefrac{T}{\tau} }.$
\end{proposition}

The proofs of Propositions~\ref{appendix_shortsition1} and~\ref{appendix_shortsition2} appear after this proof. Applying the fact that $\tilde{r}_{i,\p}^{\OPT}(\q)\leq\tau$ and Propositions~\ref{appendix_shortsition1} and~\ref{appendix_shortsition2} to Equation~\eqref{eq:sdjsth5} we conclude that
\begin{align*}
     \E[\tilde{r}^{\OPT}] &\leq \frac{T}{\tau}\cdot \E_{i\sim UNI(1,\nicefrac{T}{\tau}),\p\sim \ro^{(i-1)\cdot \tau}} \left[\E_{\q}[\tilde{r}_{1:\tau}^{\PIOPT}]+ \sum_{{\q}\in S_{i,\p}} \tau\cdot\Pr[\q]\right] \\
    &\leq \frac{T}{\tau}\cdot \E_{\q}[\tilde{r}_{1:\tau}^{\PIOPT}] + T\cdot \frac{k}{\nicefrac{T}{\tau}}\\
    &= \E[\tilde{r}^{\PIOPT}]+ k\cdot \tau,
\end{align*}
This concludes the proof of Theorem~\ref{theoremOAlg1}.
\end{proof}

\begin{proofof}{appendix_shortsition1}
Assume by contradiction that there exist $i',\p'$ such that $1\leq i'\leq \frac{T}{\tau}, \p'\in U^{(i'-1)\tau}$ and 
\[
\sum_{\q\in R_{i',\p'}}\tilde{r}_{i,\p}^{\OPT}(\q)\cdot \Pr[\q] > E_{\q}[\tilde{r}_{1:\tau}^{\PIOPT}].
\]

Next, we present $\Balg$, an algorithm that satisfies Properties~\ref{prop_PI12343}-\ref{prop_commited} and receives a better reward than $\PIOPT$ (leading to a contradiction).  


\begin{algorithm}[H]
\caption{$\Balg$ algorithm}
\begin{algorithmic}[1]
\For {$1\leq i \leq \nicefrac{T}{\tau}$}\label{alglineB:for1}
\State $\forall a\in K: c_a \gets \delta_a$
\For{$1\leq t \leq \tau$} \label{alglineB:for2}
\State observe $u_{(i-1)\tau+t}$\label{alglineB:observe}
\State $a^{\OPT}\gets a^{\OPT(\p'\oplus (u_{(i-1)\tau+1},\dots,u_{(i-1)\tau+t }))}_{(i-1)\tau+t }$\label{alglineB:optArm}
\If {$ \exists g\subseteq Z' \text{ s.t. } a^{\OPT} \notin g \text{ and }\sum_{a\in g} c_a = \tau-t+1$}\label{alglineB:if}
\State pick $a\in g$ arbitrarily \label{alglineB:subsudypart}
\Else \quad pick $a\gets a^{\OPT}$ \label{alglineB:mimic}
\EndIf
\State pull $a$, $c_a \gets c_a -1$ \label{alglineB:pullandupdate}
\EndFor
\EndFor
\end{algorithmic}
\end{algorithm}

$\Balg$ satisfies Property~\ref{prop_PI12343} since it operates in phases while restarting its memory in every phase (Lines~\ref{alglineB:for1}-\ref{alglineB:for2}). In Line~\ref{alglineB:optArm}, $\Balg$ simulates $\OPT$ in phase $i'$ with prefix $\p'$ and the observed types in the current phase.
If pulling $a^{\OPT}$, the arm that $\OPT(\p'\oplus (u_{(i-1)\tau+1},\dots,u_{(i-1)\tau+t}))$ 
would pull, causes the departure of an arm in $g\subseteq Z',$ where $Z'=K^{\OPT(\p')}_{(i-1)\tau}$, then $\Balg$ subsidizes an arm in $g$ (Lines~\ref{alglineB:if}-\ref{alglineB:subsudypart}). 
Otherwise, $\Balg$ picks $a^{\OPT}$ (Line~\ref{alglineB:mimic}).
In Line~\ref{alglineB:pullandupdate}, $\Balg$ pulls the chosen arm and updates $c_a$, the counter of the remaining pulls for satisfying the exposure constraint of arm $a\in Z'$. Notice that $\Balg$ pulls only arms in $Z'$ and subsidizes every arm in $Z'$ if needed therefore, it commits to $Z'$.

Now, observe that
\begin{observation}\label{llary1}
For every $ \q \in R_{i',\p'},$ it holds that $\tilde{r}^{\Balg}_{1:\tau}(\q)=\tilde{r}_{i',\p'}^{\OPT}(\q)$.
\end{observation}
From Observation~\ref{llary1} we conclude that 
\begin{align*}
E_{\q}[\tilde{r}_{1:\tau}^{\Balg}]&=\sum_{\q\in U^{\tau}}\tilde{r}^{\Balg}_{1:\tau}(\q)\Pr[\q] \\
&= \sum_{\q\in R_{i',\p'}}\tilde{r}_{i',\p'}^{\OPT}(\q)\Pr[\q]+ \sum_{\q\in S_{i',\p'}} \tilde{r}^{\Balg}_{1:\tau}(\q)\Pr[\q] \\
&\geq \sum_{\q\in R_{i',\p'}}\tilde{r}_{i',\p'}^{\OPT}(\q)\Pr[\q] \\
&> E_{\q}[\tilde{r}_{1:\tau}^{\PIOPT}].   
\end{align*}

This is a contradiction to the definition of $\PIOPT$. This concludes the proof of Proposition~\ref{appendix_shortsition1}.
\end{proofof}

\begin{proof}[\textnormal{\textbf{Proof of Observation~\ref{llary1}}}]
If the condition in Line~\ref{alglineB:if} is satisfied, then acting like $\OPT$ in phase $i'$ with prefix $\p'$ causes the departure of an arm from $g$ at the end of the phase with probability $1$. Therefore, if the realization is $\q\in U^{\tau}$ and the condition in Line~\ref{alglineB:if} is satisfied in some round, then we know that $\q\in S_{i',\p'}.$ Consequently, for every $ \q \in R_{i',\p'}$, the condition is not satisfied during every round in the phase. Namely, $\Balg$ acts as $\OPT$ and gets the same reward.  
\end{proof}

\begin{proofof}{appendix_shortsition2}
It holds that $\sum_{\q\in S_{i,\p}} \Pr[\q]=\Pr[A_{i,\p}],$ where we define $A_{i,\p}$ as the event that at least one arm departs in phase $i$ when $\OPT$ is executed with the prefix $\p\in U^{(i-1)\tau}$. Since there are $k$ arms, there are at most $k$ phases where arms depart. So, it holds that $\sum_{i=1}^{\nicefrac{T}{\tau}} \ind_{A_{i,\p}}\leq k$ almost surely. We conclude that
\begin{align*}
     E_{i\sim UNI(1,\nicefrac{T}{\tau}),\p\sim \ro^{(i-1)\tau}}\left[ \sum_{\q\in S_{i,\p}} \Pr[\q]\right] &=
    E_{i\sim UNI(1,\nicefrac{T}{\tau}),\p\sim \ro^{(i-1)\tau}}\left[ \Pr[A_{i,\p}]\right] =
    E_{i\sim UNI(1,\nicefrac{T}{\tau}),\pmb{l}\sim \ro^T}\left[ \ind_{A_{i,\pmb{l}}} \right] \\
    &= E_{\pmb{l}\sim \ro^T}\left[ \sum_{i=1}^{\nicefrac{T}{\tau}} \Pr[i] \cdot \ind_{A_{i,\pmb{l}}} \right]  = \frac{1}{\nicefrac{T}{\tau}} E_{\pmb{l}\sim \ro^T}\left[ \sum_{i=1}^{\nicefrac{T}{\tau}} \ind_{A_{i,\pmb{l}}} \right] \\
    &\leq \frac{k}{\nicefrac{T}{\tau}}. 
\end{align*}
\end{proofof}

\subsection{Proof of Theorem~\ref{thm:dp}}\label{Asubsec:thmdp}
\begin{proof}
We denote by $\PIOPT(\His,Z)$ the optimal algorithm that satisfies Property~\ref{prop_PI12343}, commits to $Z\subseteq K$ and has pulled the arms according to the multi-set history $H$ in the previous rounds of the current phase in any arbitrary order. Mathematically, \[\PIOPT(\His,Z)=\argmax_{A\in A(\His)}\E[\tilde{r}_{\abs{\His}+1:\tau}^A],\] where $A(\His)$ represents the set of all the algorithms that maintain Property~\ref{prop_PI12343}, commit to the arm subset $Z\subseteq K$ and pull the arms according to $\His$ in the previous $\abs{\His}$ rounds. For values of $\His$ such that arms depart in $\PIOPT(\His,Z)$ with probability $1$, we say that $\PIOPT(\His,Z)$ is undefined and let $\tilde{r}^{\PIOPT(\His,Z)}=-\infty$. 
Next, we prove Lemma~\ref{lem: MERisPICO appendix}.
\begin{lemma}\label{lem: MERisPICO appendix}
$\forall H,Z: \MER(H,Z)=\E_{\q}[\tilde{r}^{\PIOPT(\His,Z)}_{\abs{H}+1:\tau}].$
\end{lemma}

The proof of Lemma~\ref{lem: MERisPICO appendix} appears after the proof of the theorem. The definition of $\PIOPT(\His,Z)$ and Lemma~\ref{lem: MERisPICO appendix} suggest that given the observed type $u_t$, $\PIOPT(\His,Z)$ pulls an arm from \[\argmax_{a\in Z} (\E_{\q}[\tilde{r}_{\abs{H}+2:\tau}^{\PIOPT(\His\cup \{a\},Z)}]) + \MmuSub_{u_t,a} ) = \argmax_{a\in Z} ( \MER(\His \cup \{a\},Z) + \MmuSub_{u_t,a} );\] therefore, the arms that $\DP(Z)$ and $\PIOPT(\His,Z)$ pull yield the same expected reward. Namely, $\E_{\q}[\tilde{r}_{1:\tau}^{\DP(Z)}]=\E_{\q}[\tilde{r}_{1:\tau}^{\PIOPT(\emptyset,Z)}]$.
In addition, Inequality~\eqref{eq:assumption1} suggests that $\PIOPT(\emptyset,Z)$ is well defined w.r.t. $\His=\emptyset$; hence $\E_{\q}[\tilde{r}_{1:\tau}^{\PIOPT(\emptyset,Z)}] = \E_{\q}[\tilde{r}_{1:\tau}^{\PIOPT(Z)}]$. Finally, $\E_{\q}[\tilde{r}_{1:\tau}^{\DP(Z)}]=\E_{\q}[\tilde{r}_{1:\tau}^{\PIOPT(Z)}]=\MER(\emptyset,Z)$.
\end{proof}

\begin{proofof}{lem: MERisPICO appendix}
We prove the lemma by induction on the phase histories cardinality $t$, $0\leq t\leq \tau$. 

\textbf{Base case $t=\tau$.} If $\His$ satisfies the exposure constraints, then $\MER(H,Z)=\E_{\q}[\tilde{r}_{\abs{H}+1:\tau}^{\PIOPT(\His,Z)}]=0$. Otherwise, $\E_{\q}[\tilde{r}_{\abs{H}+1:\tau}^{\PIOPT(\His,Z)}]=\MER(H,Z)=-\infty$. So in the base case, it holds that $\MER(H,Z)=\E_{\q}[\tilde{r}_{\abs{H}+1:\tau}^{\PIOPT(\His,Z)}]$.

\textbf{Induction assumption.}
Assume that $\MER(\His,Z)$ returns the remaining expected reward of $\PIOPT(\His,Z)$ for every subset $Z\subseteq K$ and history $\His: |\His|\geq t$. We show that it also holds when $|\His|=t-1$.

\textbf{Induction step.}
Let $u\in U$ be the user type entering the system in round $t$. If $\PIOPT(\His,Z)$ pulls an arm $a\in K$ for $u$, then from the induction assumption it holds that $\MER(\His \cup \{a\})$ is the expected reward of $\PIOPT(\His,Z)$ in the remaining rounds (since $\abs{H\cup \{a\}}=t$). Because $\PIOPT(\His,Z)$ is optimal, it pulls an arm in $\argmax_{a\in Z} ( \MER(\His \cup \{a\}) + \MmuSub_{u,a} )$; hence, 
\[\E_q[r_{\abs{H}+1:\tau}^{\PIOPT(\His,Z)}]=  \max_{a\in Z} ( \MER(\His \cup \{a\}) + \MmuSub_{u,a}).\]
Since $u$ is unknown prior to round $t$, then 
\[
\E_q[r_{\abs{H}+1:\tau}^{\PIOPT(\His,Z)}]= \E_{u\sim\ro}\left[ \max_{a\in Z} ( \MER(\His \cup \{a\}) + \MmuSub_{u,a})\right] = \MER(\His,Z).
\]
To conclude, we have shown that $\MER(\His,Z)$ returns the expected reward of $\PIOPT(\His,Z)$ in the remaining rounds.
\end{proofof}

\subsection{Proof of Proposition~\ref{MERruntime}}
\begin{proof}
We search over $2^k$ arm subsets, representing possible values of $Z$, in Line~\ref{algline:DPbestZ}. For each arm subset, we construct all possible histories. For every history $H$, we take an expectation over $n$ user types and compute the maximum over $\abs{Z} \leq k$ arms (Line~\ref{algline:MERreturn}), resulting in $kn$ steps. Consequently, the runtime is $O(2^k \cdot k n  \cdot \mathcal{H})$, where $\mathcal H$ denotes the number of histories. 

We are left to bound $\mathcal H$ to complete the analysis. Every history of length $x$, $0\leq x \leq \tau$ is a multiset of arms of cardinality $x$; therefore, using the elementary stars and bars combinatorial theorems, we get that $
\mathcal H \leq \tau \cdot {\tau+k-1 \choose k-1} \leq \tau\cdot (\tau+k)^{k-1} $. Altogether, the runtime is
$
O(2^k \cdot k n  \cdot \mathcal{H})=O\left(2^k \cdot k n  \cdot \tau \cdot (\tau+k)^{k-1} \right).
$
\end{proof}

\subsection{Proof of Proposition~\ref{prop_short1234322}}\label{subsec: proof cplus is like cq}
\begin{proof}
For every $\q\in U^{\tau}$ that satisfies  $\forall u\in U : |x_u(\q)-\roSub_u \cdot \tau|\leq \sqrt{\tau \log \tau}$, it holds that $\DOAlg(c(\q),Z)-\DOAlg(\cv^+,Z) \leq n\sqrt{\tau \log \tau}$.
Notice that $\DOAlg(\cv^+,Z)$ loses at most $1$ for every slack user. Further, for non-slack users it can act the same as $\DOAlg(c(\q),Z)$.

From Theorem~\ref{myHoeffding}, it holds that $\Pr[|x_u-\roSub_u\cdot \tau|>\sqrt{\frac{\log \tau}{\tau}}\cdot \tau=\sqrt{\tau \log \tau}]\leq \frac{2}{\tau^2}$.
Applying the union bound, we get $\Pr[\exists u\in U : |x_u-\roSub_u\cdot \tau|>\sqrt{\tau \log \tau}]\leq \frac{2n}{\tau^2}$.
Therefore,
\[\E_{\q}[\DOAlg(c(\q),Z)]\leq (1-\frac{2n}{\tau^2})\cdot (\DOAlg(\cv^+,Z)+ n\sqrt{\tau \log \tau}) + \frac{2n}{\tau^2}\cdot \tau \leq \DOAlg(\cv^+,Z)+ n\sqrt{\tau \log \tau}.\]
\end{proof}

\subsection{Proof of Proposition~\ref{prop:badevent_ifonlyif}
}\label{subsec: badevent ifonlyif proof}
\begin{proof}
In every round that a type $u_t$-user enters the system and $M(u_t,\cdot)=0$, we treat $u_t$ as $u^\star$ (Line~\ref{alglineor:setstar}).  Notice that $M(u^\star,\cdot)=n\sqrt{\tau\log\tau}$ at the beginning of every phase. Further, every round that invokes the else clause in Line~\ref{alglineor:setstar} decreases $M(u^\star,\cdot)$ by $1$ (Line~\ref{alglineor:updateM}). So, we execute Line~\ref{alglineor:changeubad} if and only if the expression $M(u_t,\cdot)=0$ holds for more than $n\sqrt{\tau\log\tau}$  rounds in a phase (due to the if clause in Line~\ref{alglineor:badevent}).

We conclude from the fact that $\forall u\in U: M(u,\cdot)=c_u$ at the beginning of every phase that the number of rounds that $M(u,\cdot)=0 $ holds is equal to $\max(0,c(\q)_{u}-c_{u})$.
I.e., \[\sum_{u\in U} \max(0,c(\q)_{u} -c_{u}) > n\sqrt{\tau\log\tau}=\sum_{u\in U} c(\q)_{u} -c_{u}\iff \textnormal{ we execute Line~\ref{alglineor:changeubad}.}\]
Finally, the bad event was occurred only if exists $u\in U$ such that $c(\q)_{u} -c_{u}<0$ (by definition of the bad event). We conclude that   
\begin{align*}
    \exists u\in U \text{ s.t. } c(\q)_{u} -c_{u}<0 &\iff  \sum_{u\in U} \max(0,c(\q)_{u} -c_{u}) > \sum_{u\in U} c(\q)_{u} -c_{u} \\& \iff \textnormal{Line~\ref{alglineor:changeubad} was executed.}
\end{align*}
\end{proof}

\subsection{Proof of Proposition~\ref{prop:DOlalgiscommitted}}\label{subsec: DOlalg is committed proof}
\begin{proof}
It is phase-independent since it operates in phases while restarting its memory in every phase. To see that it is committed, recall that in every round, $\OR(Z)$ picks a pair $(u,a)$ for which $M(u,a)>0$ holds and pulls $a$ (Line~\ref{alglineor:at}). Later, in Line~\ref{alglineor:updateM} it decreases $M(u,a)$ by $1$. From Lines~\ref{alglineor:badevent}-\ref{alglineor:setstar} and the fact that $M(\cdot,\cdot)=\tau$ in the beginning of each phase, we conclude that $\OR(Z)$ can always pick a pair $(u,a)$ for which $M(u,a)>0$ holds.
Since every iteration executes Line~\ref{alglineor:updateM} and decreases $1$ from a positive entry in $M$ for $\tau$ rounds, by the end of the phase, $M(u,a)=0$ for every pair $(u,a)$. We further denote by $N(a)=M(\cdot,a)$ at the beginning of each phase. $\OR(Z)$ pulls each arm $a\in Z$ for exactly $N(a)$ rounds in each phase. This is true because by the end of the phase $M(\cdot,a)=0$, and since $\OR(Z)$ pulls $a$ in a round if and only if $M(\cdot,a)$ decreases by $1$ at the end of that round. Finally, by definition of $M$ it holds that $\DOAlg(\cv^+,Z)$ also pulls $a\in Z$ exactly $N(a)$ times. $\DOAlg(\cv^+,Z)$ commits to $Z$ by definition; hence, $\OR(Z)$ also commits to $Z$.
\end{proof}

\subsection{Proof of Lemma~\ref{lemma:lcbz and pico z}}\label{Asubset:lcbz and pico z}
\begin{proof}
The proof relies on the following Proposition~\ref{prop_short1234321}  that we prove after this proof.
\begin{proposition}\label{prop_short1234321}
For every $ Z\subseteq K$, it holds that
\[\E_{\q}[\tilde{r}_{1:\tau}^{\OR(Z)}({\q})]\geq \DOAlg(\cv^+,Z)-\frac{2n}{\tau}.\]
\end{proposition}
Next, the way we define $\DOAlg$ ensures that for every algorithm $A$ that commits to $Z$, 
\begin{equation}\label{eq:match and realizations}
\forall \q\in U^{\tau}: r_{1:\tau}^{A}(\q)\leq \DOAlg(c(\q),Z).    
\end{equation}
Combining Propositions~\ref{prop_short1234322} and \ref{prop_short1234321} with Inequality~\eqref{eq:match and realizations}, we conclude that
\begin{align*}
     \E_{\q}[\tilde{r}_{1:\tau}^{\PIOPT(Z)}] &\leq \E_{\q}[{\DOAlg(c(\q),Z)}] \leq \DOAlg(\cv^+,Z)+n\sqrt{\tau\log\tau}
     \nonumber \\& \leq \E_{\q}[\tilde{r}_{1:\tau}^{\OR(Z)}]+ n\sqrt{\tau\log\tau}+\frac{2n}{\tau} = \E_{\q}[\tilde{r}_{1:\tau}^{\OR(Z)}]+\tilde{O}(n\sqrt{\tau}).
\end{align*}
This completes the proof of the lemma.
\end{proof}
\begin{proofof}{prop_short1234321}
Denote by $x_u$ the number of times that a user of type $u\in U$ enters the system in the first phase.
If $\forall u\in U : |x_u-\roSub_u \cdot \tau|\leq \sqrt{\tau \log \tau}$ then $\OR(Z)$ acts like $\DOAlg(\cv^+,Z)$ for the non-slack users (stems directly from the algorithm definition). The reward of $\DOAlg(\cv^+,Z)$ from the slack users is $0$. Therefore, for every $\q\in U^{\tau}$ that satisfy  $\forall u\in U : |x_u(\q)-\roSub_u \cdot \tau|\leq \sqrt{\tau \log \tau}$ it holds that $\tilde{r}_{1:\tau}^{\OR(Z)}(\q)\geq \DOAlg(\cv^+,Z)$.

From Theorem~\ref{myHoeffding}, it holds that $\Pr[|x_u-\roSub_u\cdot \tau|>\sqrt{\frac{\log \tau}{\tau}}\cdot \tau=\sqrt{\tau \log \tau}]\leq \frac{2}{\tau^2}$.
Applying the union bound, we get $\Pr[\exists u\in U : |x_u-\roSub_u\cdot \tau|>\sqrt{\tau \log \tau}]\leq \frac{2n}{\tau^2}$.

To conclude, $\E_{\q}[\tilde{r}_{1:\tau}^{\OR(Z)}]\geq (1-\frac{2n}{\tau^2})\cdot \DOAlg(\cv^+,Z) + \frac{2n}{\tau^2}\cdot 0 \geq \DOAlg(\cv^+,Z) -\frac{2n}{\tau} $.
\end{proofof}

\subsection{Proof of Theorem~\ref{theoremMainRegret1}}\label{Asubsec:EES regret}
The proof of Theorem~\ref{theoremMainRegret1} appears in Subsection~\ref{subsec:EESalg}.  We now prove the auxiliary propositions.
\begin{proofof}{shortappendix_propReg1}
The main tool we use in this proof is Hoeffding inequality. For completness, we include the versions we use here.
\begin{theorem}[Hoeffding inequality]\label{myHoeffding}
Let $X$ be a random variable such that $X\sim Bin(n,p) .$ Then it holds that \[\Pr \left[|X-\E[X]|>\epsilon \cdot n \right] \leq \exp(-2\cdot \epsilon^2\cdot n).\]
\end{theorem}
\begin{theorem}[Hoeffding inequality second version ]\label{myHoeffding2}
Let  $X_1, X_2,\dots ,X_n $ be mutually independent random variables, but not necessarily identically distributed. Further, assume that $\forall i, X_i\in\left[0,1\right]$ almost surely. Let $\bar{X_n}=\frac{x_1+\dots+x_n}{n}$, and let $\mu_n=E\left[\bar{X_n}\right]$.  So, \[
\Pr\left[|\bar{X_n}-\mu_n|> \sqrt{\nicefrac{\alpha \log T}{n}}\right]\leq 2\cdot T^{-2\alpha}.
\]
\end{theorem}

Denote by $x_u$ the number of rounds that type $u$-users enter the system in the exploration stage. 
Therefore, $x_u\sim bin(\frac{1}{\gamma} T^{\nicefrac{2}{3}},\roSub_u)$. From Theorem~\ref{myHoeffding} we conclude that
\[\Pr\left[|\hat{\roSub}_u-\roSub_u|>\sqrt{\frac{\log T}{\frac{1}{\gamma}\cdot T^{\nicefrac{2}{3}}}}\right]= \Pr\left[ |x_u-\E[x_u]|>\sqrt{\frac{\log T}{\frac{1}{\gamma}\cdot T^{\nicefrac{2}{3}}}}\cdot\frac{1}{\gamma}T^{\nicefrac{2}{3}} \right] \leq 2\cdot\exp(-2\log T) = \frac{2}{T^2}.\]

We further denote by $x_{u,a}$ the number of rounds that $a\in K$ is pulled for type $u$-users during the exploration stage. The arm $a$ is pulled at least $\gamma\cdot\tau$ rounds in a phase (see Inequality~\eqref{eq:assumption1} and Lines~\ref{alginees:if not enough}-\ref{alglinees:pullexplore1} in Algorithm~\ref{alg:learning}) and thus, at least $ T^{\nicefrac{2}{3}}$ times during the exploration stage.
In addition, in the exploration stage the pulled arm $a_t$ is independent of the arrival type $u_t$; therefore, $x_{u,a}\sim bin(N,\roSub_u),$ for $N\geq T^{\nicefrac{2}{3}}$.
From Theorem~\ref{myHoeffding} it holds that
\begin{align}\label{eq: hoefding for xua}
\Pr\left[x_{u,a}>\roSub_u \cdot T^{\nicefrac{2}{3}}-\sqrt{\roSub_u \cdot T^{\nicefrac{2}{3}}\log T}\right]\geq 1-2T^{-2}.    
\end{align}

From Theorem~\ref{myHoeffding2} we conclude that
\small
{\thinmuskip=0mu\medmuskip=0mu plus 0mu minus 0mu\thickmuskip=0mu plus 0mu minus 0mu
\begin{align*}
    &\Pr\left[|\hat{\MmuSub}_{u,a}-\MmuSub_{u,a}|>\sqrt{\frac{\log T}{\roSub_u \cdot T^{\nicefrac{2}{3}}-\sqrt{\roSub_u \cdot T^{\nicefrac{2}{3}}\log T} }} \right]   \\
    &= \Pr\left[\abs{\hat{\MmuSub}_{u,a}-\MmuSub_{u,a}}>\sqrt{\frac{\log T}{\roSub_u \cdot T^{\nicefrac{2}{3}}-\sqrt{\roSub_u \cdot T^{\nicefrac{2}{3}}\log T} }} \middle| x_{u,a}>\roSub_u \cdot T^{\nicefrac{2}{3}}-\sqrt{\roSub_u \cdot T^{\nicefrac{2}{3}}\log T} \right]\cdot \Pr\left[x_{u,a}>\roSub_u \cdot T^{\nicefrac{2}{3}}-\sqrt{\roSub_u \cdot T^{\nicefrac{2}{3}}\log T}\right]\\ & +\Pr\left[\abs{\hat{\MmuSub}_{u,a}-\MmuSub_{u,a}}>\sqrt{\frac{\log T}{\roSub_u \cdot T^{\nicefrac{2}{3}}-\sqrt{\roSub_u \cdot T^{\nicefrac{2}{3}}\log T} }} \middle| x_{u,a}\leq\roSub_u \cdot T^{\nicefrac{2}{3}}-\sqrt{\roSub_u \cdot T^{\nicefrac{2}{3}}\log T} \right] \cdot \Pr\left[x_{u,a}\leq \roSub_u \cdot T^{\nicefrac{2}{3}}-\sqrt{\roSub_u \cdot T^{\nicefrac{2}{3}}\log T}\right]  \\
    &\leq \Pr\left[\abs{\hat{\MmuSub}_{u,a}-\MmuSub_{u,a}}>\sqrt{\frac{\log T}{x_{u,a} }} \middle| x_{u,a}>\roSub_u \cdot T^{\nicefrac{2}{3}}-\sqrt{\roSub_u \cdot T^{\nicefrac{2}{3}}\log T} \right]\cdot 1  +1\cdot \Pr\left[x_{u,a}\leq \roSub_u \cdot T^{\nicefrac{2}{3}}-\sqrt{\roSub_u \cdot T^{\nicefrac{2}{3}}\log T}\right] \\
    &\leq 2\cdot T^{-2}\cdot 1 + 1\cdot 2T^{-2} = 4\cdot T^{-2} .
\end{align*}}\normalsize
The first equality above is due to conditioning on the event $\left[x_{u,a}>\roSub_u \cdot T^{\nicefrac{2}{3}}-\sqrt{\roSub_u \cdot T^{\nicefrac{2}{3}}\log T}\right]$. The second inequality is due to  $\Pr[\cdot]\leq 1$. And the third inequality stems directly from Theorem~\ref{myHoeffding2} and Inequality~\eqref{eq: hoefding for xua}. 
To conclude the proof of this proposition, we apply the union bound and get \[ \Pr\left[\exists u\in U, a\in K:|\MmuSub_{u,a}-\hat{\MmuSub}_{u,a}|>\epsilon_1^u\text{ or } |\roSub_u-\hat{\roSub}_u| >\epsilon_2 \right] \leq \frac{4\cdot k \cdot n+2\cdot n}{T^2}.\]
\end{proofof}

\begin{proofof}{shortappendix_propReg2}
We denote by $f_{\ro}:[0,1)\rightarrow U$, the function $f_{\ro}$ such that \[\forall x\in[0,1), u\in U: f_{\ro}(x)=u \iff x\in I^{\ro}_u= \left[\sum_{i=1}^{u-1}\roSub_i,\sum_{i=1}^{u}\roSub_i\right).\]

Next, observe that
\begin{observation}\label{corollary657}
If $x\sim UNI[0,1)$ then, $f_{\ro}(x)\sim \ro$.
\end{observation}
To see this, notice that 
\[\forall u\in U: \Pr_{x\sim UNI[0,1)}[f_{\ro}(x)=u]=\frac{\sum_{i=1}^{u}\roSub_i- \sum_{i=1}^{u-1}\roSub_i}{1} = \roSub_u= \Pr_{u'\sim \ro}[u'=u].\]
From Observation~\ref{corollary657}, we conclude that \[V^{\Mmu,\ro}(A)=\E_{u_1,\dots,u_T\sim \ro}[\sum_{t=1}^T \MmuSub_{u_t,a^A_t}] = \E_{x_1,\dots,x_T \sim UNI[0,1)}[\sum_{t=1}^T \MmuSub_{f_{\ro}(x_t),a^A_t}].\] Moreover,
\begin{observation}\label{coromini1241}
$\Pr_{x\sim UNI[0,1)}\left[f_{\ro}(x)\ne f_{\hat{\ro}}(x)\right]\leq \epsilon_2\cdot n^2$.
\end{observation}
The proof of Observation~\ref{coromini1241} appears at the end of the section. We conclude from Observation~\ref{coromini1241} that given $x_1,\dots,x_T\sim UNI[0,1)^T$, the number of rounds that $f_{\ro}(x_t)\ne f_{\hat{\ro}}(x_t)$ is $\sim bin(T,p)$  for $p\leq\epsilon_2\cdot n^2$. We further denote this r.v. by $l$; namely, $l\sim bin(T,p)$. Theorem~\ref{myHoeffding} suggests that $\Pr[|l-p\cdot T|>\sqrt{T\log T}]\leq 2\cdot T^{-2}$. Therefore, it holds w.h.p. that $l<p\cdot T +\sqrt{T\log T} \leq \tilde{O}(\sqrt{\gamma}n^2\cdot T^{\nicefrac{2}{3}})$ (see the definition of $\epsilon_2$).

We assume w.l.o.g. that $V^{\Mmu,\ro}(A)> V^{\hat{\Mmu},\hat{\ro}}(A)$. Since $A$ is stable and the sequences 
\[
\left(f_{\ro}(x_t)\right)_{t=1}^T,\left(f_{\hat{\ro}}(x_t)\right)_{t=1}^T
\]
differ by exactly $l$ elements (point-wise), it holds that
\begin{align}\label{eq:use the stable}
    \abs{V^{\Mmu,\ro}(A)-V^{\Mmu,\hat{\ro}}(A)}=\left\lvert\E_{x_1,\dots,x_T \sim UNI[0,1)} \left[\sum_{t=1}^T \MmuSub_{f_{\ro}(x_t),a_t}\right]-\E_{x_1,\dots,x_T \sim UNI[0,1)}\left[\sum_{t=1}^T \MmuSub_{f_{\hat{\ro}}(x_t),a_t}\right]\right\rvert \nonumber \\ \leq O(1)\cdot \tilde{O}(\sqrt{\gamma}n^2\cdot T^{\nicefrac{2}{3}}).
\end{align}

We conclude that
\begin{align*}
    V^{\Mmu,\ro}(A)- &V^{\hat{\Mmu},\hat{\ro}}(A)= \E_{x_1,\dots,x_T \sim UNI[0,1)}\left[\sum_{t=1}^T \MmuSub_{f_{\ro}(x_t),a_t}\right]-\E_{x_1,\dots,x_T \sim UNI[0,1)}\left[\sum_{t=1}^T \hat{\MmuSub}_{f_{\hat{\ro}}(x_t),a_t}\right] \\
    &\leq \E_{x_1,\dots,x_T \sim UNI[0,1)}\left[\sum_{t=1}^T \MmuSub_{f_{\ro}(x_t),a_t}\right]-\E_{x_1,\dots,x_T \sim UNI[0,1)}\left[\sum_{t=1}^T \MmuSub_{f_{\hat{\ro}}(x_t),a_t}-\max_{u\in U}\epsilon_1^u\right] \\
    &\leq \tilde{O}(\sqrt{\gamma}n^2\cdot T^{\nicefrac{2}{3}}) + \max_{u\in U}\epsilon_1^u \cdot T  =\tilde{O}(\sqrt{\gamma}n^2\cdot T^{\nicefrac{2}{3}}).
\end{align*}
The first inequality follows from the fact that $\abs{\mu_{u,a}-\hat{\mu}_{u,a}}\leq\epsilon_1^u \leq \max_{u\in U}\epsilon_1^u$. The second inequality follows from Inequality~\eqref{eq:use the stable}, and the last equality from the fact that $\max_{u\in U}\epsilon_1^u = O(\sqrt{\nicefrac{\log T}{T^{\nicefrac{2}{3}}}})$.
Symmetric arguments hint that $|V^{\Mmu,\ro}(A)- V^{\hat{\Mmu},\hat{\ro}}(A)|\leq \tilde{O}(\sqrt{\gamma}n^2\cdot T^{\nicefrac{2}{3}})$. This concludes the proof of Proposition~\ref{shortappendix_propReg2}.
\end{proofof}

\begin{proofof}{shortappendix_propReg3}
We upper bound the expression $V^{\Mmu,\ro}(A(\Mmu,\ro))-V^{\Mmu,\ro}(A(\hat{\Mmu},\hat{\ro})).$ 
From the definition of $err(A)$, it holds that
\begin{align} \label{eq: err reason}
V^{\hat{\Mmu} ,\hat{\ro}}(A(\hat{\Mmu} ,\hat{\ro}) \geq V^{\hat{\Mmu} ,\hat{\ro}}(\PIOPT(\hat{\Mmu} ,\hat{\ro}))-err(A) \geq V^{\hat{\Mmu} ,\hat{\ro}}(A(\Mmu ,\ro)) -err(A).    
\end{align}
Consequently, 
\begin{align*}
    V^{\Mmu,\ro}(A(\Mmu,\ro))-&V^{\Mmu,\ro}(A(\hat{\Mmu},\hat{\ro}))=V^{\Mmu,\ro}(A(\Mmu,\ro))- V^{\hat{\Mmu},\hat{\ro}}(A(\hat{\Mmu},\hat{\ro}))+V^{\hat{\Mmu},\hat{\ro}}(A(\hat{\Mmu},\hat{\ro})) -V^{\Mmu,\ro}(A(\hat{\Mmu},\hat{\ro})) \\& \leq V^{\Mmu,\ro}(A(\Mmu,\ro))- (V^{\hat{\Mmu},\hat{\ro}}(A(\Mmu,\ro))-err(A)) +V^{\hat{\Mmu},\hat{\ro}}(A(\hat{\Mmu},\hat{\ro})) -V^{\Mmu,\ro}(A(\hat{\Mmu},\hat{\ro})) \\& \leq err(A)+\tilde{O}(\sqrt{\gamma}n^2\cdot T^{\nicefrac{2}{3}})+\tilde{O}(\sqrt{\gamma}n^2\cdot T^{\nicefrac{2}{3}}) = err(A)+ \tilde{O}(\sqrt{\gamma}n^2\cdot T^{\nicefrac{2}{3}}).
\end{align*}
The first inequality stems from Inequality~\eqref{eq: err reason} and the second inequality follows from Proposition~\ref{shortappendix_propReg2}. Notice that, Proposition~\ref{shortappendix_propReg3} holds even if we weaken the definition of $\err(A)$  to the following expression \[\err(A)=V^{\hat{\Mmu},\hat{\ro}}(\PIOPT(\hat{\Mmu},\hat{\ro}))-V^{\hat{\Mmu},\hat{\ro}}(A(\hat{\Mmu},\hat{\ro}))\] instead of  \[\err(A)=\sup_{\Mmu',\ro'}\left\{ V^{\Mmu',\ro'}(\PIOPT(\Mmu',\ro'))-V^{\Mmu',\ro'}(A(\Mmu',\ro')) \right\}.\] As a result, Theorem~\ref{theoremMainRegret1} also holds under this weaker definition.
\end{proofof}

\begin{proof}[\textnormal{\textbf{Proof of Observation~\ref{coromini1241}}}]
The law of total probability suggests that
\begin{equation}\label{eq:totalprop21345}
    \Pr_{x\sim UNI[0,1)}[f_{\ro}(x)\ne f_{\hat{\ro}}(x)]=\sum_{u=1}^n \Pr_{x\sim UNI[0,1)} [f_{\hat{\ro}}(x)\ne u | f_{\ro}(x)=u] \cdot \Pr_{x\sim UNI[0,1)}[f_{\ro}(x)=u].
\end{equation}

Now, we consider an arbitrary $u\in U$. From the definition of $f_{\ro}$, we have 
\[\Pr_{x\sim UNI[0,1)} [f_{\hat{\ro}}(x)\ne u | f_{\ro}(x)=u]=\Pr_{x\sim UNI[0,1)}[x \notin I^{\hat{\ro}}_u| x\in I^{\ro}_u]=\Pr_{x\sim UNI[0,1)}[x\in I^{\ro}_u \setminus I^{\hat{\ro}}_u | x\in I^{\ro}_u].\]

Let $\lambda(I)$ be the Lebesgue measure of a measurable set $I$ in $\mathbb{R}$. Next, we prove that $\lambda(I^{\ro}_u \setminus I^{\hat{\ro}}_u) \leq \epsilon_2 (n-1)$ by a case analysis.
\begin{enumerate}
    \item If $I^{\ro}_u \subseteq I^{\hat{\ro}}_u$, then $\lambda(I^{\ro}_u \setminus I^{\hat{\ro}}_u)=0$.
    \item If $I^{\hat{\ro}}_u\subseteq I^{\ro}_u $, then $\lambda(I^{\ro}_u \setminus I^{\hat{\ro}}_u)\leq \epsilon_2$. 
    Because, $\lambda(I^{\ro}_u)=\roSub_u, \lambda(I^{\hat{\ro}}_u)=\hat{\roSub}_u$ and the difference between them is at most $\epsilon_2$.
    \item\label{enum:blabla} If $\sum_{i=1}^{u-1}\roSub_i \leq \sum_{i=1}^{u-1}\hat{\roSub}_i$ and $\sum_{i=1}^{u}\roSub_i \leq \sum_{i=1}^{u}\hat{\roSub}_i$, then $\lambda(I^{\ro}_u \setminus I^{\hat{\ro}}_u)\leq \sum_{i=1}^{u-1}\hat{\roSub}_i - \sum_{i=1}^{u-1}\roSub_i \leq (n-1)\epsilon_2$.
    \item Other cases are symmetric to Case~\ref{enum:blabla}.
\end{enumerate}

Therefore,
{\thinmuskip=0mu\medmuskip=0mu plus 0mu minus 0mu\thickmuskip=0mu plus 0mu minus 0mu
\[\Pr_{x\sim UNI[0,1)} [f_{\hat{\ro}}(x)\ne u | f_{\ro}(x)=u] \cdot \Pr_{x\sim UNI[0,1)}[f_{\ro}(x)=u] = \Pr_{x\sim UNI[0,1)}[x\in I^{\ro}_u \setminus I^{\hat{\ro}}_u | x\in I^{\ro}_u]\cdot \roSub_u \leq \frac{\epsilon_2 (n-1)}{\roSub_u} \cdot \roSub_u.\]}
Overall,  Equation~\eqref{eq:totalprop21345} suggests that $\Pr_{x\sim UNI[0,1)}[f_{\ro}(x)\ne f_{\hat{\ro}}(x)]\leq \sum_{u=1}^n \epsilon_2 (n-1) \leq \epsilon_2 \cdot n^2$.
\end{proof}

\subsection{Proof of Theorem~\ref{theMainLB} and Theorem~\ref{theMainLB2}}\label{Asubsec:lb regret}
\begin{proof}
The proof of Theorem~\ref{theMainLB} appears in Subsection~\ref{subsec:lowerBound}. Theorem~\ref{theMainLB} demonstrates a lower bound for regret in the regime where $\tau=O(T^{\nicefrac{2}{3}})$, while Theorem~\ref{theMainLB2} demonstrates it in the regime where $\tau=\Omega(T^{\nicefrac{2}{3}})$. We prove both theorems using the following construction.
We show two indistinguishable instances. In both instances $n=2,k=3,\ro=(1-p,p),\delt=(0,p\tau,p\tau)$ for any $0<p <\nicefrac{1}{2}$, and the difference is only in the expected rewards. We assume Bernoulli rewards, where the expected rewards in the first instance are $\Mmu=\left(\begin{smallmatrix}
\nicefrac 1 2 & 0 & 0 \\ 0 &  \nicefrac {1+\epsilon_1 } {2} & \nicefrac 1 2 \end{smallmatrix}\right)$, while in the second instance, $\Mmu'=\left(\begin{smallmatrix} \nicefrac 1 2 & 0 & 0 \\ 0 &  \nicefrac {1} {2} & \nicefrac {1+\epsilon_1 } {2} \end{smallmatrix}\right)$. We set the value of $\epsilon_1$ according to the respective regime:
\begin{itemize}
    \item If $\tau=\Theta(T^x)$ for some $x\geq\nicefrac{2}{3}$, then $\epsilon_1=\frac{1}{2\sqrt{2c}}T^{-\nicefrac{1}{2} x}$ for a constant $c$ such that $c\cdot T^{x}>\tau$.
    \item Otherwise, if $\tau=O(T^{\nicefrac{2}{3}})$, then $\epsilon_1=\frac{1}{2\sqrt{2c}}T^{-\nicefrac{1}{3}}$ for a constant $c$ such that $c\cdot T^{\nicefrac{2}{3}}>\tau$.
\end{itemize}
For each regime, we pick each instance with probability half. 

Let $t^{\star}$ denote the last round of the last phase that all arms are viable (possibly an r.v.). For example, if the first arm departs by the end of phase $i$, then $t^{\star}=(i-1) \cdot \tau$ (the last round in phase $i-1$).  Next, we prove the following lemmas. Their proofs appear after the proof of this theorem.
\begin{lemma}[Mirroring Lemma~\ref{lemmaLBbeginning}]\label{lemmaLo2}
For every algorithm $A$, $\E[R^A]\geq \Omega(p\cdot t^{\star})$ for both instances.
\end{lemma}
\begin{lemma}\label{lemmaLo1}
Let $1\leq x_{u,a} \leq T$ be the number of times that an arm $a\in K$ was pulled for a user type $u$, and $r:\{0,1\}^{x_{u,a}}\rightarrow\{\frac{1}{2},\frac{1+\epsilon_1}{2}\}$ be any rule that determines the expected reward $\mu=\mu_{u,a}$. Therefore, if $x_{u,a} \leq 2\cdot\max (T^{\nicefrac 2 3},\tau)$ then, \[
\forall r: \Pr[r(observations)=\hat{\mu} | \mu=\hat{\mu} ]<0.75 \quad \forall \hat{\mu} \in \left\{\frac{1}{2},\frac{1+\epsilon_1}{2}\right\}.
\]
\end{lemma}
\begin{corollary}[Mirroring Lemma~\ref{lemmaLBtail}]\label{corLo1}
No algorithm can distinguish the two instances with probability of at least $\nicefrac{3}{4}$  after $2\cdot \max( T^{\nicefrac{2}{3}},\tau)$ rounds.
\end{corollary}
We finish the proof by providing a different analysis for each regime.
\begin{enumerate}
    \item Assume that $\tau\leq T^{\nicefrac{2}{3}}$. Lemma~\ref{lemmaLo2} asserts that for any algorithm $A$, if $t^{\star}> T^{\nicefrac 2 3}$, the regret is at least  $\Omega( T^{\nicefrac 2 3})$. Otherwise, $t^{\star}\leq  T^{\nicefrac 2 3}$; thus, an arm departs after less than $T^{\nicefrac 2 3}+\tau \leq 2\cdot T^{\nicefrac 2 3}$ rounds. Therefore,  Corollary~\ref{corLo1} hints that the departing arm is the better one with a probability greater than $\nicefrac 1 4$; hence, with a probability of at least $\nicefrac 1 4$,  $A$ loses at least $\nicefrac{\epsilon_1}{2}=\Theta(T^{-\nicefrac 1 3})$ each time a type $2$-user arrives  in the remaining rounds. Consequently,
    \[
E[R^A]\geq \min\left((T-t^{\star})\cdot p \cdot \Theta(T^{-\nicefrac 1 3}) \cdot 0.25 ,\Omega(p\cdot t^{\star})\right)=\Omega(T^{\nicefrac 2 3}).
\]
\item Alternatively, assume that $\tau=c\cdot T^{x}$ for $c>0,x \geq \frac{2}{3}$. Lemma~\ref{lemmaLo2} claims that for any algorithm $A$, if $t^{\star}\geq \tau$, the regret is at least  $\Omega(\tau)=\Omega(T^{\nicefrac 2 3})$. Otherwise, $t^{\star}< \tau$; hence,  an arm departs at the end of the first phase (after $\tau$ rounds). 
Corollary~\ref{corLo1} asserts that the probability that the departing arm is better than the other is at least $\nicefrac 1 4$; hence, with that probability $A$ loses at least $\nicefrac{\epsilon_1}{2}=\Theta(T^{-\frac {1} {2} x})$ each time a type 2-user arrives in the remaining $T-\tau $ rounds. To conclude,
\[
E[R^A]\geq \min\left((T-t^{\star})\cdot p \cdot \Theta(T^{-\frac {1} {2} x}) \cdot 0.25 ,\Omega(p\cdot t^{\star})\right)= \Omega(T^{\min(x,1-\frac {1} {2} x)})=\Omega(T^{1-\frac {1} {2} x}).
\]
\end{enumerate}
This completes the proof.
\end{proof}

\begin{proofof}{lemmaLo2}
Type 2-users enter uniformly with probability $p$ and $A$ pulls $a_2,a_3$ at least $2\cdot p\cdot t^{\star}$ times. Denote by $X_2$ the number of times that type $2$-users enters the system in the first $t^{\star}$ rounds. Notice that $\Pr[X_2 \leq p\cdot t^{\star}]=0.5$; therefore, $A$ subsidizes $a_2$ or $a_3$ for $\Omega(p\cdot t^{\star})$ times with probability $\geq 0.5$. Notice that the expected cost of subsidizing  $a_2$ or $a_3$ is $\nicefrac 1 2-0$; hence, $\Pr[R^A\geq \Omega(p\cdot t^{\star})]\geq 0.5$, and overall $E[R^A]\geq \Omega(p\cdot t^{\star})$. 
\end{proofof}

\begin{proofof}{lemmaLo1}
This proof follows the lines of the lower bound construction of \citet{slivkins2019introduction}. Let $(\mu_1,\mu_2) =\{\frac{1}{2},\frac{1+\epsilon_1}{2}\}$.
\begin{definition}\label{def:RCeps}
 Let $RC_{\epsilon}$  denote a random coin with a bias of $\nicefrac{\epsilon}{2}$ , i.e., a distribution over $\{0, 1\}$ with expectation $\nicefrac{(1 + \epsilon)}{2}$. Let
 $RC_{\epsilon}^n$ be the distribution over $\{0, 1\}^n$ that each entry is drawn i.i.d. from $RC_{\epsilon}$.
\end{definition}
\begin{lemma}[\citet{slivkins2019introduction}]\label{lemmaLo5}
Consider the sample space $\Omega=\{0,1\}^n$ and two distributions on $\Omega$, $p=RC_{\epsilon}^n,q=RC_0^n$ , for
some $0<\epsilon <\nicefrac 1 2$. Then $|p(A)-q(A)|\leq \epsilon \sqrt{n}$ for any event $A\subset \Omega$.
\end{lemma}
We assume by contradiction that there exists a rule $r$ such that $\Pr[r(observations) = \hat{\mu} | \mu=\hat{\mu} ]\geq 0.75.$ 
Denote $A_0\subset \Omega$, the event that $r$ returns $\mu_1$ (w.l.o.g.); therefore,
\begin{equation}\label{eq:lowerbound to cont}
 \Pr[A_0|\mu=\mu_1]\geq 0.75 , \Pr[A_0|\mu=\mu_2]< 0.25 \Rightarrow \Pr[A_0|\mu=\mu_1]-\Pr[A_0|\mu=\mu_2]> 0.5.   
\end{equation}
Let $P_i(A) = \Pr[A | \mu=\mu_i]$, for each event $A\subset\Omega$ and each $i \in \{1, 2\}$. Then $P_i = P_{i,1} \times \dots \times P_{i,x_{u,a}}$,
where $P_{i,j}$ is the distribution of the $j$'th coin where $\mu=\mu_i$. Therefore, $P_1,P_2$ satisfy the condition of Lemma~\ref{lemmaLo5} ($P_1\sim RC_0^{x_{u,a}},P_2\sim RC_{\epsilon_1}^{x_{u,a}} $). Namely,
\[
\forall A\subset\Omega: |P_1(A)-P_2(A)|\leq \epsilon_1\cdot\sqrt{x_{u,a}}\leq  \frac{1}{2} \textnormal{ (by definition of } \epsilon_1\textnormal{ for each regime}).
\]
The above contradicts Inequality~\eqref{eq:lowerbound to cont}.
\end{proofof}

\subsection{Proof of Proposition~\ref{prop:missingDelta}}\label{Asubsec:missing delta}
\begin{proofof}{prop:missingDelta}
We show two instances in which any algorithm unaware of $\delt$ suffers a linear regret. Let $n=2,k=2,\ro=(0.7,0.3), \Mmu=\left( \begin{smallmatrix}1 & 0\\0 & 1\\ \end{smallmatrix}\right)$. The two instances differ in the exposure threshold: In Instance 1, $\delt^1=(0,0.3\tau)$, while in Instance $2$, $\delt^2=(0,0.4\tau)$. 

$\OPT$ subsidizes $a_2$ in both instances, since the long-term expected reward with a subsidy is $\geq 0.9 T$ and without subsidy is $0.7 T + O(\tau)$. 

An algorithm that does not know where the thresholds are $\delt^1$ or $\delt^2$ cannot distinguish the instances. If the algorithm pulls $a_2$ at least $0.4\tau$ rounds in every phase (the exposure constraint under $\delt^2$), then in Instance $1$ it subsidizes $a_2$ for $0.1T$ times more than $\OPT$ in expectation. Therefore, $\E[\R]\geq 0.1 T$.

Alternatively, the algorithm pulls $a_2$ less than $0.4\tau$ rounds in a phase. In Instance $2$, arm $a_2$ departs and the total reward is $0.7 T+O(\tau)$, which is roughly $0.2T$ less than $\OPT$'s reward of $0.9T$. Overall, we prove that every algorithm that does not know $\delt$ suffers a linear regret.
\end{proofof}

\section{The Planning Task} \label{AppenPlanning}

We define the planning task in Subsection~\ref{subsec:pick z}. This section formally presents it and proves Proposition~\ref{prop:NPcompleteness}.
\subsection{Mathematical formulation}
The planning task is the following optimization problem. 
\begin{equation*}
\argmax_{X,Z} \sum_{t=1}^{\tau}\sum_{j=1}^k \MmuSub_{u_t,j}\cdot X_{t,j}\quad,
\end{equation*}
such that
\begin{enumerate}
    \item $\forall t\in \{1,\dots,\tau\}: \sum_{j=1}^k X_{t,j}=1$ (one arm per round). \label{Pconst1}
    \item $\forall j\in K: \sum_{t=1}^{\tau} X_{t,j} \geq \delta_{j}\cdot Z_j$ (commitment, see Property~\ref{prop_commited}). \label{Pconst2}
    \item $\forall t\in \{1,\dots,\tau\},j\in K: X_{t,j}\leq Z_j$ (commitment, see Property~\ref{prop_commited}) .\label{Pconst3}
    \item $\forall t\in \{1,\dots,\tau\}, \forall j\in K: X_{t,j}\in \{ 0,1\}.  $\label{Pconst4}
    \item $ \forall j\in K: Z_{j}\in \{ 0,1\}.  $\label{Pconst5}
\end{enumerate}

The binary matrix $X$ represents the matching of users to arms, namely $X_{t,j}=1$ if and only if the user $u_t$ is matched to an arm $j$. The binary variables $Z_j$ for $j\in K$ get $Z_j=1$ if and only if arm $j$ is in the committed subset. Notice that Constraints \ref{Pconst4} and \ref{Pconst5} are binary constraints, so this problem is a Mixed Integer Linear Programming that is generally hard to solve. For ease of presentation, we denote the value of this optimization problem by $f=f(X,Z)$. We want to find $\argmax_{X,Z} f(X,Z)$.

In Subsection~\ref{Asubsec:np complete}, we show that finding $X,Z$ that maximize $f$ is an NP-hard problem. 


\subsection{Proof of Proposition~\ref{prop:NPcompleteness}}\label{Asubsec:np complete} \label{subsec:appendixNPComplete}
\begin{proof}
We show a polynomial reduction from Exact-Cover to our problem. I.e., $Exact-Cover \leq_p planning$.  Recall that we represent the planning problem by the tuple $<U,K,\Mmu,\delt,\tau,\q>$, where $\q \in U^\tau$ is known (complete information). Next, we define the Exact-Cover problem.
\begin{definition}
\textnormal{
The input to the Exact-Cover problem is a universe  $\calu$ and a collection of sets of the universe  $S$. We say that a sub-collection $S^\star \subseteq S$ is an \textit{exact cover} of $\calu$ if each element in $\calu$ is contained in exactly one subset in $S^\star$. The decision problem is to determine whether an exact cover exists.
}
\end{definition}

\begin{remark}
Recall that in Subsection~\ref{subsec:parametric} we had two explicit assumptions about the instances we consider:
\begin{enumerate}
    \item Inequality~\eqref{eq:assumption1} suggests that $\sum_{a\in K} \max(\delta_a, \gamma \cdot \tau) \leq \tau$.
    \item The number of contexts $n$ is constant and independent of $\tau$.
\end{enumerate}
We present below two reductions. The first reduction violates these assumptions, but is easily explained. The second reduction, which appears in Subsection~\ref{subsubsec:appendVirtual},  satisfies the above assumptions but is slightly more involved. 
\end{remark}

\subsubsection{First reduction}
Given an exact cover instance $<\calu, S>$, we generate the following planning problem.
\begin{itemize}
    \item $U\leftarrow \calu$, each element in the world $\calu$ is a user type.
    \item $\tau\leftarrow\abs{\calu}$.
    \item $\q\leftarrow permute(\calu)$, each user type arrives once. 
    \item $K\leftarrow S$, each set in $S$ is associated with an arm.
    \item $\delta_j=|S_j|$.
    \item $\MmuSub_{i,j}=1$ if  $i\in S_j$; namely, the utility of type $i$-users for arms $j\in K$ is $1$ if $i \in \calu$ is in the set $S_j$. Otherwise, if $i \notin S_j$,  $\MmuSub_{i,j}=0$.
\end{itemize}
To prove the theorem, it suffices to show the following lemma.
\begin{lemma}\label{lemma:exact cover}
There exists an exact cover $S^{\star} \subseteq S$ if and only if $\max f(X,Z) =\tau.$ 
\end{lemma}
\begin{proof}[\textnormal{\textbf{Proof of Lemma~\ref{lemma:exact cover} for the first reduction}}]
To prove the lemma, we need to show that both directions hold.
\paragraph{Direction $\Rightarrow$:}
Assume that there is an exact cover $S^{\star} \subseteq S$. We shall construct $X,Z$ such that $f(X,Z) =\abs{\calu}$ (recall that $\tau=\abs{\calu}$). We set $X$ such that
\[
X_{i,j}=
\begin{cases}
1 & i\in S_j \text{ and } S_j \in S^{\star} \\
0 & \textnormal{otherwise}
\end{cases}.
\]
Further, let $Z_j=1$ if the set $S_j\in S$ is in the exact cover $S^\star$; otherwise, $Z_j=0$.

Next, we show that $f(X,Z)=\tau$. By definition, for every pair $(i,j)$ such that $X_{i,j}=1$, $\MmuSub_{i,j}=1$ also holds. Moreover, there are exactly $\abs{\calu}=\tau$ entries in $X$ that are equal one, because each element $i\in \calu$ is in exactly one set $S\in S^{\star}$; therefore, it holds that
\[f(X,Z) = \sum_{i=1}^{\tau}\sum_{j=1}^k \MmuSub_{i,j}\cdot X_{i,j} =\sum_{i=1}^{\tau}\sum_{j=1}^k  X_{i,j} =\tau. 
\]
We are left to show that $(X,Z)$ satisfies the constraints.
$X$ satisfies Constraint~\eqref{Pconst1} because each element $i\in \calu$ is in exactly one set in $S^{\star}$. Constraint~\eqref{Pconst2} holds, because for each arm $j\in K$ such that $S_j \in S^\star$ there are exactly $\abs{S_j}=\delta_j$ users $i\in U$ for which $X_{i,j}=1$. Constraint~\eqref{Pconst3} is satisfied since $X_{i,j}>0$ only if $S_j\in S^{\star}$. But in such a case, $Z_j=1$ holds as well. The rest of the constraints hold trivially. 
\paragraph{Direction $\Leftarrow$:} Assume that there exist $(X,Z)$ such that
$f(X,Z) =\tau $. We shall construct an exact cover $S^\star$. Let 
\[
S^{\star}=\left\{ S_j : \exists i \in \calu \text{ s.t. } X_{i,j}=1 \right\}.
\]
Since $f(X,Z)=\tau$, the reward in every round is exactly $1$. I.e., 
\begin{align}\label{eq:i in sj}
    X_{i,j}=1 \Rightarrow \mu_{i,j}=1 \Rightarrow i \in S_j.
\end{align}
Now, we present the following Propositions~\ref{propNP2} and~\ref{propNP011}.
\begin{proposition}\label{propNP2}
If $S_j\in S^{\star}$, then for every $i \in S_j$ it holds that $X_{i,j}=1$
\end{proposition}

\begin{proposition}\label{propNP011} 
For every $S_j, S_k \in S^\star$ such that $S_j \ne S_k$ $S_j\cap S_k = \emptyset.$
\end{proposition}
The proofs of Propositions~\ref{propNP2} and~\ref{propNP011} appear after the proof of this theorem. 
From Proposition~\ref{propNP2}, Equation~\eqref{eq:i in sj} and the construction of $S^{\star}$, we conclude that $X_{i,j}=1 \iff i\in S_j$  and $S_j\in S^{\star}$. Therefore, 
\begin{equation*}\begin{split}
|\cup_{S_j \in S^{\star}}S_j|= \sum_{S_j\in S^{\star}}\sum_{i\in S_j} X_{i,j} = \sum_{i=1}^{\tau} \sum_{j=1}^k  X_{i,j} =
\sum_{i=1}^{\tau} \sum_{j=1}^k \MmuSub_{i,j}\cdot X_{i,j} =f(X,Z)=\tau=\abs{\calu}.    
\end{split}\end{equation*}
The first equality is due to Proposition~\ref{propNP011}, the second equality stems from the fact that the other summands are $0$, and the third equality is due to Equation~\eqref{eq:i in sj}. Overall, we get $\cup_{S_j \ne S_k\in S^{\star}} S_j\cap S_k = \emptyset $ and $\cup_{S_j \in S^{\star}}S_j=\calu$. This suggests that $S^{\star}\subseteq S$ is an exact cover of $\calu$. This concludes the proof of Lemma~\ref{lemma:exact cover}.
\end{proof}
\subsubsection{Second reduction: Incorporating virtual users and arms}\label{subsubsec:appendVirtual}
As we mentioned above, the previous reduction generates instances that do not satisfy our assumptions from Subsection~\ref{subsec:parametric}. We now introduce a new reduction that does satisfy these assumptions and explains the minimal modifications required to the proof of Lemma~\ref{lemma:exact cover}. 

For the new reduction, we add virtual users and a virtual arm. The reduction is given by
\begin{itemize}
    \item $U\gets \calu \cup \{\tilde u \}$, where $\tilde u$ is a virtual user type.
    \item $\tau\gets\abs{\calu}+W$, for  $W \in \mathbb N$ (we elaborate on $W$ shortly).
    \item $\q$ is a permutation of $\calu$ and $W$ times $\tilde u$.
    \item $K\gets S \cup \{ {\tilde S} \}$, where $\tilde S$ is a virtual arm.
    \item For all $j, 1\leq j \leq k$, $\delta_j=|S_j|$. For $j=k+1$ (the virtual arm),  $\delta_j=0$.
    \item $\MmuSub_{i,j}=1$ if  $i\in S_j \in S$ and $S_j \neq \tilde S$. Otherwise, $\MmuSub_{i,j}=0$. For the virtual user type $\tilde u$ (the $n+1$'s user type), $\MmuSub_{n+1,j}=0$ for all arms $j, 1\leq j \leq k$. For the virtual user and the virtual arm (the $k+1$'s arm), $\MmuSub_{n+1,k+1}=1$.
\end{itemize}
Notice that the number of user types $n$ is $\abs{\calu}+1$, whereas the phase length $\tau=\abs{\calu}+W$, where $W$ could be any polynomial function of the input to Exact-Cover. This  also ensures that Inequality~\eqref{eq:assumption1} can be satisfied.

Next, we modify the proof of the lemma. In the $\Rightarrow$ direction, we define the same $X$ as before for the real users and  $X_{i,j}=1 \iff j=k+1$ for the virtual users. The reward from the virtual users is $W$, and the reward from the real users is $\abs{\calu}$. Therefore, the total reward is $\tau$, and $X$ still satisfies the constraints with $Z\cup\{a_{k+1}\}$ (the virtual arm).

In the $\Leftarrow$ direction, $\max f(X,Z)=\tau$; therefore,  the utility of every match is $1$. Particularly, each virtual user is matched to the virtual arm, and each real user  is matched to a real arm. We can therefore construct $S^{\star}$ as before using real users only and real arms only. Consequently, there is an exact cover in $S$.

We have thus proved that $Exact-Cover \leq_p planning$. This concludes the proof of Proposition~\ref{prop:NPcompleteness}.
\end{proof}

\begin{proofof}{propNP2}
If $X$ pulls an arm $j$, it must commit to $j$. Namely, it pulls $j$ for at least $\delta_j=|S_j|$ times (the second constraint). But, we know that $X$ can pull $j$ only if $i\in S_j$ (see Inequality~\eqref{eq:i in sj}). Therefore, it pulls $j$ for every $i\in S_j$ since each user type arrives once.
\end{proofof}
\begin{proofof}{propNP011}
Assume by contradiction that there exists $S_j \ne S_k\in S^{\star}$ such that $S_j\cap S_k \ne \emptyset$. Therefore, there exists $i\in \calu$ such that $i\in S_j$ and $i\in S_k$. Proposition~\ref{propNP2} suggests that $X_{i,j}=X_{i,k}=1$; this is a contradiction to the fact that $(X,Z)$ is a valid solution so the Constraint~\ref{Pconst1} must hold.
\end{proofof}

\section{Relaxation of Assumption~\ref{assumption:gamma}}\label{sec: assumption relaxation}
In this section, we relax Assumption~\ref{assumption:gamma} and reconsider our previously suggested learning algorithm. As demonstrated before, this assumption guarantees fast exploration and allows us to neglect extreme cases, in which high regret is inevitable. Subsection~\ref{subsec: del geq tau} shows that if $\sum_{a\in K} \delta_a >\tau$, then the regret of any algorithm is high. Subsection~\ref{subsec: f(tau)} analyzes $\mainAlg$'s regret when $\sum_{a\in K} \delta_a \leq \tau$, but fast exploration is not guaranteed. 

\subsection{Unsatisfiable exposure constraints}\label{subsec: del geq tau}
In this subsection, we prove that some instances with $\sum_{a\in K} \delta_a >\tau$ yield high regret for any algorithm. For convenience, we assume that $\delta_1\geq\delta_2\geq\dots\geq \delta_k$. 


\begin{proposition}\label{prop: linear regret ass1}
Fix any learning algorithm $A$. There exists an instance with $\sum_{a=1}^{k-1} \delta_a \geq\tau$ for which $\E[\R^A]=\Theta(T)$.
\end{proposition}
\begin{proposition}\label{prop: almost linear regret ass1}
Fix any learning algorithm $A$. There exists an instance with $\sum_{a\in K} \delta_a > \tau$ and $\sum_{a=1}^{k-1} \delta_a <\tau$  for which $\E[\R^A] = \Omega(\nicefrac{T}{\sqrt{\tau-\sum_{a=1}^{k-1}\delta_a}})\subseteq \Omega(\nicefrac{T}{\sqrt{\delta_k}}) \subseteq \Omega(\nicefrac{T}{\sqrt{\nicefrac{\tau}{k}}})$.
\end{proposition}
Next, we prove Propositions~\ref{prop: linear regret ass1} and~\ref{prop: almost linear regret ass1}.

\begin{proofof}{prop: linear regret ass1}
We prove this proposition with a similar construction to the one in the proof of Theorem~\ref{theMainLB}. Namely,  we show two instances that are indistinguishable w.h.p. if $A$ is unaware of $\Mmu$. In both instances, $k\geq 2$ and $\delta$ are given, $n=k-1$ and $\ro=(\nicefrac{1}{k-1},\nicefrac{1}{k-1},\dots)$.  We assume Bernoulli rewards.
 Instance $1$ has $\Mmu^1$ and Instance $2$ has $\Mmu^2$, where for every user $u, 1\leq u \leq n:$ \[
\mu_{u,a}^1= \begin{cases} 1 & \text{ if } a=a_k \\  0.5 & \text{ if } a=u \text{ and } 1\leq a \leq k-1 \\ 0 & \text{otherwise} \end{cases},
\mu_{u,a}^2= \begin{cases} 0 & \text{ if } a=a_k \\  0.5 & \text{ if } a=u \text{ and } 1\leq a \leq k-1 \\ 0 & \text{otherwise} \end{cases}.
\]

Notice that the instances differ only in the expected rewards of arm $a_k$. So, the only way to distinguish the instances is by  pulling arm $a_k$ and gathering information about its rewards. Since $\tau-\sum_{a=1}^{k-1}\delta_a\leq 0$, then pulling arm $a_k$ even once leads to the departure of another arm $a\in K\setminus{a_k}$ by the end of the phase, since we fail to satisfy its exposure constraint. Therefore, if $A$ pulls arm $a_k$ even once, it suffers a regret of at least $\approx 0.5\cdot\frac{1}{k-1}\cdot T$ in Instance $2$. On the other hand, if $A$ does not pull arm $a_k$ at all, then $A$ suffers a regret of approximately $0.5T$ in Instance $1$. 

We further denote by $C^A$ the event that $A$ pulls arm $a_k$.
Ultimately, assume we pick Instance $i$ with probability $\nicefrac{1}{2}$. Then, $A$ has a regret of at least
\begin{align*}
    \E[R^A]\geq \min(\E[R^A|C^A], \E[R^A|\bar{C^A}]) &\geq \min(\Pr[\Mmu^2]\cdot \E[R^A|C^A,\Mmu^2], \Pr[\Mmu^1]\cdot \E[R^A|\bar{C^A},\Mmu^1])  \\ &  \geq 0.5(\min(0.5\cdot\nicefrac{1}{k-1}\cdot T,0.5T))=\Omega(T).
\end{align*}
The last inequality is due to the fact that $\Pr[\Mmu^1]=\Pr[\Mmu^2]=\nicefrac{1}{2}$. 
\end{proofof}

\begin{proofof}{prop: almost linear regret ass1}
We use similar construction to the one in the proof of Proposition~\ref{prop: linear regret ass1}, in which the only difference is in the expected rewards of $a_k$.  Namely, for every user $u, 1\leq u \leq n$
\begin{align*}
\mu_{u,a}^1= \begin{cases} 0.5+\epsilon & \text{ if } a=a_k \\  0.5 & \text{ if } a=u \text{ and } 1\leq a \leq k-1 \\ 0 & \text{otherwise} \end{cases},
\mu_{u,a}^2= \begin{cases} 0.5-\epsilon & \text{ if } a=a_k \\  0.5 & \text{ if } a=u \text{ and } 1\leq a \leq k-1 \\ 0 & \text{otherwise} \end{cases},    
\end{align*}
where $\epsilon=\frac{1}{4\sqrt{\tau-\sum_{a=1}^{k-1}\delta_a}}$. Now, we rely on the following Lemma~\ref{lemma: relax bad regret}.
\begin{lemma}\label{lemma: relax bad regret}
No algorithm can distinguish the two instances with a probability of at least $\nicefrac{3}{4}$  after pulling arm $a_k$ for $\tau-\sum_{a=1}^{k-1}\delta_a$ rounds or fewer.
\end{lemma}
The proof of Lemma~\ref{lemma: relax bad regret} is identical to the proof of Lemma~\ref{lemmaLBtail} and hence omitted. 

Note that $\delta_k>\tau-\sum_{a=1}^{k-1}\delta_a>0$, because $\sum_{a\in K} \delta_a > \tau$ and $\sum_{a=1}^{k-1} \delta_a <\tau$. Therefore, if $A$ pulls arm $a_k$ fewer than or exactly $\tau-\sum_{a=1}^{k-1}\delta_a$ rounds, then arm $a_k$ departs and in Instance $1$, the regret is approximately $T \cdot \epsilon$. On the other hand, if $A$ pulls arm $a_k$ for $\delta_k>\tau-\sum_{a=1}^{k-1}\delta_a$ rounds, another arm $a\in K\setminus {a_k}$ will depart by the end of the phase, since we fail to satisfy its exposure constraint. Therefore, $A$'s regret in Instance $2$ is approximately $\frac{1}{k-1}\cdot T \cdot \epsilon$. Furthermore, we denote by $C^A$ the event that $A$ pulled arm $a_k$ more than $\tau-\sum_{a=1}^{k-1}\delta_a$ rounds. Ultimately, the regret of $A$ is at least
\begin{align*}
&\E[\R^A]\geq \min(\E[\R^A|C^A],\E[\R^A|\bar{C^A}])\geq \min(\Pr[\Mmu^2]\cdot\E[\R^A|C^A,\Mmu^2],\Pr[\Mmu^1]\cdot\E[\R^A|\bar{C^A},\Mmu^1]) \geq \\ & 0.25\cdot \min(\nicefrac{1}{k-1}\cdot T \cdot \epsilon,T\cdot\epsilon) = \Omega(\nicefrac{T}{\sqrt{\tau-\sum_{a=1}^{k-1}\delta_a}})\subseteq \Omega(\nicefrac{T}{\sqrt{\nicefrac{\tau}{k}}}) .
\end{align*}
The last inequality stems from Lemma~\ref{lemma: relax bad regret}.
\end{proofof}

\subsection{Slow exploration}\label{subsec: f(tau)}
In this subsection, we weaken Assumption~\ref{assumption:gamma} so that fast exploration is no longer guaranteed. We do this by modifying Equation~\eqref{eq:assumption1} to 
\begin{equation}\label{eq:assumption1weaker}
    \sum_{a\in K} \max(\delta_a, f(\tau)) \leq \tau,
\end{equation}
where unlike Equation~\eqref{eq:assumption1}, $f(\tau)$ does not have to be linear in $\tau$. For instance, $f(\tau)$ could be $\nicefrac{\tau}{2k} , \sqrt{\tau}$ or $O(1)$. It is worth noting that if Assumption~\ref{assumption:gamma} holds, then $f(\tau)=\gamma\tau = \Theta(\tau)$. Importantly,  $\sum_{a\in K} \delta_a \leq \tau$ still holds.

Since we do not use Assumption~\ref{assumption:gamma} in the proof of Theorem~\ref{theoremMainRegret1}, namely, we do not  assume that $\gamma$ is a constant, we conclude that
\begin{align*}
\E[r^{\mainAlg(\SSO)}] &\geq \E[\tilde{r}^{\SSO}]-  \tilde{O}((\nicefrac{1}{\gamma}+\sqrt{\gamma}n^2)\cdot T^{\nicefrac{2}{3}})-\err(\SSO) = \\ &\E[\tilde{r}^{\SSO}]-  \tilde{O}\left(\left(\nicefrac{\tau}{f(\tau)}+\sqrt{\nicefrac{f(\tau)}{\tau}}\cdot n^2\right)\cdot T^{\nicefrac{2}{3}}\right)-\err(\SSO) .
\end{align*}
Notice that if $f(\tau) \ne \Theta(\tau)$ (i.e., Assumption~\ref{assumption:gamma} does not hold), the regret of $\mainAlg$ increases by a factor of $\nicefrac{\tau}{f(\tau)}$. For example, if $f(\tau)=O(1)$, then $\E[\R^{\DPS}]$ increases from $O(T^{\nicefrac{2}{3}})$ to $O(\tau T^{\nicefrac{2}{3}})$, which could be significantly greater. The following Theorem~\ref{theoremMainRegret1 relaxation} proves that if we slightly modify the length of  $\mainAlg$'s exploration stage, then its regret could increase by a factor of $(\nicefrac{\tau}{f(\tau)})^{\nicefrac{1}{3}}$ instead of $\nicefrac{\tau}{f(\tau)}$.
\begin{theorem}\label{theoremMainRegret1 relaxation}
    If Equation~\eqref{eq:assumption1weaker} holds and $\mainAlg$ explores for $\left(\frac{\tau}{f(\tau)}\right)^{\nicefrac{1}{3}} T^{\nicefrac{2}{3}}$ rounds, then it holds that
    \[ \E[r^{\mainAlg(\SSO)}]\geq \E[\tilde{r}^{\SSO}]-  \tilde{O}\left(\left((\nicefrac{\tau}{f(\tau)})^{\nicefrac{1}{3}}+\sqrt{(\nicefrac{f(\tau)}{\tau})^{\nicefrac{1}{3}}}\cdot n^2\right)\cdot T^{\nicefrac{2}{3}}\right)-\err(\SSO).\]
\end{theorem}

The proof of this mirrors the proof of Theorem~\ref{theoremMainRegret1} with only replacing the length of the exploration stage ($(\nicefrac{\tau}{f(\tau)})^{\nicefrac{1}{3}}\cdot T^{\nicefrac{2}{3}}$ instead of $\frac{1}{\gamma} T^{\nicefrac{2}{3}}=\nicefrac{\tau}{f(\tau)}\cdot T^{\nicefrac{2}{3}}$); hence, we omit it. 
Ultimately, we bound the regret of $\mainAlg(\DPS)$ and $\mainAlg(\ORS)$ using Theorem~\ref{theoremMainRegret1 relaxation};
\begin{align*}
&\E[\R^{\mainAlg(\DP^\star)}] \leq   \tilde{O}\left(\left((\nicefrac{\tau}{f(\tau)})^{\nicefrac{1}{3}}+\sqrt{(\nicefrac{f(\tau)}{\tau})^{\nicefrac{1}{3}}}\cdot n^2\right)\cdot T^{\nicefrac{2}{3}}+k\cdot \tau\right), \\& \E[\R^{\mainAlg(\ORS)}] \leq   \tilde{O}\left(\left((\nicefrac{\tau}{f(\tau)})^{\nicefrac{1}{3}}+\sqrt{(\nicefrac{f(\tau)}{\tau})^{\nicefrac{1}{3}}}\cdot n^2\right)\cdot T^{\nicefrac{2}{3}}+k\cdot \tau + \nicefrac{nT}{\sqrt{\tau}} \right). 
\end{align*}

\section{Longer Phase Length}\label{sec: long phase reg}
The body of the paper develops and analyzes optimal regret algorithms for the regime $\tau=O(T^{\nicefrac{2}{3}})$. However, as Proposition~\ref{prop: tight LB pico opt} demonstrates, our algorithms have sub-optimal regret in the regime where $\tau=\Omega(T^{\nicefrac{2}{3}})$. In this section, we address the case where $\tau=\Omega(T^{\nicefrac{2}{3}})$. This section is structured as follows. First, Subsection~\ref{subsec:proof of prop tightlb} includes the proof of Proposition~\ref{prop: tight LB pico opt}. Then, in Subsection~\ref{subsec:longer Phase Complete Info} we develop a deterministic and complete information optimal algorithm (thereby extending $\DOAlg$). Third, in Subsection~\ref{subsec:longer Phase Incomplete Info}, we leverage this deterministic algorithm to develop an approximately optimal algorithm for the stochastic optimization task, which has better performance than $\DPS$ in this regime. Forth, in Subsection~\ref{subsec:appen long phase learning}, we analyze the performance of $\mainAlg$ using this stochastic optimization algorithm as $\SSO$. Lastly, in Subsection~\ref{subsec:appen long phase lower bound}, we show a lower bound for the regret in this regime.


\subsection{Proof of Proposition~\ref{prop: tight LB pico opt}}\label{subsec:proof of prop tightlb}
\begin{proof}
    We consider the following instance: $k=2$ arms, $n=2$ user types, arrival distribution $\ro=(0.8,0.2)$, exposure thresholds $\delta=(0,0.5\cdot\tau)$, and expected utility matrix $\Mmu = \bigl( \begin{smallmatrix}1 & 0\\ 0 & 1\end{smallmatrix}\bigr).$
Both $\PIOPT$ and $\OPT$ do not subsidize $a_2$, but $\OPT$ uses $a_2$ till it departs and in particular in the first phase; hence, we get $\E[\tilde{r}^{\OPT}]\geq \tau+0.8\cdot (T-\tau)$ and $\E[\tilde{r}^{\PIOPT}]=0.8T$. To conclude, $\E[\tilde{r}^{\OPT}]-\E[\tilde{r}^{\PIOPT}] \geq 0.2\tau=\Omega(\tau)$.
\end{proof}

\subsection{Planning task for multiple phases}\label{subsec:longer Phase Complete Info}
Recall the planning task from Subsection~\ref{subsec:pick z}, where the goal is to maximize the reward in one phase while 
assuming complete information about the user arrival vector. The optimal solution Subsection~\ref{subsec:lcb_approx} proposes is executing $\DOAlg(c(\q),Z)$ for every $Z\subseteq K$ and  choosing the best match. In this subsection, we present the $\LongDOAlg$ algorithm (implemented in Algorithm~\ref{alg: long doalg}), which extends the planning task to maximize the reward for more than one phase.

We overload the notation and denote by $\DOAlg(c(\q),Z_1,Z_2)$ the optimal algorithm in one phase that pulls only arms from $Z_1\subseteq K$ and satisfies the constraints of every arm in $Z_2\subseteq Z_1$. Note that if $Z_1=Z_2$, then $\DOAlg(c(\q),Z_1,Z_2)$ is equivalent to $\DOAlg(c(\q),Z_1)$. We present the implementation details of $\DOAlg(c(\q),Z_1,Z_2)$ in Subsection~\ref{subsubsec: DOAlg(c(q),Z1,Z2) algorithm}. In particular, we show that its run time is $O(\tau^3)$. 

$\LongDOAlg$ takes as input all the instance parameters. In addition, it also gets the user arrival sequence for all $T$ rounds. We further denote this sequence by $\q_1\oplus\dots\oplus\q_{\nicefrac{T}{\tau}}$, where $\forall i: \q_i\in U^{\tau}$.
It computes the maximum reward achievable in the first $i$ phases while maintaining the constraints of the arms $Z$ 
throughout these phases for every $Z\subseteq K$ ($1\leq i\leq \nicefrac{T}{\tau}$). We further denote this reward by $r^i(Z)$.
In Line~\ref{alglineL:init} $\LongDOAlg$ initializes the $r^0(Z)$ values for every $Z\subseteq K$. It then calculates $r^1(Z)$ for every $Z\subseteq K$, then $r^2(Z)$ for every $Z\subseteq K$ and so on (Lines~\ref{alglineL:forloop}-\ref{alglineL: heavy line}).
Finally, the algorithm returns the maximum  reward that can be achieved after $\nicefrac{T}{\tau}$ phases (Line~\ref{alglineL: return}). It is worth noting that the reward-maximizing matchings $M_1,M_2,\dots,M_{\nicefrac{T}{\tau}}$ and the reward-maximizing subsets, denoted by $Z^{\nicefrac{T}{\tau}}\subseteq\dots\subseteq Z^1\subseteq Z^0= K$, can be easily reconstructed. Now, we prove Proposition~\ref{proposition long DO runtime} and Theorem~\ref{thm: long DOAlg is optimal}.
\begin{proposition}\label{proposition long DO runtime}
The time complexity of $\LongDOAlg$ is $O(\frac{T}{\tau}(3^k\cdot \tau^3))$.
\end{proposition}
\begin{proofof}{proposition long DO runtime}
The number of options for $Z_1,Z_2$ such that $Z_2\subseteq Z_1\subseteq K$ is
\[
\abs{\{(Z_1,Z_2)|Z_2\subseteq Z_1\subseteq K\}}=\sum_{Z_1\subseteq K} 2^{\abs{Z_1}} = \sum_{i=1}^k {k \choose i} 2^i = (1+2)^k=3^k. 
\]
Therefore, the time complexity of Line~\ref{alglineL: heavy line} is $3^k\cdot\tau^3$. The run time of Line~\ref{alglineL: return} is $2^k$ so, the total run time of $\LongDOAlg$ is $O(\frac{T}{\tau}(3^k\cdot \tau^3))$.
\end{proofof}
Next, we prove the optimality of $\LongDOAlg$.
\begin{theorem}\label{thm: long DOAlg is optimal}
$\LongDOAlg$ is optimal.
\end{theorem}


\begin{algorithm}[t]
\caption{$\LongDOAlg$ algorithm}\label{alg: long doalg}
\begin{algorithmic}[1]
\Statex \textbf{Input}: $U,K, \delt, \tau, T, \mu, \q_1\oplus\dots\oplus\q_{\nicefrac{T}{\tau}}$
\State for every $Z\subseteq K: \quad r^0(Z)\gets  0$ \label{alglineL:init}
\For {$i=1,2,\dots,\nicefrac{T}{\tau}$}\label{alglineL:forloop}
\State for every $Z_2\subseteq K:$ \[r^i(Z_2)\gets \max_{Z_1: Z_2\subseteq Z_1\subseteq K} \DOAlg(c(\q_i),Z_1,Z_2)+r^{i-1}(Z_1)\] \label{alglineL: heavy line}
\EndFor
\State return $\max_{Z\subseteq K} r^{\nicefrac{T}{\tau}}(Z)$ \quad  (and $Z^{\nicefrac{T}{\tau}},\dots, Z^1$ or $M_1,M_2,\dots,M_{\nicefrac{T}{\tau}}$ if needed)
\label{alglineL: return}
\end{algorithmic}
\end{algorithm}

\begin{proofof}{thm: long DOAlg is optimal}
First, we show Proposition~\ref{prop: to small is better}.
\begin{proposition}\label{prop: to small is better}
For every user arrival sequence $\q \in U^{\tau}$ and arm subsets $Z_1,Z_2,Z_3$ such that $Z_3\subseteq Z_2\subseteq Z_1 \subseteq K$, it holds that \[\DOAlg(c(q),Z_1,Z_2)\leq \DOAlg(c(q),Z_1,Z_3).\]
\end{proposition}
Now, we leverage Proposition~\ref{prop: to small is better} to Lemma~\ref{lem: rz is optimal}.
\begin{lemma}\label{lem: rz is optimal}
For every $i\geq 1 , Z\subseteq K$, it holds that $r^i(Z)$ is the maximal reward that an algorithm can achieve in the first $i$ phases while maintaining the exposure constraints of arms $Z$.
\end{lemma}
It stems directly from Lemma~\ref{lem: rz is optimal} that $\max_{Z\subseteq K} r^{\nicefrac{T}{\tau}}(Z)$ is the maximal reward that any algorithm can achieve. Therefore, $\LongDOAlg$ is optimal.
\end{proofof}

\begin{proofof}{prop: to small is better}
$\DOAlg(c(\q),Z_1,Z_3)$ can satisfy the constraints of $Z_2$ since $Z_3\subseteq Z_2$ and $Z_2\subseteq Z_1$.
Namely, it can do the same as $\DOAlg(c(\q),Z_1,Z_2)$.
From the optimally of $\DOAlg$ it holds that the reward of $\DOAlg(c(\q),Z_1,Z_3)$ is at least as the one of $\DOAlg(c(\q),Z_1,Z_2)$. 
\end{proofof}


\begin{proofof}{lem: rz is optimal}
We prove this lemma by induction on the phase number $i$.

\textbf{Base case $i=1$:} Let $A$ be any algorithm that satisfies the exposure constraints of arms $Z\subseteq K$ after one phase. We further denote by $Z_A$ the arms that are still available after one phase. It holds that $Z\subseteq Z_A \subseteq K$. 

We denote by $A_{i:j}(\q)$ the reward of $A$ in rounds $i,\dots,j$. Notice that
\begin{align*}
  A_{1:\tau}(\q)\leq \DOAlg(c(\q),K,Z_A)\leq \DOAlg(c(\q),K,Z) \leq r^1(Z).
\end{align*}
The first inequality is due to $\DOAlg$'s optimally, the second stems from Proposition~\ref{prop: to small is better}, and the third stems directly from $r^1(Z)$ definition (see Algorithm~\ref{alg: long doalg}).

\textbf{Inductive step:} Let $A$ be any algorithm that satisfies the exposure constraints of arms $Z\subseteq K$ throughout the first $i$ phases. We assume that the lemma holds for every $j < i$ and prove that it holds for $i$ too.

We further denote by $Z_A'$ and $Z_A$ the subsets of arms that are still available after $i-1$ phases and $i$ phases respectively. Notice that $Z\subseteq Z_A \subseteq Z_A'\subseteq K$. It holds that
\begin{align*}
A_{1:i\tau}(\q)&=A_{1:(i-1)\tau}(\q)+A_{(i-1)\tau+1:i\tau}(\q)\leq r^{i-1}(Z_A')+\DOAlg(c(\q),Z_A',Z_A)\\
&\leq r^{i-1}(Z_A')+\DOAlg(c(\q),Z_A',Z)\leq r^i(Z).
\end{align*}
The first inequality is due to $\DOAlg$'s optimality and the inductive assumption. The second stems from Proposition~\ref{prop: to small is better} and the third from the definition of $r^i(Z)$.
\end{proofof}

\subsubsection{The $\DOAlg(c(\q),Z_1,Z_2)$ algorithm}\label{subsubsec: DOAlg(c(q),Z1,Z2) algorithm}
$\DOAlg(c(\q),Z_1,Z_2)$ is similar to $\DOAlg(c(\q),Z_1)$. The only difference is that it treats $\delta_a$ as $0$ for every $a\in Z_1\setminus Z_2$. It is straightforward that this algorithm: (1) has a runtime of $O(\tau^3)$, (2) satisfies the constraints for $Z_2$, and (3) only pulls arms from $Z_1$. $\DOAlg(c(\q),Z_1,Z_2)$'s optimality stems directly from $\DOAlg(c(\q),Z)$'s optimality.

\subsection{Adapting $\ORS$ to the long phase regime}\label{subsec:longer Phase Incomplete Info}

In Subsection~\ref{subsec:lcb_approx}, we presented stochastic optimization algorithms with an $\Omega(\tau)$ factor in their approximation error. This error is unacceptable in the long phase regime where $\tau=\Omega(T^{\nicefrac{2}{3}})$. In this subsection, we propose the $\LPOR$ algorithm (implemented in Algorithm~\ref{alg:longlcbapprox}) by combining $\LongDOAlg$ (Subsection~\ref{subsec:longer Phase Complete Info}) with $\ORS$ (Subsection~\ref{subsec:lcb_approx}).
There is one key difference between $\LPOR$ and $\ORS$. $\LPOR$ uses a phase-dependent matching matrix $M_i$ in phase $i$ for every $1\leq i\leq \nicefrac{T}{\tau}$, which is a part of the output of $\LongDOAlg({\cv^+\oplus\cv^+\dots\oplus\cv^+})$. In contrast, $\ORS$ uses $\DOAlg({\cv^+},Z^\star)$ once to output a phase-independent matching matrix $M$. This change allows $\LPOR$ to skip the $\Omega(\tau)$ factor in its approximation error ($\LPOR$ does not satisfy Properties~\ref{prop_PI12343},\ref{prop_commited}). Now, we show Lemma~\ref{lemma:lplcb is optimal}, which bounds the approximation error of $\LPOR$.
\begin{lemma}\label{lemma:lplcb is optimal}
$\E_{\q\sim \ro^T}[\tilde{r}^{\LPOR}(\q)]\geq \E_{\q\sim \ro^T}[\LongDOAlg(\q)]-O(\nicefrac{nT}{\sqrt{\tau}})$
\end{lemma}
The proof of the lemma appears below. The optimality of $\LongDOAlg$ in Theorem~\ref{thm: long DOAlg is optimal} suggests that 
\begin{align}\label{eq:LDO is better than OPT}
    \forall \q\in U^\tau: \LongDOAlg(\q)\geq \tilde{r}^{\OPT}(\q). 
\end{align}
Combining Lemma~\ref{lemma:lplcb is optimal} and Equation~\eqref{eq:LDO is better than OPT}, we conclude that
\begin{align}\label{eq:LPOR related to OPT}
\E_{\q\sim \ro^T}[\tilde{r}^{\LPOR}(\q)]\geq \E_{\q\sim \ro^T}[\tilde{r}^{\OPT}(\q)]-O(\nicefrac{nT}{\sqrt{\tau}}).    
\end{align}
Finally, observe that if $\tau=\Omega(T^{\nicefrac{2}{3}})$, the approximation error of $\LPOR$ ($O(\nicefrac{T}{\sqrt{\tau}})$) is smaller than the approximation errors of $\DPS$ and $\ORS$ ($O(\tau), O(\nicefrac{T}{\sqrt{\tau}}+\tau)$ respectively). 
We leave the analysis of $\mainAlg$'s regret to the next subsection, and now prove Lemma~\ref{lemma:lplcb is optimal}.

\begin{algorithm}[t]
\caption{$\LPOR$ algorithm}\label{alg:longlcbapprox}
\begin{algorithmic}[1]
 \Statex\textbf{Input}: $Z,U,\delt,\tau,T,\Mmu,\ro$
\State execute $\LongDOAlg({\cv^+\oplus\cv^+\dots\oplus\cv^+})$ and let $M_1,\dots,M_{\nicefrac{T}{\tau}}$ be the reward-maximizing matchings and   $Z^{\nicefrac{T}{\tau}}\subseteq\dots\subseteq Z^1\subseteq Z^0= K$ the reward-maximizing subsets \label{lpalglin:matchinglcb}
\For{$i=1,2,\dots, \nicefrac{T}{\tau}$}
\For {$t=1,\dots,\tau$}\label{lpalglin:for}
    \State observe $u_t$ \label{lpalglin:observeu}
    \If {$M_i(u^\star,\cdot) = 0$ and $M_i(u_t,\cdot)=0$ } \label{lpalglin:badevent} 
    \State select $u_t \in \{u\in U: M_i(u,\cdot) >0  \}$\label{lpalglin:changeubad} \Comment{bad event} 
    \ElsIf {$M_i(u_t,\cdot)=0$}  $u_t \gets u^\star$\label{lpalglin:setstar}
    \EndIf
    \State select $a_t \in \left\{a\in Z^{i-1} :M_i(u_t,a) >0   \right\}$, pull $a_t$ \label{lpalglin:at}
    \State $M_i(u_t,a_t) \gets M_i(u_t,a_t)-1$ \label{lpalglin:updateM}
\EndFor
\EndFor
\end{algorithmic}
\end{algorithm}

\begin{proofof}{lemma:lplcb is optimal}
The proof of this lemma goes along the lines of the proof of Lemma~\ref{lemma:lcbz and pico z}. We prove the following propositions.
\begin{proposition}\label{prop: lplcb is near to doalg}
For every $i, 1\leq i \leq \nicefrac{T}{\tau}$, it holds that \[\E_{\q\sim\ro^{\tau}}[\tilde{r}_{(i-1)\tau+1:i\tau}^{\LPOR}(\q)]\geq \DOAlg({\cv^+},Z^{i-1},Z^i)-\frac{2n}{\tau}.\]
\end{proposition}
\begin{proposition}\label{prop: cplus is approx}
It holds that $\LongDOAlg({\cv^+\oplus\dots\oplus\cv^+})\geq \E_{\q\sim\ro^T}[\LongDOAlg(\q)]-O(\frac{nT}{\sqrt{\tau}})$.
\end{proposition}

From Proposition~\ref{prop: lplcb is near to doalg} we conclude that 
\begin{align}\label{eq: leverage proposition 13}
    \E_{\q\sim\ro^{T}}[\tilde{r}^{\LPOR}(\q)]&=\sum_{i=1}^{\nicefrac{T}{\tau}}\E_{\q\sim\ro^{\tau}}[\tilde{r}_{(i-1)\tau+1:i\tau}^{\LPOR}(\q)] \geq \sum_{i=1}^{\nicefrac{T}{\tau}} \DOAlg({\cv^+},Z^{i-1},Z^i)-\frac{2n}{\tau} \\
    &= \LongDOAlg({\cv^+\oplus\dots\oplus\cv^+})-O(\frac{nT}{\tau^2}).\nonumber
\end{align}
Combining Proposition~\ref{prop: cplus is approx} and Inequality~\eqref{eq: leverage proposition 13}, we get \[\E_{\q\sim\ro^{T}}[\tilde{r}^{\LPOR}(\q)]\geq \E_{\q\sim\ro^T}[\LongDOAlg(\q)]-O(\frac{nT}{\sqrt{\tau}}).\]
This ends the proof of Lemma~\ref{lemma:lplcb is optimal}.
\end{proofof}

\begin{proofof}{prop: lplcb is near to doalg}
Fix a phase $i$, $1\leq i \leq \nicefrac{T}{\tau}$. $\LPOR$ uses $M_i$ in phase $i$, which is the output of $\DOAlg({\cv^+},Z^{i-1},Z^i)$. In the clean event, its reward in phase $i$ is at least as $\DOAlg({\cv^+},Z^{i-1},Z^i)$'s reward. Because in the clean event, $\LPOR$ pulls the same arms as $\DOAlg$ for its real users and the reward of $\DOAlg$ from its slack users is $0$ (see Algorithm~\ref{alg:longlcbapprox} and Subsection~\ref{subsec:lcb_approx}). Due to Inequality~\eqref{eq:concentration}, the bad event's probability is at most $\nicefrac{2n}{\tau^2}$. Therefore,
\[
\E_{\q\sim\ro^{\tau}}[\tilde{r}_{(i-1)\tau+1:i\tau}^{\LPOR}(\q)]\geq (1-\nicefrac{2n}{\tau^2})\DOAlg({\cv^+},Z^{i-1},Z^i) +\nicefrac{2n}{\tau^2}\cdot 0 \geq \DOAlg({\cv^+},Z^{i-1},Z^i) - \nicefrac{2n}{\tau}.
\]
\end{proofof}
\begin{proofof}{prop: cplus is approx}
In Subsection~\ref{subsec:lcb_approx}, we show that $\DOAlg$ is invariant under permutations of the user arrival vector. The same arguments apply for $\LongDOAlg$; thus, it is invariant under in-phase permutations of the user arrival vector. Mathematically, if
$\q_i$ is a permutation of $\q_i'$ for every $i$, $1\leq i\leq\nicefrac{T}{\tau}$, then $\LongDOAlg(\q_1\oplus\dots\oplus\q_{\nicefrac{T}{\tau}}) = \LongDOAlg(\q_1'\oplus\dots\oplus\q_{\nicefrac{T}{\tau}}')$. Leveraging in-phase invariance along with stability (Property~\ref{prop:stability}), we get 
\begin{align}\label{eq:long DOAlg stability}
\abs{\LongDOAlg(\q_1\oplus\dots\oplus\q_{\nicefrac{T}{\tau}})-\LongDOAlg(\q_1'\oplus\dots\oplus\q_{\nicefrac{T}{\tau}}')}\leq O(1)\cdot \sum_{i=1}^{\nicefrac{T}{\tau}} d(c(\q_i),c(\q_i')),
\end{align} 
where $d(c(\q_i),c(\q_i'))$ is the distance between the aggregates $c(\q_i),c(\q_i')$. We conclude from Inequality~\eqref{eq:concentration} that if $\q\sim\ro^\tau$ then
\begin{align}\label{eq: distance cases}
   d({\cv^+},c(\q)) \leq
\begin{cases}
n\sqrt{\tau\log\tau} & w.p.\geq 1-\nicefrac{2n}{\tau^2} \text{ (clean event)} \\
\tau    & w.p.\leq \nicefrac{2n}{\tau^2}
\end{cases}. 
\end{align}
Using Inequality~\eqref{eq:long DOAlg stability} and Inequality~\eqref{eq: distance cases} we conclude that
\begin{align*}
    &\abs{\LongDOAlg(\cv^+\oplus\dots\oplus\cv^+)- \E_{\q_1,\dots,\q_{\nicefrac{T}{\tau}}\sim\ro^\tau}[\LongDOAlg(\q_1\oplus\dots\oplus\q_{\nicefrac{T}{\tau}})]}  \\
    &\leq O(1)\cdot \sum_{i=1}^{\nicefrac{T}{\tau}} \E_{\q_i\sim\ro^\tau}[d(\cv^+,c(\q_i))]
    \leq O\left(\frac{T}{\tau}\right)\cdot\left(n\sqrt{\tau\log\tau}\cdot( 1-\nicefrac{2n}{\tau^2}) +\tau\cdot \nicefrac{2n}{\tau^2}\right) \leq O\left(\frac{nT}{\sqrt{\tau}}\right). 
\end{align*}
\end{proofof}

\subsection{Learning in the long phase regime}\label{subsec:appen long phase learning}
In this subsection, we discuss the performance guarantees of $\mainAlg(\LPOR)$. Recall that Theorem~\ref{theoremMainRegret1} assumes that $\SSO$ satisfies Properties~\ref{prop_PI12343}-\ref{prop:stability}. However, with a minor modification of $\err$,  Theorem~\ref{theoremMainRegret1} also holds if $\SSO$ satisfies only stability (Property~\ref{prop:stability}). 
\begin{theorem}\label{thm: main regret no PICO}
Fix a stochastic optimization algorithm $\SSO$ satisfying Property~\ref{prop:stability}.
Let $\err(\SSO)=\sup\left\{ \E[\tilde{r}^{\OPT}]-\E[\tilde{r}^{\SSO}]\right\}$ be the approximation error of $\SSO$ w.r.t. $\OPT$. Then, 
\[
\E[r^{\mainAlg(\SSO)}]\geq \E[\tilde{r}^{\SSO}]- \tilde{O}\left((\sqrt{\gamma}n^2+\nicefrac{1}{\gamma})\cdot T^{\nicefrac{2}{3}}\right)-\err(\SSO).
\]
\end{theorem}
The proof of Theorem~\ref{thm: main regret no PICO} is identical to the proof of Theorem~\ref{theoremMainRegret1}; hence, we omit it. Since $\LPOR$ satisfies Property~\ref{prop:stability}, Theorem~\ref{thm: main regret no PICO} and Inequality~\eqref{eq:LPOR related to OPT} hint that $\E[\R^{\mainAlg(\LPOR)}]\leq \tilde O( T^{\nicefrac{2}{3}}+\nicefrac{T}{\sqrt{\tau}})$. Namely, if $\tau=\Omega(T^{\nicefrac{2}{3}})$, then $\E[\R^{\mainAlg(\LPOR)}] = \tilde O( T^{\nicefrac{2}{3}})$.


\subsection{Lower bound}\label{subsec:appen long phase lower bound}
In this subsection, we show a lower bound for the regime $\tau=\Omega(T^{\nicefrac{2}{3}})$. 
The next theorem modifies our lower bound construction from Theorem~\ref{theMainLB} for the case of longer phase length $\tau$.
\begin{theorem}\label{theMainLB2}
Let $A$ be any learning algorithm. If $\tau=\Theta(T^x)$ for some $x\geq\frac{2}{3}$, then it holds that $\E[\R^A]= \Omega(T^{1-\nicefrac{1}{2}x})= \Omega(\nicefrac{T}{\sqrt{\tau}}).$
\end{theorem}
In particular, if $x=1$, namely, $\tau=\Theta(T)$, we get a  regret of $O(\sqrt T)$, and $x=\nicefrac{2}{3}$ results in a regret of $O( T^{\nicefrac{2}{3}})$ as expected. The proof of this theorem is combined with the proof of Theorem~\ref{theMainLB}, in Appendix~\ref{Asubsec:lb regret}.

Intuitively, the lower bound is different in this regime since, e.g., if $\tau=T$, our model  reduces to the standard MAB setting. As a result, the optimal regret is $O(\sqrt{T})$ , which is significantly better than our algorithms that achieve $O(T^{\nicefrac{2}{3}})$.  As the $\tau=\Theta(T)$ case demonstrates, the gap follows from our algorithms and calls for improving them for this regime.

\section{Robust Approximation}\label{subsec: robust_approx_appen}
Subsection~\ref{subsec:pick z} introduces the planning task, which is the optimization problem 
\[
\max_{Z\subseteq K} \DOAlg(\cv,Z),
\]
with the aim of demonstrating that implementing $\OPT$ (or even $\PIOPT$) is an NP-hard problem. In this section, our objective is to develop both efficient and low-regret algorithms. However, comparing runtime-efficient algorithms to inefficient benchmarks (such as $\OPT$ or $\PIOPT$) is an apple-to-orange comparison. 
Therefore, we adopt a relaxed version of regret, called $\alpha$-regret (see Definition~\ref{def: alpha regret}), and develop an efficient version of  $\mainAlg(\ORS)$. The proposed algorithm has a runtime of $O(T+\tau^3\cdot k^5)$ and exhibits a low $\alpha$-regret.

Before we go on, we remark that \citet{googleOmer20} investigated the planning task as a special case and proved the following Theorem~\ref{thm: from omer google submodularity}.
\begin{theorem}[\citet{googleOmer20}, Theorem $2$]\label{thm: from omer google submodularity}
    Denote by $f(Z)=\DOAlg(\cv,Z)$. Then, $f(Z)$ is a sub-modular function.
\end{theorem}
Next, we introduce the following useful definitions.
\begin{definition}[$\alpha$-Regret]\label{def: alpha regret}
For every $\alpha\in (0,1]$, the expected $\alpha$-regret of a learning algorithm $A$ with respect to the benchmark algorithm $\OPT$ is defined as \[
\E[R_{\alpha}^A]=\alpha\cdot\E[\tilde{r}^{\OPT}]-\E[r^A]. 
\]
\end{definition}
\begin{definition}[$(\alpha,\beta)$-Robust Approximation]\label{def: Robust Approximation}
Let $f:2^K\rightarrow R$ be a sub-modular function, and fix $\epsilon>0$. Let $\hat{f}$ be a value oracle that satisfies $\abs{f(Z)-\hat{f}(Z)}<\epsilon$ for all $Z\subseteq K$.
An algorithm $A$ is called $(\alpha,\beta)$-robust approximation if its output $Z^A\subseteq K$, using the value oracle $\hat{f}$, satisfies
\[
f(Z^A)\geq \alpha \max_{Z\subseteq K} f(Z) - \beta\epsilon.
\]
\end{definition}
We adopt the following theorem from the work of \citet{nie2023framework}.
\begin{theorem}[\citet{nie2023framework}, Theorem $5$]\label{thm: Greedy is Robust Approximation}
The Greedy algorithm presented in \citet{nemhauser1978analysis} is a $(1-\nicefrac{1}{e},2k)$-robust approximation algorithm for sub-modular maximization under a $k$-cardinality constraint.
\end{theorem}

We leverage Theorems~\ref{thm: from omer google submodularity},\ref{thm: Greedy is Robust Approximation} to introduce the $\AORS$ algorithm that approximates $\ORS$. $\AORS$ is different from $\ORS$ only in lines~\ref{alglineor:bestZ}-\ref{alglineor:followORZ} (refer to Algorithm~\ref{alg:lcbapprox234f}), where instead of following $Z^\star=\argmax_{Z\subseteq K}\DOAlg(Z)$ it does the following:
\begin{enumerate}
    \item pick $Z'$- the output of the Greedy algorithm for the task: $\max_{Z\subseteq K}\DOAlg(Z)$
    \item follow $\OR(Z')$
\end{enumerate}
Theorem~\ref{thm: AORS regret guar} presents the performance guarantees of $\mainAlg(\AORS)$.
\begin{theorem}\label{thm: AORS regret guar}
$\mainAlg(\AORS)$ has an $O(T+\tau^3 k^5)$ runtime. Further, it holds that
\[
(1-\nicefrac{1}{e})\cdot\E[\tilde{r}^{\PIOPT}]-\E[r^{\mainAlg(\AORS)}]\leq \tilde{O}(\frac{nT}{\sqrt{\tau}}+ (nk+\sqrt{\gamma}n^2+\nicefrac{1}{\gamma})T^{\nicefrac{2}{3}}).
\]
\end{theorem}

Overall, $\mainAlg(\AORS)$ incurs a $(1-\nicefrac{1}{e})$-regret of $\tilde{O}(\frac{T}{\sqrt{\tau}}+T^{\nicefrac{2}{3}})$ with respect to $\PIOPT$ and $\tilde{O}(\tau+ \frac{T}{\sqrt{\tau}}+T^{\nicefrac{2}{3}})$ with respect to $\OPT$. Consequently, the $(1-\nicefrac{1}{e})$-regret of $\mainAlg(\AORS)$ is same as the $1$-regret of $\mainAlg(\ORS)$.
However, $\AORS$'s time complexity is $O(T+\tau^3 k^5)$, which is \emph{polynomial} in $k$, in contrast to $\ORS$'s exponential runtime.

\begin{proof}[\textnormal{\textbf{Proof of Theorem~\ref{thm: AORS regret guar}}}]
To enhance clarity and maintain consistency, we explicitly denote the reward matrix inside the function $\DOAlg$. Namely,  we denote the algorithm $\DOAlg(\cv,Z)$ given the matrix $\Mmu$ as $\DOAlg(\cv,Z,\Mmu)$. Additionally, we define $\hat{\cv}^+$ as the lower confidence aggregates obtained from $\hat{\ro}$, mirroring the notation of $\cv^+$ for $\ro$. It is worth mentioning that when $\mainAlg(\AORS)$ calls the greedy algorithm, it already knows $\hat{\ro},\hat{\Mmu}$. Next, we prove the following claim.
\begin{proposition}\label{prop: match is robust}
Let $f(Z) = \DOAlg(\cv^+,Z,\Mmu)$ and $\hat{f}(Z) = \DOAlg(\hat{\cv}^+,Z,\hat{\Mmu})$ be two sub-modular functions. Then, for every $Z\subseteq K$, it holds that $\abs{f(Z)-\hat{f}(Z)}\leq\epsilon,$ where $\epsilon = \tilde{O}(n\tau\cdot T^{-\nicefrac{1}{3}})$.
\end{proposition}
From the Greedy algorithm's guarantees, Proposition~\ref{prop: match is robust}, \Cref{thm: Greedy is Robust Approximation} and \Cref{thm: from omer google submodularity}, we conclude that using $O(k^2)$ calls to the oracle $\DOAlg(\hat{\cv}^+,Z,\hat{\Mmu})$, the Greedy algorithm outputs $Z'\subseteq K$ such that
\begin{align}\label{eq: robust approx}
\DOAlg(\cv^+,Z',\Mmu)\geq (1-\nicefrac{1}{e})\DOAlg(\cv^+,Z^{\star},\Mmu) - 2k \cdot \tilde{O}(n\tau\cdot T^{-\nicefrac{1}{3}}).    
\end{align}
In addition, we show in the proof of \Cref{thm:ors}  that 
\begin{align}\label{eq: insight from theorem 4}
\DOAlg(\cv^+,Z^{\star},\Mmu)\geq \E_q[\tilde{r}_{1:\tau}^{\PIOPT}]-\tilde{O}(n\sqrt{\tau}).    
\end{align}

From this point onward, we proceed to analyze the performance of $\mainAlg(\AORS)$. The runtime of the Greedy algorithm is $O(k^2)\cdot O(\tau^3k^3)$ so, the runtime of $\mainAlg(\AORS)$ is $O(T+\tau^3k^5)$. We overload the notation of $V^{\Mmu,\ro}(Alg)$ from the proof of Theorem~\ref{theoremMainRegret1} and denote by $V_{1:\tau}^{\Mmu,\ro}(Alg)$ the reward of the stochastic optimization algorithm $Alg$ in one phase where the instance parameters are $\Mmu,\ro$.  Now, we prove Proposition~\ref{prop: approximate OR near to match}.
\begin{proposition}\label{prop: approximate OR near to match}
It holds for every $Z\subseteq K$ that \[V_{1:\tau}^{\Mmu,\ro}(\OR(\hat{\ro},\hat{\Mmu},Z))\geq \DOAlg(\cv^+,Z,\Mmu) -\tilde{O}(\sqrt{\gamma}n^2\tau T^{-\nicefrac{1}{3}}).\]
\end{proposition}
Overall, we get that
\begin{align}\label{eq: AORS approx PICO}
&V_{1:\tau}^{\Mmu,\ro}(\OR(\hat{\ro},\hat{\Mmu},Z'))
\geq \DOAlg(\cv^+,Z',\Mmu) -\tilde{O}(\sqrt{\gamma}n^2\tau T^{-\nicefrac{1}{3}})\nonumber
\\ \geq& (1-\nicefrac{1}{e})\DOAlg(\cv^+,Z^{\star},\Mmu) - 2k \cdot \tilde{O}(n\tau\cdot T^{-\nicefrac{1}{3}}) -\tilde{O}(\sqrt{\gamma}n^2\tau T^{-\nicefrac{1}{3}})\nonumber
\\ \geq& (1-\nicefrac{1}{e})(\E_q[\tilde{r}_{1:\tau}^{\PIOPT}]-\tilde{O}(n\sqrt{\tau})) - 2k \cdot \tilde{O}(n\tau\cdot T^{-\nicefrac{1}{3}}) -\tilde{O}(\sqrt{\gamma}n^2\tau T^{-\nicefrac{1}{3}})\nonumber
\\ =& (1-\nicefrac{1}{e})\E_q[\tilde{r}_{1:\tau}^{\PIOPT}] - \tilde{O}(n\sqrt{\tau}+ (nk+\sqrt{\gamma}n^2)\tau T^{-\nicefrac{1}{3}}).
\end{align}
The first inequality stems from Proposition~\ref{prop: approximate OR near to match}. The second follows from Inequality~\eqref{eq: robust approx}, and the last inequality is due to Inequality~\eqref{eq: insight from theorem 4}. Ultimately, we get 
\begin{align*}
    &\E[r^{\mainAlg(\AORS)}] = \E[r_{1:\nicefrac{1}{\gamma}T^{\nicefrac{2}{3}}}^{explore}] + \E[r_{\nicefrac{1}{\gamma}T^{\nicefrac{2}{3}}:T}^{\AORS}] 
    \geq \E[\tilde{r}^{\OR(\hat{\ro},\hat{\Mmu},Z')}]- \nicefrac{1}{\gamma}T^{\nicefrac{2}{3}} 
    \\ &= \frac{T}{\tau}V_{1:\tau}^{\Mmu,\ro}(\OR(\hat{\ro},\hat{\Mmu},Z')) - \nicefrac{1}{\gamma}T^{\nicefrac{2}{3}} 
    \geq (1-\nicefrac{1}{e})\E[\tilde{r}^{\PIOPT}]- \tilde{O}\left(\frac{nT}{\sqrt{\tau}}+ (nk+\sqrt{\gamma}n^2+\nicefrac{1}{\gamma})T^{\nicefrac{2}{3}}\right),
\end{align*}
where the last equality is due to Equation~\eqref{eq:phaseToT}, and the last inequality is due to Inequality~\eqref{eq: AORS approx PICO}. 
\end{proof}

\begin{proof}[\textnormal{\textbf{Proof of Proposition~\ref{prop: match is robust}}}]
In this proof, we assume that we are in the "clean" event. I.e., we assume that for every $u\in U$ it holds that  $|\roSub_u-\hat{\roSub}_u|\leq \tilde{O}(T^{-\nicefrac{1}{3}})$. We can assume it since we show in the proof of Proposition~\ref{shortappendix_propReg1} that the "bad" event happens with probability $\leq \nicefrac{2n}{T^3}$. Consequently, its effect on the expected rewards is negligible. It stems directly from the definition of $\cv^+$ that
\[\sum_{u\in U} \abs{\cv^+_u-\hat{\cv}^+_u} =\sum_{u\in U} \abs{\roSub_u\tau-\hat{\roSub}_u\tau} = \tilde{O}(n\tau T^{-\nicefrac{1}{3}}).\]

Let's fix $Z\subseteq K$. Assuming that $\DOAlg(\hat{\cv}^+, Z, \hat{\Mmu})$ chooses the same matching as $\DOAlg(\cv^+, Z, \Mmu)$, its expected reward is at least $\DOAlg(\cv^+, Z, \Mmu)$'s expected reward minus two factors: The difference in aggregate values ($\tilde{O}(n\tau T^{-\nicefrac{1}{3}})$) and the difference in preferences scaled by the matching size ($\tau \cdot \tilde{O}(T^{-\nicefrac{1}{3}})$).

As $\DOAlg$ is an optimal algorithm, if $\DOAlg(\hat{\cv}^+, Z, \hat{\Mmu})$ does not choose the same match as $\DOAlg(\cv^+, Z, \Mmu)$, it means that its reward would be better than the reward it could achieve by selecting the match of $\DOAlg(\cv^+, Z, \Mmu)$. Overall, we get that
\[
\DOAlg(\hat{\cv}^+,Z,\hat{\Mmu})\geq \DOAlg(\cv^+,Z,\Mmu) - \tilde{O}(n\tau T^{-\nicefrac{1}{3}}) - \tau\cdot \tilde{O}(T^{-\nicefrac{1}{3}}) = \DOAlg(\cv^+,Z,\Mmu) - \tilde{O}(n\tau T^{-\nicefrac{1}{3}}).
\]
From symmetry, it holds for every $Z\subseteq K$ that \[\abs{\DOAlg(\cv^+,Z,\Mmu)- \DOAlg(\hat{\cv}^+,Z,\hat{\Mmu})}\leq \tilde{O}(n\tau T^{-\nicefrac{1}{3}}).\] 
\end{proof}

\begin{proof}[\textnormal{\textbf{Proof of Proposition~\ref{prop: approximate OR near to match}}}]
We prove the proposition by showing the following inequalities.
\begin{align*}
    &V_{1:\tau}^{\Mmu,\ro}(\OR(\hat{\ro},\hat{\Mmu},Z)) \\ \geq& V_{1:\tau}^{\hat{\Mmu},\hat{\ro}}(\OR(\hat{\ro},\hat{\Mmu},Z))- \tilde{O}(\sqrt{\gamma}n^2\tau T^{-\nicefrac{1}{3}}) \\ \geq& \DOAlg(\hat{\cv}^+,Z,\hat{\Mmu})-\nicefrac{2n}{\tau}-  \tilde{O}(\sqrt{\gamma}n^2\tau T^{-\nicefrac{1}{3}}) \\ \geq& \DOAlg(\cv^+,Z,\Mmu)- \tilde{O}(n\tau\cdot T^{-\nicefrac{1}{3}}) - \tilde{O}(\sqrt{\gamma}n^2\tau T^{-\nicefrac{1}{3}}) =  \DOAlg(\cv^+,Z,\Mmu) - \tilde{O}(\sqrt{\gamma}n^2\tau T^{-\nicefrac{1}{3}}),
\end{align*}
where the first inequality is due to Proposition~\ref{shortappendix_propReg2}. The second inequality follows from Proposition~\ref{prop_short1234321} and the last inequality stems from Proposition~\ref{prop: match is robust}.
\end{proof}

}\fi}

\end{document}